\pdfoutput=1
\documentclass{article}

\PassOptionsToPackage{numbers,sort}{natbib}
\usepackage[main, final]{neurips_2026}

\usepackage[utf8]{inputenc}
\usepackage[T1]{fontenc}
\usepackage{hyperref}
\usepackage{url}
\usepackage{graphicx}
\usepackage{subcaption}
\usepackage{booktabs}
\usepackage{multirow}
\usepackage{algorithm}
\usepackage{algorithmic}
\usepackage{amsmath}
\usepackage{amssymb}
\usepackage{mathtools}
\usepackage{amsthm}
\usepackage{enumitem}
\usepackage[capitalize,noabbrev]{cleveref}
\usepackage[disable,textsize=tiny]{todonotes}

\theoremstyle{plain}
\newtheorem{theorem}{Theorem}[section]

\newtheorem{lemma}[theorem]{Lemma}
\newtheorem{corollary}[theorem]{Corollary}
\theoremstyle{definition}
\newtheorem{definition}[theorem]{Definition}
\newtheorem{assumption}[theorem]{Assumption}
\theoremstyle{remark}
\newtheorem{remark}[theorem]{Remark}

\newcommand{\cT}{{\cal T}}

\def\calA{\mathcal{A}}  
 \def\calE{\mathcal{E}} 
 \def\calH{\mathcal{H}}

\def\calP{\mathcal{P}}  
 \def\calT{\mathcal{T}} 
  \def\calX{\mathcal{X}}
\def\calY{\mathcal{Y}} \def\calZ{\mathcal{Z}}

\def\E{\mathbb{E}}

\def\RR{\mathbb{R}}
\def\TT{\mathbb{T}}

\def\SS{\mathbb{S}}

\def\mat{\hbox{\rm mat}}


\title{Guaranteed Noisy CP Tensor Recovery via Riemannian Optimization on the Segre Manifold}

\author{%
  Ke Xu \\
  Department of Applied and Computational\\
  Mathematics and Statistics\\
  University of Notre Dame\\
  Notre Dame, IN, USA\\
  \texttt{kxu6@nd.edu}
  \And
  Yuefeng Han \\
  Department of Applied and Computational\\
  Mathematics and Statistics\\
  University of Notre Dame\\
  Notre Dame, IN, USA\\
  \texttt{yuefeng.han@nd.edu}
}

\begin{document}
\maketitle

\begin{abstract}
Recovering a low-CP-rank tensor from noisy linear measurements is a central challenge in high-dimensional data analysis, with applications spanning tensor PCA, scalar-on-tensor regression, tensor-on-tensor regression, and beyond. We exploit the intrinsic geometry of rank-one tensors by casting the recovery task as an optimization problem over the Segre manifold, the smooth Riemannian manifold of rank-one tensors. This geometric viewpoint yields two powerful algorithms: Riemannian Gradient Descent (RGD) and Riemannian Gauss-Newton (RGN), each of which preserves feasibility at every iteration. Under mild noise assumptions, we prove that RGD converges at a local linear rate, while RGN exhibits an initial local quadratic phase followed by linear refinement to the statistical error. Extensive synthetic experiments validate these convergence guarantees and demonstrate the practical effectiveness of our methods.
\end{abstract}

\section{Introduction}
\vspace{-1ex}

Tensor decomposition, particularly the CP decomposition, has emerged as a powerful tool for analyzing high-dimensional data across diverse domains such as chemometrics, neuroscience, and recommendation systems \citep{kolda2009tensor, tang2023multivariate, frolov2017tensor, bi2021tensors, zhou2025factor}, and tensor-valued time series in economics and finance \citep{han2024cp, chen2026estimation, han2022rank, han2024tensor, chen2026diffusion}. Specifically, for an order-$d$ tensor $\calT \in \RR^{p_1 \times \cdots \times p_d}$ with $d\geq3$, the CP decomposition is a sum of rank-one tensors: 
\vspace{-1ex}
\begin{align} \label{cp-tensor}
\calT
=\sum_{i=1}^r \lambda_iu_{1,i}\otimes u_{2,i}\otimes\cdots\otimes u_{d,i},    
\end{align}
where $\otimes$ denotes the tensor product, each factor $u_{k,i}\in\RR^{p_k}$ satisfies $\|u_{k,i}\|_2=1$, $r$ is the CP rank, and $\lambda_i>0$ after absorbing each component sign into one factor. Under mild identifiability conditions (e.g., Kruskal's criterion \citep{kruskal1977three}), this representation is essentially unique up to scaling and permutation, making it a widely adopted model in multi-way data analysis. 

In practice, one often only observes noisy measurements of $\calT$, for example
\vspace{-1ex}
\[
\calY = \calA(\calT) + \calE,
\]
where $\calA$ is a linear observation operator (possibly compressive) and $\calE$ denotes additive noise.

In this work, we address the problem of recovering the underlying low-CP-rank tensor $\calT$ from noisy measurements. In particular, we perform optimization directly on the \emph{Segre manifold}, a smooth Riemannian manifold composed of rank-one tensors. Utilizing Riemannian optimization techniques ensures that the iterates remain on the manifold, thereby preserving the structure of the CP model and achieving improved convergence properties over traditional Euclidean approaches \citep{breiding2018riemannian}.
Our contributions can be summarized as follows:
\begin{enumerate}[leftmargin = *]
\vspace{-1ex}
    \item We develop \textbf{Riemannian Gradient Descent (RGD)} and \textbf{Riemannian Gauss-Newton (RGN)} algorithms for noisy low-CP-rank tensor recovery from general linear measurements, including tensor decomposition and tensor-on-tensor regression, by directly optimizing on the Segre manifold and using its established geometry \citep{breiding2018riemannian,swijsen2022tensor}. 
\vspace{-0.5ex}    
    \item We derive \textbf{convergence guarantees} for both RGD and RGN in the noisy case and analyze how the CP geometry affects the convergence. In particular, for CP tensor-on-tensor regression under general linear measurements, we provide algorithm-level local convergence guarantees and nonasymptotic statistical-error guarantees.
\vspace{-0.5ex}    
    \item Numerical experiments demonstrate that our algorithms yield robust factor recovery under noise and achieve competitive convergence and reconstruction accuracy relative to standard methods.
\end{enumerate}

\vspace{-2ex}
\subsection{Related Work}
\vspace{-1ex}

Alternating Least Squares (ALS) \citep{carroll1970analysis, harshman1970foundations, kolda2009tensor, comon2009tensor} is the classical workhorse for CP tensor decomposition, valued for its conceptual simplicity and low per-iteration cost. Despite its widespread use, explicit local convergence rates for ALS on non-orthogonal CP tensors were established only recently \citep{tang2025revisit, hu2025convergence}. Early theoretical analyses circumvented this difficulty by restricting attention to the orthogonal CP model \citep{anandkumar2014tensor, wang2017tensor}. Other work has treated the non-orthogonal case under incoherence assumptions: \citet{anandkumar2014guaranteed} analyzed alternating rank-one updates when the factors are random vectors on the sphere, while \citet{sharan2017orthogonalized} proposed orthogonalized ALS, which orthogonalizes the factor estimates before each sweep and provably recovers the factors from a random start.

Low-rank tensor regression offers another important motivation for the general linear measurement formulation. Classical formulations with tensor covariates and scalar responses impose low-rank structure on the coefficient tensor, often through CP-type parameterizations \citep{zhou2013tensor}. Tensor-on-tensor regression models predict a tensor-valued response from a tensor-valued covariate. Approaches include a penalized CP estimator for the contracted tensor product model \citep{lock2018tensor}, likelihood estimators for several low-rank formats \citep{llosa2023reduced}, the latent-variable method HOPLS \citep{zhao2013hopls}, and methods for sparse coefficients or outliers \citep{mao2022decomposable,lee2024robust}. Tucker-rank RGD and RGN have nonasymptotic recovery and convergence theory \citep{luo2024tensor}, and tensor-train IHT and RGD have RIP-based recovery guarantees \citep{qin2024computational}. Our framework departs from these approaches by casting CP decomposition and regression with a CP coefficient as instances of a common noisy linear inverse problem, optimized over products of Segre manifolds, with algorithm-level convergence to a nonasymptotic statistical error bound.

A complementary line of work leverages manifold optimization for tensor estimation, with notable success on low-rank matrix and Tucker decompositions \citep{boumal2023intromanifolds, luo2023low, luo2024tensor}. Tensors of fixed Tucker rank form a smooth Riemannian manifold whose tangent space admits a simple parameterization \citep{kressner2014low}, providing a natural and computationally tractable framework for optimization. Tangent-space projections on the Tucker manifold also underlie recent debiased inference for low-Tucker-rank regression and PCA \citep{xu2025statistical}. By incorporating this geometry directly into the optimization process, such methods have yielded substantial improvements in tensor recovery, especially under noise.

Extending these ideas to the CP setting, however, presents distinct challenges: the CP model is inherently non-orthogonal, leading to slow convergence, spurious local minima, and higher computational cost. Our work addresses these challenges by explicitly leveraging the geometric structure of the rank-one tensor space. The key observation is that a rank-one tensor can be regarded as a Tucker rank-one tensor, sidestepping the non-orthogonality that complicates direct optimization on the CP manifold. Such greedy rank-one updates are a natural strategy for CP decomposition \citep{zhang2001rank}, and linear convergence rates have been established for incoherent CP tensors \citep{anandkumar2014guaranteed, sun2017provable}. \citet{breiding2018riemannian} use the Cartesian product of $r$ Segre manifolds in a Riemannian Gauss-Newton trust-region method for CP approximation, and \citet{swijsen2022tensor} derive geodesics on the Segre manifold. Building on this geometry and on the projection--retraction analysis of \citet{luo2023low,luo2024tensor}, we develop algorithms that respect the geometry of the CP model and ensure feasibility at every iteration, and we analyze the coupling of the $r$ components under noise.

\vspace{-2ex}
\subsection{Organization}
\vspace{-1ex}

The remainder of the paper is organized as follows. Section~\ref{sec:model} introduces a unified linear observation framework and instantiates it on two core examples: CP tensor decomposition and tensor-on-tensor regression, with scalar-on-tensor regression recovered as the special case. Section~\ref{sec:method} develops the proposed Riemannian optimization algorithms in full detail. Section~\ref{sec:theory} establishes the theoretical analysis, including local convergence guarantees. Section~\ref{sec:experiment} reports comprehensive numerical experiments, and Section~\ref{sec:conclusion} concludes with directions for future work. Computational complexity, additional numerical results, real data analysis, and proofs are deferred to the appendix.

\textbf{Notation.} $[n]=\{1,\ldots,n\}$, $\langle\cdot,\cdot\rangle$ is the entrywise inner product, $\operatorname{vec}$ is vectorization, $\times_k$ is the mode-$k$ product, $\preceq$ is the Loewner order, and $\widetilde O(\cdot)$ hides logarithmic factors. For a linear map $\mathcal M$ and a subspace $\TT$, $\mathcal M^*$, $\|\mathcal M\|$ and $\mathcal M|_{\TT}$ are the adjoint, operator norm and restriction, $\calP_{\TT}$ is the orthogonal projector onto $\TT$, $\calP_{\TT}^{\perp}=I-\calP_{\TT}$, and $I$ and $\mathrm{Id}$ are identities.

\vspace{-2ex}
\section{Model and Problem Formulation}\label{sec:model}
\vspace{-1ex}

Our goal is to recover the signal tensor $\calT$ admitting the CP decomposition in \eqref{cp-tensor}, by solving an optimization problem that exploits the geometry of the Segre manifold. Specifically, we consider %
\vspace{-1ex}
\begin{equation}\label{eq:main_opt}
\min_{\calT_i \in\operatorname{Seg}, i \in [r]}\mathcal{L} (\{\calT_i \}_{i=1}^r) = \min_{\calT_i \in\operatorname{Seg}, i \in [r]}
  \frac12\Bigl\|\calY
       -\sum_{i=1}^r \calA(\calT_i)\Bigr\|_{\mathrm{F}}^2,
\end{equation}
where $\operatorname{Seg}$ denotes the Segre manifold of rank-one tensors (Definition~\ref{def:segre}), and $\calA: \RR^{p_1 \times \cdots \times p_d} \to \mathcal{H}$ is a (possibly random) linear operator from the ambient tensor space to a finite-dimensional Euclidean space \(\mathcal{H}\), equipped with the Frobenius norm $\|\cdot\|_{\mathrm{F}}$. This formulation is general enough to accommodate both complete and compressive observations of the tensor. In particular, for tensor decomposition, $\mathcal H=\RR^{p_1\times\cdots\times p_d}$, whereas for tensor-on-tensor regression, $\mathcal H
=(\RR^{p_{d_{\rm in}+1}\times\cdots\times p_d} )^n$, with $d_{\rm out}=d-d_{\rm in}$ and the scalar-response case corresponding to $d_{\rm out}=0$.

Previous work has largely focused on estimating the tensor factors by iteratively optimizing over each tensor mode \citep{carroll1970analysis, sharan2017orthogonalized}. In contrast, our formulation directly iterates on the Segre manifold, the smooth Riemannian manifold composed of rank-one tensors. This intrinsic approach leverages the rank-one structure of each component, preserving the CP structure throughout the optimization. This formulation accommodates a variety of applications, two of which we highlight below.

\textbf{Tensor Decomposition.} When the entire signal tensor $\cT$ is observed, we set \(\calA = \mathrm{Id}: \RR^{p_1\times \cdots \times p_d} \rightarrow \RR^{p_1\times \cdots \times p_d}\). In this case, the problem in \eqref{eq:main_opt} becomes $\min_{(\calT_1,\dots,\calT_r)\,\in\,\operatorname{Seg}^r}\frac{1}{2}\|\calY-\sum_{i=1}^r \calT_i\|_{\mathrm{F}}^2$, which is exactly the classical noisy CP decomposition.

\textbf{Tensor-on-Tensor Regression.}
Let \(d_{\rm in}\in\{1,\ldots,d\}\) denote the number of input modes, and \(d_{\rm out}=d-d_{\rm in}\) the number of output modes,
so that the unknown coefficient tensor \(\calT\in \RR^{p_1\times\cdots\times p_d}\) partitions accordingly. For each \(m=1,\ldots,n\), we observe a covariate tensor
\(\calX_m\in \RR^{p_1\times\cdots\times p_{d_{\rm in}}}\), and a response tensor $\calY_m \in
\RR^{p_{d_{\rm in}+1}\times\cdots\times p_d}$ satisfying $\calY_m=\langle \calX_m,\calT\rangle_{\rm in}+\calE_m$, where $\calE_m$ is a noise tensor of the same dimensions as $\calY_m$, and \(\langle \calX_m,\calT\rangle_{\rm in}\) denotes contraction of
\(\calT\) with \(\calX_m\) along the first \(d_{\rm in}\) modes, namely
\vspace{-1ex}
\[
\bigl(\langle \calX_m,\calT\rangle_{\rm in}\bigr)_{j_{d_{\rm in}+1},\dots,j_d}
=
\sum_{j_1=1}^{p_1}
\cdots
\sum_{j_{d_{\rm in}}=1}^{p_{d_{\rm in}}}
(\calX_m)_{j_1,\dots,j_{d_{\rm in}}}
\calT_{j_1,\dots,j_{d_{\rm in}},j_{d_{\rm in}+1},\dots,j_d}.
\]
To express this model in \eqref{eq:main_opt}, we stack and scale the responses and errors as
\(\calY=n^{-1/2}(\calY_1,\ldots,\calY_n)\) and
\(\calE=n^{-1/2}(\calE_1,\ldots,\calE_n)\), with \(\calH\) equipped with the usual product Frobenius inner product. Then \(\calY=\calA(\calT)+\calE\), where
\( \calA:\RR^{p_1\times\cdots\times p_d}  \to  (\RR^{p_{d_{\rm in}+1}\times\cdots\times p_d})^n\)
is defined by \(\calA(\calT)
=
\frac{1}{\sqrt n}
\left(
\langle \calX_1,\calT\rangle_{\rm in},
\dots,
\langle \calX_n,\calT\rangle_{\rm in}
\right)\),
whose adjoint is \(\calA^*(\calZ_1,\ldots,\calZ_n)
=
\frac{1}{\sqrt n}
\sum_{m=1}^n \calX_m\otimes \calZ_m\), yielding \(\calA^*\calA(\calT)
=
\frac{1}{n}
\sum_{m=1}^n
\calX_m\otimes
\bigl(\langle \calX_m,\calT\rangle_{\rm in}\bigr)\).

\vspace{-2ex}
\section{Method}\label{sec:method}
\vspace{-1ex}

In this section, we present two algorithms for noisy CP tensor recovery: Riemannian Gradient Descent (RGD) and Riemannian Gauss-Newton (RGN). Rather than optimizing in the full ambient space, both methods sequentially update the $r$ rank-one tensor components on the Segre manifold. %

\vspace{-2ex}
\subsection{Background: Geometry of the Segre Manifold}\label{sec:background}
\vspace{-1ex}

We begin by reviewing the geometric foundations underlying our Riemannian framework \citep{breiding2018riemannian,swijsen2022tensor}.
A nonzero \emph{rank-one tensor} in $\RR^{p_1 \times p_2 \times \cdots \times p_d}$ takes the form $\calT = u_1 \otimes u_2 \otimes \cdots \otimes u_d,$ with \(u_k \in \RR^{p_k} \setminus \{0\}\) for \(k=1,\ldots,d\). The set of projective classes of such tensors constitutes the \emph{Segre variety} in algebraic geometry \citep{landsberg2011tensors}.
Equipping the set of nonzero rank-one tensors in the ambient space $\RR^{p_1\times\cdots\times p_d}$ 
with the Frobenius metric instead yields a smooth Riemannian submanifold, the \emph{Segre manifold} ($\operatorname{Seg}$), whose geometry is detailed in \cite{jacobsson2024warped}.

\begin{definition}[Segre Manifold]  \label{def:segre}
The \textbf{Segre manifold} is the set of all nonzero rank-one tensors in the ambient space $\RR^{p_1\times\cdots\times p_d}$,
\vspace{-1ex}
\[
\operatorname{Seg} \;=\; \Big\{\, u_1\otimes u_2 \otimes \cdots \otimes u_d : u_l \in \RR^{p_l}\setminus\{0\},\ \forall l \in[d] \,\Big\}.
\]
It is a smooth embedded submanifold of $\RR^{\prod_{l=1}^d p_l} \setminus \{0\}$ of dimension \(\dim(\operatorname{Seg}) = 1 + \sum_{l=1}^d (p_l-1)\).
\end{definition}

\begin{remark}\label{remark:segre-quotient}
The Segre manifold admits the equivalent parameterization
\vspace{-1ex}
\[
\operatorname{Seg} \ \cong\ \big((0,\infty) \times \mathbb{S}^{p_1-1}\times \cdots \times \mathbb{S}^{p_d-1}\big)\big/G,
\]
where $\mathbb S^{p-1}$ is the unit sphere of $\RR^{p}$ and $G=\{(\varepsilon_1,\dots,\varepsilon_d)\in\{\pm1\}^d : \prod_{k=1}^d \varepsilon_k=1\}$ acts by simultaneous sign flips. The quotient absorbs the sign ambiguity arising from the fact that different sign patterns of the factor vectors can represent the same tensor.
\end{remark}

We therefore optimize over $r$-tuples of rank-one tensors lying on the Segre manifold ($\operatorname{Seg}$). To ensure these components remain distinguishable, we impose an incoherence condition among them, effectively acting as a soft-orthogonality constraint.

\begin{assumption}\label{assumption:incoherence} For each mode $l \in [d]$, the factor vectors satisfy the incoherence condition
\vspace{-1ex}
\[
\mu_l = p_l \cdot \max_{\substack{i, j \in [r], i \neq j}} |\langle u_{l,i}, u_{l,j}\rangle|^2,
\]
and we denote the maximum incoherence by $\eta = \max_{l \in [d]}\sqrt{\mu_l / p_l} = \max _l \max _{i \neq j}\left|\left\langle u_{l, i}, u_{l, j}\right\rangle\right|$.
\end{assumption}

This incoherence assumption is standard in the CP tensor estimation literature \citep{anandkumar2014guaranteed, cai2022uncertainty}. By Lemma 2 of \citet{anandkumar2014guaranteed}, if $\{u_{l, i}\}_{l \in [d], i \in [r]}$ are drawn i.i.d. from the unit sphere $\SS^{p_l - 1}$, then $\max _{i \neq j}\{|\langle u_{l, i}, u_{l, j}\rangle|\} =\widetilde O(1 / \sqrt{p_l})$ with high probability.  Existing analyses typically require $\eta$ to vanish with the dimension, e.g., $\eta = O(1/{\sqrt{\max_{l \in [d]} p_l}})$ \citep{anandkumar2014guaranteed, cai2022uncertainty}. Our analysis, by contrast, only requires $\eta$ to be bounded by a sufficiently small constant, independent of dimension.

Any CP tensor of rank $r$ admits a Tucker representation with multilinear rank bounded componentwise by $(r,\ldots,r)$, and in the rank-one case, the two parameterizations coincide. Optimizing over the product of $r$ rank-one manifolds, rather than over each of the $d$ mode factors of a rank-$r$ tensor, thus fully exploits the intrinsic rank-one structure and handles non-orthogonal factor interactions.

\vspace{-1ex}
\subsection{Riemannian Gradient Descent on the Segre Manifold}\label{subsec:RGD}
\vspace{-1ex}

Since standard Euclidean gradient descent ignores the underlying manifold structure, we employ Riemannian gradient descent instead. For a current rank-one component $\widehat{\calT}_i\in \operatorname{Seg}$ with tangent space $\widehat{\TT}_i$, each component update projects the Euclidean gradient onto $\widehat{\TT}_i$ and retracts the result back to the manifold,
\vspace{-1ex}
\begin{align*}
\widehat{\calT}_i^+ = \mathcal{R}_{\widehat{\calT}_i} \Big( -\alpha_t \, \calP_{\widehat{\TT}_i}\Big(\nabla_{\calT_i} \mathcal{L}\big( \{\widehat{\calT}_j \}_{j=1}^r\big)\Big)\Big),
\end{align*}
where \(\alpha_t>0\) is the step size, $\nabla_{\calT_i} \mathcal{L}$ is the partial Euclidean gradient, and $\widehat{\calT}_i^+$ denotes the updated component.

\textbf{Tangent Space of the Segre Manifold.} The tangent space captures the manifold's local linear structure around a base point. For $\calT_i=\lambda_i u_{1,i}\otimes\cdots\otimes u_{d,i}$ with $\lambda_i>0$ and $\|u_{k,i}\|_2=1$, the tangent space \(\TT_i:=T_{\calT_i}\operatorname{Seg}\) consists of all first-order variations along the factor directions. Concretely, every $\xi_i \in \TT_i$ decomposes as
\vspace{-1ex}
{\small \[
\xi_i=\sum\nolimits_{k=1}^d u_{1, i} \otimes \cdots \otimes u_{k-1, i} \otimes h_{k, i} \otimes u_{k+1, i} \otimes \cdots \otimes u_{d, i} ,
\]}
where each $h_{k, i} \in \RR^{p_k}$ is an infinitesimal perturbation of the $k$-th factor. 

For each mode $k$, let $\calP_{k, i}=u_{k, i} u_{k, i}^{\top}$ denote the orthogonal projector onto the span of $u_{k,i}$, and let $\calP_{k, i}^{\perp}=I_{p_k}-u_{k, i} u_{k, i}^{\top}$ denote its complement. %
Minimizing $\|\widetilde{\calT}-\xi_i\|_{\mathrm{F}}^2$ subject to $\xi_i \in \TT_i$ yields the projection of an arbitrary tensor $\widetilde{\cT}$ onto the tangent space at $\calT_i$,
\vspace{-1ex}
\begin{equation}
\xi_i =\calP_{\TT_{i}}(\widetilde{\cT})=\sum\nolimits_{k=1}^d \widetilde{\cT} \times_k \calP_{k, i}^{\perp} \times_{l \neq k}  \calP_{l,i}+ \widetilde{\cT} \times_{l \in [d]} \calP_{l,i}. \label{eq:tangent space projection}
\end{equation}
\vspace{-1ex}

\textbf{Retraction.}
A tangent-space update generally lands off the manifold, so we apply a retraction to map it back onto the Segre manifold. Common choices include the truncated higher-order singular value decomposition (T-HOSVD) \citep{de2000multilinear} and its sequential variant (ST-HOSVD) \citep{vannieuwenhoven2012new,breiding2018riemannian}, and recent work derives explicit geodesics and the exponential map on the Segre manifold \citep{swijsen2022tensor,jacobsson2024warped}. We refer the reader to \citet{boumal2023intromanifolds} for a comprehensive treatment of these operators. We adopt T-HOSVD: for a tangent increment $\xi_i$ at $\calT_i$, $\mathcal R_{\calT_i}(\xi_i)$ is the rank-$(1,\ldots,1)$ T-HOSVD of $\calT_i+\xi_i$ \citep{kressner2014low}, and by Lemma~\ref{lemma:local-thosvd-retraction} it is a smooth retraction whenever \mbox{$\|\xi_i\|_{\mathrm F}<\|\calT_i\|_{\mathrm F}/[2(\sqrt d+1)]$}.
\vspace{-1ex}
\subsection{Riemannian Gauss-Newton on the Segre Manifold}\label{subsec:RGN}
\vspace{-1ex}

Riemannian gradient descent is conceptually simple but can converge slowly, particularly for large-scale problems or when high accuracy is required. Incorporating second-order information offers a powerful remedy. The Riemannian Gauss-Newton method \citep{luo2024tensor}, tailored to nonlinear least-squares, provides an efficient approximation to the full Riemannian Newton step.

At each component update, RGN seeks a tangent-space direction $\Delta_i\in\widehat\TT_i$ satisfying the Gauss-Newton equation
\vspace{-1ex}
\[
\left(\left.\calP_{\widehat\TT_i}\calA^*\calA\calP_{\widehat\TT_i}\right|_{\widehat\TT_i}\right)\Delta_i
=
\calP_{\widehat\TT_i}\calA^*\!\left\{\calA\!\left(\sum_{j=1}^r\widehat\calT_j\right)-\calY\right\}.
\]
which has a unique solution whenever the restricted normal operator on the left is invertible. Replacing the Riemannian Hessian by the normal operator $\calA^*\calA$ restricted to the current tangent space captures the essential curvature at low cost. To maintain feasibility, RGN then retracts $-\Delta_i$ back onto the Segre manifold. This approach still amounts to a least-squares problem, but in a drastically lower-dimensional space: the tangent space has $1 + \sum_{l \in [d]} (p_l-1)$ degrees of freedom, compared to $\prod_{l \in [d]} p_l$ parameters in the ambient tensor space $\RR^{p_1 \times p_2 \times \cdots \times p_d}$.

\begin{algorithm}[ht]
\caption{Riemannian Gradient Descent for CP Tensor Estimation}
\label{alg:RGD_CP}
\textbf{Input:} Observation \(\calY\in\calH\), linear operator \(\calA:\RR^{p_1\times\cdots\times p_d}\to\calH\), target CP rank \(r\), step size \(\alpha_t\), and initial rank-one tensor estimates \(\{\calT_i^{(0)}\}_{i=1}^r\).
\begin{algorithmic}[1]
\STATE Compute the residual $\mathrm{res} \gets \sum_{j=1}^r \calA(\calT_j^{(0)}) - \calY$.
\FOR{$t = 0, 1, \dots, t_{\max}-1$}
\FOR{$i = 1, \dots, r$}
  \STATE \textbf{(RGD Update)} With $\TT_i^{(t)}$ the tangent space at $\calT_i^{(t)}$, update
  \vspace{-1ex}
  {\[
    \calT_i^{(t+1)}
    =
    \mathcal{R}_{\calT_i^{(t)}}\!\left(
    -\alpha_t\,
    \calP_{\TT_i^{(t)}}\calA^*\!\bigl(\mathrm{res}\bigr)
    \right),
  \]}
  \STATE Update $\mathrm{res}\gets\mathrm{res}+\calA(\calT_i^{(t+1)}-\calT_i^{(t)})$.
  \ENDFOR
\ENDFOR

\end{algorithmic}
\textbf{Output:} $ \{\calT_i^{(t_{\max})} \}_{i=1}^r$.
\end{algorithm}

\begin{algorithm}[ht]
\caption{Riemannian Gauss-Newton for CP Tensor Estimation}
\label{alg:RGN_CP}
\textbf{Input:} Observation \(\calY\in\calH\), linear operator \(\calA:\RR^{p_1\times\cdots\times p_d}\to\calH\), target CP rank \(r\), and initial rank-one tensor estimates \(\{\calT_i^{(0)}\}_{i=1}^r\).
\begin{algorithmic}[1]
\STATE Compute the residual $\mathrm{res} \gets \sum_{j=1}^r \calA(\calT_j^{(0)}) - \calY$.
\FOR{$t = 0,1,\dots,t_{\max}-1$}
  \FOR{$i = 1,\dots,r$}
    \STATE \textbf{(RGN Update)} Solve for $\Delta_i^{(t)} \in \TT_i^{(t)}$:
    \vspace{-1ex}
    \[
      \left(\left.\calP_{\TT_i^{(t)}}\calA^*\calA\calP_{\TT_i^{(t)}}\right|_{\TT_i^{(t)}}\right)\Delta_i^{(t)} = \calP_{\TT_i^{(t)}}\calA^*(\mathrm{res}),
    \]
    \STATE \textbf{(Retraction)} $\calT_i^{(t+1)} \gets \mathcal{R}_{\calT_i^{(t)}} (-\Delta_i^{(t)})$.
    \STATE Update $\mathrm{res}\gets\mathrm{res}+\calA(\calT_i^{(t+1)}-\calT_i^{(t)})$.
  \ENDFOR
\ENDFOR
\end{algorithmic}
\textbf{Output:} $\{\calT_i^{(t_{\max})}\}_{i=1}^r$.
\end{algorithm}

\vspace{-2ex}
\section{Theoretical Analysis}\label{sec:theory}
\vspace{-2ex}

\subsection{Convergence Analysis of Riemannian Optimization}
\vspace{-2ex}

In this subsection, we establish a deterministic convergence theory for both RGD and RGN, stated in Theorems~\ref{thm:local_convergence_rgd} and~\ref{thm:local_convergence_rgn} for one component update. Even in the presence of noise, these bounds control the distance between each updated component and its true rank-one component by exploiting the intrinsic geometry of the Segre manifold. Appendix~\ref{sec:proof_cor} composes them over each sequential pass, and the resulting local convergence guarantees are stated in Section~\ref{subsec:applications}. Throughout this section, $d\geq3$ and $\lambda_1\geq\cdots\geq\lambda_r>0$.

For a current collection $\{\widehat{\calT}_j\}_{j=1}^r$, let $\widehat{\TT}_i$ be the tangent space at $\widehat{\calT}_i$. Whenever the restricted normal operator $\calP_{\widehat{\TT}_i}\calA^*\calA\calP_{\widehat{\TT}_i}\big|_{\widehat{\TT}_i}$ is invertible, define
\(
\widehat{\mathsf{C}}_{i}
:=(\calP_{\widehat{\TT}_i} \calA^* \calA \calP_{\widehat{\TT}_i}\big|_{\widehat{\TT}_i})^{-1}\calP_{\widehat{\TT}_i}
\calA^* \calA \calP_{\widehat{\TT}_i}^{\perp}.
\)
This operator captures the \emph{coupling} between the normal space $\widehat{\TT}_i^{\perp}$ and the tangent space $\widehat{\TT}_i$ induced by the normal operator $\calA^*\calA$: it maps an off-tangent perturbation into $\widehat{\TT}_i$, normalized by the local action of $\calA^*\calA$ restricted to the tangent space.
To quantify the worst-case strength of this coupling over all rank-one directions (the Segre manifold), we introduce the Segre-restricted operator norm
\(
\|\mathsf{C}\|_{\mathrm{Seg}\to \mathrm{F}}
:=\sup_{V\in \operatorname{Seg}, \|V\|_{\mathrm{F}}=1}\ \|\mathsf{C}V\|_{\mathrm{F}}.
\)

For either method, let $\widehat\xi_i\in\widehat{\TT}_i$ denote its tangent increment, namely $\widehat\xi_i=-\alpha_t\calP_{\widehat{\TT}_i}\calA^*\{\calA(\sum_{j=1}^r\widehat{\calT}_j)-\calY\}$ for RGD and $\widehat\xi_i=-\widehat\Delta_i$ for RGN, where $\widehat\Delta_i$ is the RGN normal-equation solution.

\begin{theorem}[Local Convergence of RGD]\label{thm:local_convergence_rgd}
Consider one component update with step size $\alpha_t>0$ from the current collection $\{\widehat{\calT}_j\}_{j=1}^r$, let $\widehat{\calT}_i^+$ denote the updated $i$-th component, and let $\varepsilon=\max_{j\in[r]}\|\widehat{\calT}_j-\calT_j\|_{\mathrm F}/\lambda_j$ denote the largest relative Frobenius error of the current components, where $\lambda_j$ is the weight of the $j$-th CP component. Let $\kappa=\lambda_1/\lambda_r$ be the condition number and $\eta$ the incoherence parameter from Assumption~\ref{assumption:incoherence}. Suppose that $\varepsilon\leq1/(4d)$ and \mbox{$\max_{i\in[r]}\|\widehat\xi_i\|_{\mathrm F}/\|\widehat{\calT}_i\|_{\mathrm F}<1/\{2(\sqrt d+1)\}$}. Then, for every $i\in[r]$,
{\small
\begin{align*} 
\frac{\|\widehat{\calT}_i^+-\calT_i\|_{\mathrm F}}{\lambda_i} \leq&  \underbrace{C_1 \max_{k\in[r]}
\left\|
\calP_{\widehat{\TT}_k}(I-\alpha_t\calA^*\calA)
\calP_{\widehat{\TT}_k}
\right\|
\varepsilon}_{\text{first-order contraction}} + \underbrace{C_1 r\alpha_t\kappa
\max_{\substack{k,j\in[r]\\ k\ne j}}
\left\|
\calP_{\widehat{\TT}_k}\calA^*\calA
\calP_{\widehat{\TT}_k}^\perp
\calP_{\widehat{\TT}_j}
\right\|
\varepsilon}_{\text{first-order contraction}} \\ 
&+ \underbrace{C_2 \Big[ 1
+ r\alpha_t \kappa\max_{k \in [r]} 
\sup_{\substack{V\in\operatorname{Seg}\\ \|V\|_{\mathrm F}=1}}
\big\|
\calP_{\widehat{\TT}_k}\calA^*\calA
\calP_{\widehat{\TT}_k}^\perp V
\big\|_{\mathrm F}
\Big] \varepsilon^2}_{\text{second-order contraction}}
+\underbrace{\frac{\|\widehat\xi_i\|_{\mathrm F}^2}
{\lambda_i\|\widehat{\calT}_i\|_{\mathrm F}}}_{\text{retraction remainder}} \\
&+  \underbrace{C_2 r\alpha_t\kappa \max_{k \in [r]}\big\|\calP_{\widehat{\TT}_k} \calA^* \calA\calP_{\widehat{\TT}_k}\big\|  \big[\eta^{d-2} + \varepsilon\big]\varepsilon}_{\text{cross-component interaction}} %
+    \underbrace{C_1 \alpha_{t} \max_{k \in [r]} \frac{\big\|\calP_{\widehat{\TT}_k} (\calA^*\calE )\big\|_{\mathrm{F}}}{ \lambda_k}}_{\text{noise term}},
\end{align*}}
where $C_1=1$ and $C_2=4d$.
\end{theorem}

\begin{theorem}[Local Convergence of RGN] \label{thm:local_convergence_rgn}
With the notation of Theorem~\ref{thm:local_convergence_rgd}, suppose that $\calP_{\widehat{\TT}_k}\calA^*\calA\calP_{\widehat{\TT}_k}\big|_{\widehat{\TT}_k}$ is invertible for every $k\in[r]$ and that the conditions on $\varepsilon$ and $\widehat\xi_i$ in Theorem~\ref{thm:local_convergence_rgd} hold. Then the RGN update satisfies, for every $i\in[r]$,
{\begin{align*}
\frac{\|\widehat{\calT}_i^+-\calT_i\|_{\mathrm F}}{\lambda_i}
    \leq & \underbrace{C_1r\kappa \max _{k \neq j, k, j \in[r]}\big\| \widehat{\mathsf C}_{k} \calP_{\widehat{\TT}_j}\big\|  \varepsilon}_{\text{first-order contraction}} 
+ \underbrace{C_1\max_{k \in[r]} \frac{\big\|\big(\calP_{\widehat{\TT}_k} \calA^* \calA \calP_{\widehat{\TT}_k}\big|_{\widehat{\TT}_k}\big)^{-1} \calP_{\widehat{\TT}_k}\calA^* (\calE )\big\|_{\mathrm{F}}}{\lambda_k}}_{\text{noise term}} \\
+ & \underbrace{C_2 r\kappa
\varepsilon
\left[
\eta^{d-2}
+
\varepsilon
\right]}_{\text{cross-component interaction}} + \underbrace{C_2 \left[
1+
\{1+r\kappa\}
\max_{k\in[r]}
\|\widehat{\mathsf C}_k\|_{\operatorname{Seg}\to \mathrm F}
\right]
\varepsilon^2}_{\text{second-order contraction}}
+\underbrace{\frac{\|\widehat\xi_i\|_{\mathrm F}^2}
{\lambda_i\|\widehat{\calT}_i\|_{\mathrm F}}}_{\text{retraction remainder}},
\end{align*}}
where $C_1=1$ and $C_2=4d$.
\end{theorem}

\begin{remark}\label{remark:sign-alignment}
Sign alignment follows from the local error bound. If $\|\widehat{\mathcal T}_i-\mathcal T_i\|_{\mathrm F}\leq\varepsilon\lambda_i$ for some $\varepsilon<1$, then $\|\widehat{\mathcal T}_i\|_{\mathrm F}\geq(1-\varepsilon)\lambda_i$ and the polarization identity gives
\[
2\langle\widehat{\mathcal T}_i,\mathcal T_i\rangle
=\|\widehat{\mathcal T}_i\|_{\mathrm F}^2+\lambda_i^2
-\|\widehat{\mathcal T}_i-\mathcal T_i\|_{\mathrm F}^2
\geq2(1-\varepsilon)\lambda_i^2>0.
\]
Thus the local condition $\varepsilon\leq1/(4d)$ in Theorems~\ref{thm:local_convergence_rgd}--\ref{thm:local_convergence_rgn} implies sign alignment at every component update. By the preceding display, an even number of factor inner products are negative. Flipping their signs is an element of $G$ in Remark~\ref{remark:segre-quotient} and leaves $\widehat{\mathcal T}_i$ unchanged, so the factorwise alignment in the proofs is a normalization. The initialization condition of Corollaries~\ref{corollary:rate_RGD_decomposition}--\ref{corollary:rate_RGN_regression} does not depend on $p_1,\ldots,p_d$, and Appendix~\ref{appendix:cpca-initialization} gives conditions under which CPCA satisfies it in CP decomposition.
\end{remark}

While both RGD and RGN have a first-order error term proportional to $\varepsilon$, RGN incorporates curvature information and attains a quadratic local phase. The scale at which its quadratic phase ends is governed by $\eta^{d-2}$ and $\|\widehat{\mathsf{C}}_{i}\calP_{\widehat{\TT}_j}\|$, which enter its first-order terms, while $\|\widehat{\mathsf{C}}_{i}\|_{\mathrm{Seg}\to \mathrm{F}}$ enters its second-order term. In CP decomposition, these coupling norms vanish exactly, so RGN has a quadratic phase above the $\eta^{d-2}$ scale followed by linear refinement, and it is purely quadratic when $\eta=0$. In CP tensor-on-tensor regression, they are controlled by the random design: under the Gaussian design of Section~\ref{subsec:applications}, $\calA^*\calA$ acts approximately as $\sigma_X^2$ times the identity on the relevant low-dimensional tangent and normal subspaces, so the tangent--normal cross interactions remain small with high probability, giving the desired contraction.

\begin{remark}
Recent work has substantially advanced our theoretical understanding of ALS, which remains the standard workhorse for CP decomposition \citep{kolda2009tensor}. \citet{hu2025convergence} established a polynomial convergence rate (of order $d-1$ for order $d$ tensors) for standard CP-ALS on orthogonally decomposable tensors, together with linear convergence for incoherently decomposable tensors. \citet{tang2025revisit} derived statistical error bounds for ALS and identified a two-phase trajectory, an initial quadratic regime followed by a linear regime, that mirrors the local behavior of our RGN method. Beyond CP decomposition, existing results for CP regression concern global or penalized estimators \citep{lock2018tensor,zhou2021tensorlinear} or objective-preserving sketches \citep{li2017near}, and the convergence of concrete iterates to the statistical error under a general linear operator has remained largely unexplored. By optimizing directly over the Segre manifold, our methods provide such guarantees in a unified geometric framework that accommodates different tasks through the choice of the forward operator.
\end{remark}

\vspace{-1ex}
\subsection{Implications in Statistics and Machine Learning}\label{subsec:applications}
\vspace{-1ex}

We now specialize the convergence theory to two concrete problems: CP tensor decomposition and CP tensor-on-tensor regression, after matching component labels at initialization. More general versions of the RGD corollaries below, with variable step sizes, are given in Appendix~\ref{sec:proof_cor}, and the computational complexity analysis is provided in Appendix~\ref{sec:complexity}.

Throughout this paper, we set $p^*=\prod_{l=1}^d p_l, p_{\rm in}^*=\prod_{\ell=1}^{d_{\rm in}}p_\ell, p_{\rm out}^*=\prod_{\ell=d_{\rm in}+1}^{d}p_\ell$ with the convention \(p_{\rm out}^*=1\) when \(d_{\rm out}=0\). The tangent dimensions of a rank-one CP component are $\mathrm{df}
=1+\sum_{\ell=1}^d(p_\ell-1)$ and $\mathrm{df}_{\rm in}=1+\sum_{\ell=1}^{d_{\rm in}}(p_\ell-1)$. We write $\kappa = \lambda_1 / \lambda_r$ for the condition number and, for iterates $\{\calT_i^{(t)}\}_{i=1}^r$, $\varepsilon^{(t)} = \max_{i \in [r]} \|\calT_i^{(t)}-\calT_i\|_{\mathrm{F}}/\lambda_i$.

\textbf{Tensor Decomposition.} Consider the noisy CP decomposition model
\[
\calY = \calT + \calE
\;\in\;\RR^{p_1\times\cdots\times p_d},
\]
where $\calT =\sum_{i=1}^r \lambda_i u_{1,i}\otimes\cdots\otimes u_{d,i}$ is as defined in \eqref{cp-tensor} and ${\rm vec}(\calE)\sim\mathcal{N}\bigl(0,\Sigma_{p^*}\bigr)$, with noise covariance $\underline{\sigma}^{2} I_{p^*} \preceq \Sigma_{p^*} \preceq \overline{\sigma}^{2} I_{p^*}$.
Under the incoherence assumption and a sufficiently accurate initialization, we obtain the following convergence guarantees.

\begin{corollary}[Convergence of RGD for CP tensor decomposition]\label{corollary:rate_RGD_decomposition}
There exist universal constants $c,C>0$ such that the following holds. Suppose that $\varepsilon^{(0)}\leq c/\{d(1+\kappa r)\}$ and $\overline\sigma\sqrt{\mathrm{df}}/\lambda_r\leq c/\{d(1+\kappa r)\}$, and that, for every $t\geq0$, $0<\alpha_t\leq1$, $|1-\alpha_t|\leq c$ and $\alpha_t\eta^{d-2}\leq c/(d\kappa r)$. Then, with probability at least $1-2\exp(-c\,\mathrm{df})$, simultaneously for all $t\geq0$,
\[
\varepsilon^{(t)}
\leq
2^{-t}\varepsilon^{(0)}
+C\frac{\overline\sigma\sqrt{\mathrm{df}}}{\lambda_r}.
\]
\end{corollary}

\begin{corollary}[Convergence of RGN for CP tensor decomposition] \label{cor:RGN1}
Under the assumptions of Corollary~\ref{corollary:rate_RGD_decomposition} with $\alpha_t=1$, the RGN iterates satisfy the same linear bound. If, in addition, $\overline\sigma\sqrt{\mathrm{df}}/\lambda_r\leq c d(1+\kappa r)\eta^{2(d-2)}$, then, with probability at least $1-2\exp(-c\,\mathrm{df})$, RGN has a quadratic-to-linear local trajectory. More precisely, for universal constants $C_3,C_4>0$ and \mbox{$t^*=\inf\{t\geq0:\varepsilon^{(t)}\leq C_4\eta^{d-2}\}$},
\[
\varepsilon^{(t)}
\leq
\begin{cases}
\dfrac{\{C_3d(1+\kappa r)\varepsilon^{(0)}\}^{2^t}}{C_3d(1+\kappa r)}, & 0\leq t<t^*,\\[2ex]
2^{-(t-t^*)}\varepsilon^{(t^*)}+C\dfrac{\overline\sigma\sqrt{\mathrm{df}}}{\lambda_r}, & t\geq t^*.
\end{cases}
\]
\end{corollary}

Theorems~\ref{thm:local_convergence_rgd}--\ref{thm:local_convergence_rgn} are deterministic. The corollaries impose the stated noise and sample-size conditions to keep their recursions inside the local basin.

\textbf{Tensor-on-Tensor Regression.} For \(d_{\rm in}\in\{1,\ldots,d\}\) and \(d_{\rm out}=d-d_{\rm in}\), we observe
\[
\calY_m
=
\langle \calX_m,\calT\rangle_{\rm in}
+
\calE_m,
\qquad m=1,\ldots,n,
\]
where \(\calX_m\in\RR^{p_1\times\cdots\times p_{d_{\rm in}}},
\calY_m,\calE_m\in
\RR^{p_{d_{\rm in}+1}\times\cdots\times p_d}\), with \(\langle \calX_m,\calT\rangle_{\rm in}\) denoting contraction over the first \(d_{\rm in}\) modes of \(\calT\). We assume that the design tensors \(\{\calX_m\}_{m=1}^n\) are i.i.d. centered Gaussian with
$\operatorname{Cov}(\operatorname{vec}(\calX_m))=\sigma_X^2 I_{p_{\rm in}^*}$, where $\sigma_X>0$, and that the noise tensors are mutually independent centered Gaussian, independent of the designs, with covariance satisfying \(\underline\sigma_\xi^2 I_{p_{\rm out}^*} \preceq \operatorname{Cov}\{\operatorname{vec}(\calE_m)\} \preceq \overline\sigma_\xi^2 I_{p_{\rm out}^*}\). A sub-Gaussian extension requires the corresponding local matrix-deviation and Hanson--Wright bounds. A tensor restricted isometry property \citep[Definition 1 and Proposition 1]{luo2024tensor} provides a standard sufficient framework. In the Gaussian corollaries below,
$\gamma=C\sqrt{(\mathrm{df}_{\rm in}+\log r)/n}$ is the restricted-design scale. The projected-noise term depends instead on the full tangent dimension $\mathrm{df}$.

\begin{remark}
Under the Gaussian design, the isotropy assumption $\operatorname{Cov}\{\operatorname{vec}(\calX_m)\}=\sigma_X^2 I_{p_{\rm in}^*}$ is equivalent to requiring the entries of each $\calX_m$ to be i.i.d. with variance $\sigma_X^2$. For a general covariance $\Sigma_X=\operatorname{Cov}\{\operatorname{vec}(\calX_m)\}$, the operators $\calA,\calA^*$ remain as defined in Section~\ref{sec:model}. Extending the statistical guarantees requires corresponding restricted conditioning bounds for $\calA^*\calA$. Estimating or preconditioning $\Sigma_X$ requires additional structural assumptions and is beyond the present scope.
\end{remark}

\begin{corollary}[Convergence of RGD for CP tensor-on-tensor regression]\label{corollary:rate_RGD_regression}
There exist universal constants $c,C>0$ such that the following holds. Suppose that $n\geq C\,\mathrm{df}$, $\gamma\leq c/(1+\kappa r)$, $\varepsilon^{(0)}\leq c/\{d(1+\kappa r)\}$ and $\{\overline\sigma_\xi/(\sigma_X\lambda_r)\}\sqrt{\mathrm{df}/n}\leq c/\{d(1+\kappa r)\}$, and that, for every $t\geq0$, $0<\alpha_t\sigma_X^2\leq1$, $|1-\alpha_t\sigma_X^2|\leq c$ and $\alpha_t\sigma_X^2\eta^{d-2}\leq c/(d\kappa r)$. Then, with probability at least
$1-C\exp\{-c(\mathrm{df}_{\rm in}+\log r)\}$, simultaneously for all $t\geq0$,
\[
\varepsilon^{(t)}
\leq
2^{-t}\varepsilon^{(0)}
+C
\frac{\overline\sigma_\xi}{\sigma_X\lambda_r}
\sqrt{\frac{\mathrm{df}}{n}}.
\]
\end{corollary}

\begin{corollary}[Convergence of RGN for CP tensor-on-tensor regression] \label{corollary:rate_RGN_regression}
Under the assumptions of Corollary~\ref{corollary:rate_RGD_regression} with $\alpha_t\sigma_X^2=1$, the RGN iterates satisfy the same linear bound. Let \mbox{$h=\gamma/d+\eta^{d-2}$}. If, in addition, $\{\overline\sigma_\xi/(\sigma_X\lambda_r)\}\sqrt{\mathrm{df}/n}\leq c d(1+\kappa r)h^2$, then, with probability at least \mbox{$1-C\exp\{-c(\mathrm{df}_{\rm in}+\log r)\}$}, for universal constants $C_3,C_4>0$ and \mbox{$t^*=\inf\{t\geq0:\varepsilon^{(t)}\leq C_4h\}$},
\[
\varepsilon^{(t)}
\leq
\begin{cases}
\dfrac{\{C_3d(1+\kappa r)\varepsilon^{(0)}\}^{2^t}}{C_3d(1+\kappa r)}, & 0\leq t<t^*,\\[2ex]
2^{-(t-t^*)}\varepsilon^{(t^*)}+C\dfrac{\overline\sigma_\xi}{\sigma_X\lambda_r}\sqrt{\dfrac{\mathrm{df}}{n}}, & t\geq t^*.
\end{cases}
\]
\end{corollary}

\begin{remark}\label{remark:rgn-phases}
The two-phase behavior in Corollaries~\ref{cor:RGN1} and~\ref{corollary:rate_RGN_regression} is driven by the normalized error recursion $\varepsilon^{(t+1)}\leq C d(1+\kappa r)\{\varepsilon^{(t)}\}^2+C\kappa r\{\gamma+d\eta^{d-2}\}\varepsilon^{(t)}$ plus the projected-noise term, where $\gamma=0$ for decomposition. While $\varepsilon^{(t)}$ remains above the transition scale, of order $\eta^{d-2}$ for decomposition and $\gamma/d+\eta^{d-2}$ for regression, the quadratic term dominates, yielding local quadratic convergence. Below this scale, the linear term dominates, and the error contracts linearly until it settles at the statistical error. The additional noise conditions in Corollaries~\ref{cor:RGN1} and~\ref{corollary:rate_RGN_regression} ensure that the quadratic phase reaches the transition scale. Without them, the same recursion holds but the quadratic phase can end at the statistical error level.
\end{remark}

\vspace{-2ex}
\section{Experiments and Results}\label{sec:experiment}
\vspace{-1ex}

\textbf{Simulation.} We evaluate the convergence behavior of RGD and RGN under varying incoherence and noise levels in CP decomposition of fourth-order tensors of size $60\times30\times30\times30$ and tensor-on-tensor regression with covariates $\calX_m\in\RR^{10\times10\times10}$, responses $\calY_m\in\RR^{10\times10}$ and $n=1000$. In both problems $r=3$, the weights have condition number $\kappa=30$, and the factors in each mode have maximal pairwise coherence $\rho\in\{0,0.25,0.5,0.75\}$. The noise level is set by the signal-to-noise ratio $s\in\{\infty,30,10,3\}$ of the weakest component. CP-ALS, RGD and RGN share the CPCA initialization, and all results are averaged over 20 replications. For regression we compare with CP-ALS and RRR \citep{lock2018tensor} in the CP format, and with Tucker-RGD \citep{luo2024tensor}, Tucker-RGN \citep{luo2023low}, TT-IHT \citep{qin2024computational} and HOPLS \citep{zhao2013hopls} in other formats. For decomposition we compare with CP-ALS \citep{kolda2009tensor}, Orth-ALS \citep{sharan2017orthogonalized}, ICO \citep{han2022tensor}, NLS-GN \citep{sorber2013optimization}, GN-CG \citep{singh2021comparison}, CP-ALS with enhanced line search (CP-ALS+ELS) \citep{rajih2008enhanced} and Hybrid AMDM \citep{singh2023alternating}. Full experimental details, including the coherence construction and the initialization, are given in Appendix~\ref{appendix:simulation}.

Figure~\ref{fig:tensor_regression_convergence_main} shows that, with noise, CP-ALS, RGD and RGN reach the same error level, whereas RRR, started at random, often does not within 20 iterations. When $\rho\le0.5$ and $s\ge10$, RGN needs no more iterations than CP-ALS and RGD, and at $\rho=0.75$ the three methods converge linearly, as described in Remark~\ref{remark:rgn-phases}. The methods in other tensor formats reach higher error levels with noise, and Tucker-RGN converges fastest without noise at $\rho=0.75$. In CP decomposition with noise, RGN reaches the same error level as CP-ALS (Figure~\ref{fig:decomposition_convergence}), whereas Orth-ALS and ICO do not when $\rho>0$ and $s\geq10$. ICO and Hybrid AMDM need the fewest iterations without noise, and NLS-GN and GN-CG need the fewest when $s\ge10$. RGD and RGN have the lowest cost per iteration in both problems (Table~\ref{tab:timing}).

\textbf{Real data.} Table~\ref{tab:trade_cv_metrics} reports five-fold cross-validation results for a tensor-on-tensor regression of 2019 export outcomes on 2018 trade-cost and market covariates of 80 exporters (Appendix~\ref{appendix:real_data}). All methods agree within one standard error. RGD and TT-IHT have not converged after 30 iterations, and their slightly lower test errors disappear at convergence (Table~\ref{tab:trade_cv_paths}).

\begin{figure}[t]
  \centering
  \includegraphics[width=\textwidth]{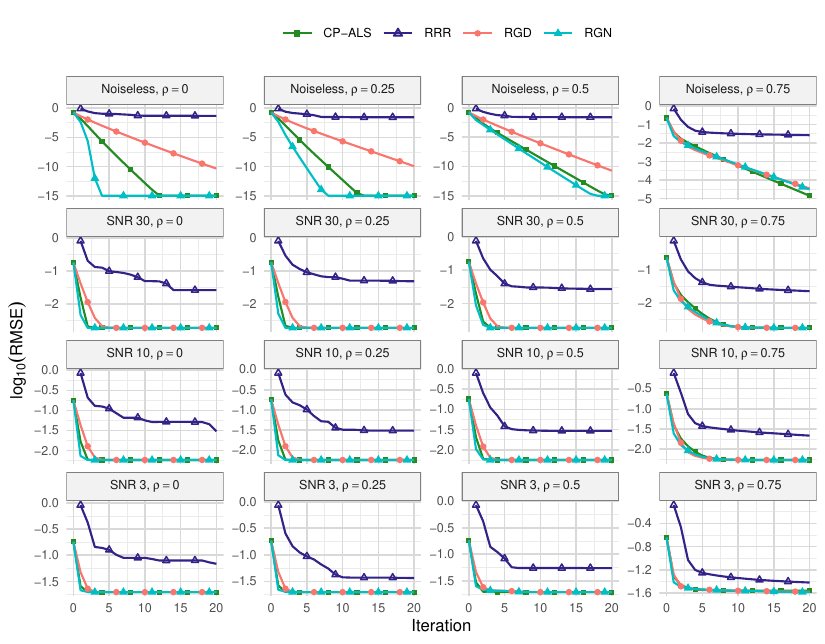}
  \caption{Convergence for CP tensor-on-tensor regression with rank $r=3$. Each curve is the root mean square, over 20 replications, of the relative Frobenius error. RRR uses a random start, and errors above 1 are omitted. Errors below $10^{-15}$ are drawn at $10^{-15}$. Point symbols are drawn at staggered iterations so that coinciding curves remain distinguishable.}
  \label{fig:tensor_regression_convergence_main}
\end{figure}

\begin{table}[b]
\centering
\caption{Five-fold international-trade regression after 30 iterations: mean${}\pm{}$SE of the held-out test RMSE and test $R^2$ over the five folds. Bold: lowest mean test RMSE in each group.}
\label{tab:trade_cv_metrics}
{\small
\setlength{\tabcolsep}{4pt}
\renewcommand{\arraystretch}{0.92}
\begin{tabular}{llcc}
\toprule
Method & Model & Test RMSE & Test $R^2$ \\
\midrule
\multicolumn{4}{l}{\textit{CP format}} \\
CP-ALS     & CP rank $=3$ & $2.6747\pm0.0286$ & $0.1313\pm0.0390$ \\
RRR        & CP rank $=3$ & $2.6763\pm0.0279$ & $0.1303\pm0.0387$ \\
RGD        & CP rank $=3$ & $\mathbf{2.6657\pm0.0278}$ & $0.1372\pm0.0379$ \\
RGN        & CP rank $=3$ & $2.6743\pm0.0286$ & $0.1316\pm0.0391$ \\
\midrule
\multicolumn{4}{l}{\textit{Other tensor formats}} \\
Tucker-RGD & Tucker rank $=(3,3,3,3)$ & $2.6748\pm0.0269$ & $0.1314\pm0.0376$ \\
Tucker-RGN & Tucker rank $=(3,3,3,3)$ & $2.6746\pm0.0276$ & $0.1315\pm0.0379$ \\
TT-IHT     & TT rank $=(3,3,3)$ & $\mathbf{2.6663\pm0.0292}$ & $0.1367\pm0.0388$ \\
HOPLS      & 3 latent blocks & $2.6795\pm0.0391$ & $0.1297\pm0.0303$ \\
\bottomrule
\end{tabular}
}
\end{table}

\section{Discussion and Future Extensions}\label{sec:conclusion}
\vspace{-1ex}

In this paper, we have proposed a unified framework with local convergence guarantees for noisy low-CP-rank tensor recovery from general linear measurements, encompassing both tensor decomposition and tensor-on-tensor regression. To the best of our knowledge, this is the first framework with algorithm-level guarantees for CP tensor-on-tensor regression. The framework offers several practical advantages: it seamlessly handles a broad class of linear measurement operators, and because each component update acts on a single Segre factor, the algorithms admit streaming implementations well-suited to large-scale or time-evolving tensor data. While we use fixed step sizes here, adaptive schemes or Riemannian momentum could further accelerate convergence.

Several extensions remain open. First, our analysis assumes that the CP rank is known exactly, and addressing rank selection and misspecification is an important next step. Second, each iteration involves retractions and tangent-space projections, which may become costly in ultra-high dimensions, and efficient approximations or randomized updates are a promising direction.

\section*{Code and Data Availability}
The R package \texttt{SegreCP}, which implements RGD and RGN, and the scripts that reproduce all experiments are available at \url{https://github.com/kxu-stat/SegreCP}. The trade data are compiled from public sources (CEPII BACI and Gravity, World Bank WDI and LPI, and WITS/UNCTAD TRAINS), and the repository includes scripts to download and process them.

\begin{ack}
We thank the area chair and the anonymous reviewers for helpful suggestions. The work of Yuefeng Han is supported in part by National Science Foundation Grant DMS-2412578. The authors declare no competing interests.
\end{ack}

\bibliography{citations}
\bibliographystyle{plainnat}

\newpage
\appendix
\onecolumn

\section*{\Large Appendices}

\section{Notation}\label{appendix:notation}

Throughout this paper, we use the following notation and conventions.
We use calligraphic uppercase letters (e.g. $\calT,\calX$) for tensors, uppercase letters (e.g. $A, U$) for matrices and, in Section~\ref{sec:theory} and the appendices, also for tensors (e.g. $V,W$), and lowercase letters (e.g. $u,v$) for vectors or scalars. For any positive integer $m$, let $[m] = \{1, 2, \dots, m\}$. We consider order-$d$ tensors with mode dimensions $p_1, p_2, \dots, p_d$, so that $\calT \in \RR^{p_1 \times \cdots \times p_d}$ contains $p^* = \prod_{l=1}^d p_l$ total entries. The CP rank is denoted by $r$, and the sample size in regression contexts is denoted by $n$.

The vectorization of a tensor $\calT$ is denoted by $\operatorname{vec}(\calT)$. The mode-$k$ unfolding (matricization) is $\mat_k(\calT)\in\RR^{p_k\times(p^* /p_k)}$. The outer (tensor) product is written $\otimes$. For vectors, it induces a rank-one tensor. The multilinear (Tucker) product of $\calT$ with matrices $U_k\in\RR^{q_k\times p_k}$ is $\calT\times_1 U_1\times_2\cdots\times_d U_d$. The $k$-mode product with $U$ alone is $\calT\times_k U$. The tensor inner product is $\langle\calT,\calT'\rangle=\sum_{i_1,\ldots,i_d} \calT_{i_1\dots i_d}\calT'_{i_1\dots i_d}$. The induced Frobenius norm is $\|\calT\|_{\mathrm{F}}=\sqrt{\langle\calT,\calT\rangle}$. For matrices, $\|\cdot\|_{\mathrm{F}}$ and $\|\cdot\|$ denote the Frobenius and spectral norms. For vectors, $\|\cdot\|=\|\cdot\|_2$ denotes the Euclidean norm. For a linear operator $\mathcal M$, denote the operator norm of $\mathcal M$ by $
\|\mathcal M\| := \sup_{\|X\|_{\mathrm F}=1}\|\mathcal M(X)\|_{\mathrm F}$. For a subspace $\TT$ of tensors, $\calP_{\TT}$ is the orthogonal projector onto $\TT$, $\calP_{\TT}^{\perp}=I-\calP_{\TT}$, and $\mathcal M|_{\TT}$ is the restriction of $\mathcal M$ to $\TT$. The adjoint of a linear operator $\calA$ is $\calA^*$, and $I$ and $\mathrm{Id}$ denote identity maps. $\widetilde O(\cdot)$ hides logarithmic factors, and $\mathbb E$ and $\mathbb P$ denote expectation and probability. T-HOSVD is the truncated higher-order singular value decomposition \citep{de2000multilinear}. For a set $S$ and a norm $\|\cdot\|$, $N(\epsilon,S,\|\cdot\|)$ is the $\epsilon$-covering number of $S$. With $\mathrm{df}$ as in Section~\ref{subsec:applications}, the CP rank satisfies $r\leq\min_{k\in[d]}\prod_{\ell\ne k}p_\ell$, so $\log r\leq\mathrm{df}-1$.

Let $\SS^{p-1}$ denote the unit sphere in $\RR^p$. For a unit vector $u \in \RR^p$, define the projector onto its span as $\calP_u = uu^\top$ and its orthogonal complement as $\calP_u^\perp = I_p - uu^\top$. For rank-one tensors, these projections are applied in a modewise manner. The set of nonzero rank-one tensors of the form $u_1 \otimes \cdots \otimes u_d$, where each $u_k \in \RR^{p_k} \setminus \{0\}$, forms the Segre manifold. Its geometric structure, including tangent spaces, Riemannian gradients, and retraction maps, is further discussed in Section~\ref{subsec:RGD}.

\section{Computational Complexity} \label{sec:complexity}

Denote $p^*=\prod_{l=1}^d p_l$ and let $r$ be the CP rank. For tensor-on-tensor regression, define \(p_{\rm in}^*=\prod_{\ell=1}^{d_{\rm in}}p_\ell\), \(p_{\rm out}^*=\prod_{\ell=d_{\rm in}+1}^{d}p_\ell\), $\mathrm{df}_{\rm in}=1+\sum_{\ell=1}^{d_{\rm in}}(p_\ell-1)$, and $\mathrm{df}_{\rm out}=1+\sum_{\ell=d_{\rm in}+1}^{d}(p_\ell-1)$. The rank-one Segre tangent dimension is $\mathrm{df}=\mathrm{df}_{\rm in}+\mathrm{df}_{\rm out}-1=1+\sum_{\ell=1}^d(p_\ell-1)$.

Table~\ref{tab:complexity} summarizes the leading per-iteration costs of all methods in Section~\ref{sec:experiment} for fixed rank and order and, in regression, under the conditions given after the table. One iteration is a full alternating update or one outer Riemannian step. HOPLS instead extracts latent blocks with inner HOOI iterations, so no per-iteration cost is given for it. The costs of RGD and RGN refer to the structure-exploiting implementations described below, and their lower-order terms are given after the table.

\begin{table}[h!]
\centering
\caption{Leading per-iteration computational costs. For NLS-GN, $m$ counts the additional residual evaluations. For RRR, $k$ runs over the output modes and $\ell$ over the input modes.}
\label{tab:complexity}
\small\setlength{\tabcolsep}{4pt}
\begin{tabular}{ll@{\qquad}ll}
\toprule
\multicolumn{2}{c}{CP decomposition} & \multicolumn{2}{c}{CP tensor-on-tensor regression}\\
\cmidrule(r){1-2}\cmidrule(l){3-4}
Method & Cost & Method & Cost\\
\midrule
CP-ALS & $O(drp^*)$ & CP-ALS & $O\{dnr(p_{\rm in}^*+p_{\rm out}^*)\}$\\
Orth-ALS & $O(drp^*)$ & RRR & $O\{nr(p^*\sum_{k}p_k^{-1}+p_{\rm out}^*\sum_{\ell}p_\ell)\}$\\
ICO & $O(drp^*)$ & Tucker-RGD & $O(np^*)$\\
NLS-GN & $O\{(d+m)rp^*\}$ & Tucker-RGN & $O(nrp^*)$\\
GN-CG & $O(drp^*)$ & TT-IHT & $O(np^*)$\\
CP-ALS+ELS & $O(2^drp^*)$ & HOPLS & N/A\\
Hybrid AMDM & $O(drp^*)$ & & \\
RGD & $O(drp^*)$ & RGD & $O\{rn(d_{\rm in}p_{\rm in}^*+p_{\rm out}^*)\}$\\
RGN & $O(drp^*)$ & RGN & $O\{rn(d_{\rm in}p_{\rm in}^*+p_{\rm out}^*)\}$\\
\bottomrule
\end{tabular}
\end{table}

The line search of CP-ALS+ELS evaluates $2^dr$ inner products with $\calY$. The RRR entry is the cost of the alternating updates of \citet{lock2018tensor}, which is at most $O(dnrp^*)$. Tucker-RGD and TT-IHT compute the full gradient of the least-squares loss, which costs $O(np^*)$, and the cost of Tucker-RGN follows from that of RGN for tensor regression in \citet[Table~1]{luo2023low}. The structured RGD implementation has cost
\(C_{\rm RGD}=O(r[n\{d_{\rm in}p_{\rm in}^*+p_{\rm out}^*\}+p_{\rm out}^*\mathrm{df}_{\rm out}])\). RGN additionally forms the input-tangent Gram matrix and solves its coupled block, so
\(C_{\rm RGN}=C_{\rm RGD}+O(r\{n\mathrm{df}_{\rm in}^2+\mathrm{df}_{\rm in}^3\})\). Under $n\gtrsim\mathrm{df}_{\rm out}$, $d_{\rm in}p_{\rm in}^*+p_{\rm out}^*\gtrsim\mathrm{df}_{\rm in}^2$, and $n(d_{\rm in}p_{\rm in}^*+p_{\rm out}^*)\gtrsim\mathrm{df}_{\rm in}^3$, these costs reduce to the regression terms in the table. For scalar-on-tensor regression, $p_{\rm out}^*=\mathrm{df}_{\rm out}=1$.

\paragraph{CP decomposition.} Classical ALS updates each of the $d$ factor matrices in turn. For a fixed mode $\ell$, the matricized tensor times Khatri--Rao product (MTTKRP) costs \(O(rp^*)\) for a dense tensor. The Gram matrices and the associated \(r\times r\) normal equations contribute
\(O(r^2\sum_{\ell=1}^d p_\ell)\) per ALS iteration, while the $d$ factorizations of the $r\times r$ normal matrices cost \(O(dr^3)\). Summing over all modes gives $O(drp^*+r^2\sum_{\ell=1}^d p_\ell+dr^3)$.

For RGD, a fused pass through the residual computes the modewise tangent contractions in $O(dp^*)$ operations per component. The exact rank-$(1,\ldots,1)$ T-HOSVD of the resulting rank-one-plus-tangent proposal is obtained from one $2\times2$ Gram eigenproblem per mode. Residual and cache updates cost $O(rp^*)$, so the total per-iteration cost is $O(drp^*)$.

RGN for CP decomposition, where the measurement operator $\mathcal{A}$ is the identity mapping, becomes equivalent to an RGD step with a unit step size ($\alpha_t=1$), and thus has an identical per-iteration cost of $O (r d p^* )$.

\paragraph{CP tensor-on-tensor regression.} %

The regression implementation contracts the design modewise and forms the exact input and output tangent sufficient statistics without constructing an ambient $p^*\times\mathrm{df}$ basis. The input contractions cost $O(nd_{\rm in}p_{\rm in}^*)$, while the response and output-tangent contractions cost $O(np_{\rm out}^*+p_{\rm out}^*\mathrm{df}_{\rm out})$, per component. These give $C_{\rm RGD}$ above. RGN uses the same contractions, solves the coupled $\mathrm{df}_{\rm in}\times\mathrm{df}_{\rm in}$ block, and updates the orthogonal output coordinates by scalar division, yielding $C_{\rm RGN}$. The exact factorization underlying these costs is given in Appendix~\ref{appendix:algorithm_implementation}. For CP-ALS, updating the factor of one mode contracts the design and the response with the factors of the other modes at cost $O\{nr(p_{\rm in}^*+p_{\rm out}^*)\}$, so one sweep costs $O\{dnr(p_{\rm in}^*+p_{\rm out}^*)\}$. The normal equations of an input mode $\ell$ add $O(nr^2p_\ell^2+r^3p_\ell^3)$, which is of no larger order for fixed $r$ under the conditions above.

\section{Simulation Results}\label{appendix:simulation}

\paragraph{Incoherence condition.} To explore scenarios with non-orthogonal factors, we generate factor matrices whose columns achieve a prescribed level of pairwise coherence. Specifically, for a given coherence parameter $\rho \in [0,1)$ and rank $r$, we first construct the $r \times r$ Gram matrix $\Gamma$ with entries
\[
\Gamma_{ij} = \rho^{|i-j|}, \quad i,j \in [r],
\]
which corresponds to an autoregressive correlation structure of order one (AR(1)). We then compute the Cholesky factorization $\Gamma=LL^\top$ and embed $L^\top$ into $\mathbb{R}^p$ by stacking $r$ identity rows on top of $(p-r)$ zero rows, forming the matrix $(I_r,\,0)^\top \in \mathbb{R}^{p \times r}$. Multiplying this matrix with $L^\top$ yields vectors with the desired correlation pattern, and each column is normalized to unit length.  

Finally, to avoid artificial alignment with the coordinate axes, we multiply by a random orthogonal matrix, the orthogonal factor of the QR decomposition of a $p\times p$ matrix with i.i.d.\ standard normal entries. The resulting factor matrix thus has columns with controlled coherence. By varying $\rho$, we tune the similarity (``coherence'') between the factors: $\rho=0$ corresponds to orthogonal columns, whereas $\rho$ close to $1$ yields highly coherent columns. The incoherence parameter of Assumption~\ref{assumption:incoherence} is $\eta=\rho$.

\paragraph{Initialization.}

We employ Composite Principal Component Analysis (CPCA) \citep{han2022tensor} as a warm start: Algorithm~\ref{alg:cpca} for tensor decomposition and Algorithm~\ref{alg:init-low-rank-tensor} for tensor-on-tensor regression. CPCA constructs the initial CP basis vectors through a two-stage spectral procedure, first decomposing a matrix unfolding and then decomposing the refolded singular vectors.

\paragraph{Signal-to-noise ratio.} Let $\mathrm{df}=1+\sum_{\ell=1}^d(p_\ell-1)$ be the dimension of the tangent space of the Segre manifold. With noise standard deviation $\sigma$, we set $s=\lambda_r/(\sigma\sqrt{\mathrm{df}/n})$ for regression and $s=\lambda_r/(\sigma\sqrt{\mathrm{df}})$ for decomposition. Up to constants, $1/s$ is the noise term in the error bounds of Section~\ref{subsec:applications}.

\paragraph{Tensor-on-tensor regression.}
The coefficient tensor $\calT^\star=\sum_{i=1}^r\lambda_i\,u_{1,i}\otimes\cdots\otimes u_{5,i}\in\RR^{10\times10\times10\times10\times10}$ has rank $r=3$ and weights $\lambda_i=\kappa^{1/2-(i-1)/(r-1)}$ with $\kappa=30$. The covariates have i.i.d.\ standard normal entries, so $\sigma_X=1$, and $\calY_m=\langle\calX_m,\calT^\star\rangle_{\rm in}+\calE_m$, where $\calE_m$ has i.i.d.\ $\mathcal N(0,\sigma^2)$ entries. RGD uses the fixed step size $\alpha_t=\sigma_X^{-2}$, which satisfies the conditions $0<\alpha_t\sigma_X^2\leq1$ and $|1-\alpha_t\sigma_X^2|\leq1/2$ of Corollary~\ref{corollary:rate_regression_appendix}. RRR starts from its own random initialization, Tucker-RGD and Tucker-RGN from the T-HOSVD of $\calA^*(\calY)$, and TT-IHT from its TT-SVD. All methods run for 20 iterations. The Tucker methods use multilinear rank $(3,\ldots,3)$, TT-IHT uses TT rank $(3,\ldots,3)$, and HOPLS extracts three latent blocks. HOPLS is not iterative and is shown as a horizontal line. Figures~\ref{fig:tensor_regression_convergence_main} and~\ref{fig:regression_other_formats} show the convergence. Without noise at $\rho=0$, Tucker-RGN misses the weakest component in three of the 20 replications, which determines its root mean square error.

\paragraph{CP decomposition.}
The signal tensor $\calT^\star$ has the same CP form with dimensions $(60,30,30,30)$, rank $r=3$, and weights $\lambda_i=\sqrt{r}\,\kappa^{1/2-(i-1)/(r-1)}p_{\max}^{d/4}$ with $\kappa=30$ and $p_{\max}=60$. We observe $\calY=\calT^\star+\calE$, where $\calE$ has i.i.d.\ $\mathcal N(0,\sigma^2)$ entries. RGD uses step size $\alpha=0.5$, which satisfies the conditions $0<\alpha\leq1$ and $|1-\alpha|\leq1/2$ of Corollary~\ref{corollary:rate_decomposition_appendix}, and RGN uses unit step size (with $\alpha=1$, RGD coincides with RGN, Appendix~\ref{sec:complexity}). All methods start from CPCA and run for 100 iterations. Figures~\ref{fig:decomposition_convergence} and~\ref{fig:decomposition_second_order} show the first 20 iterations, or all 100 at $\rho=0.75$, where convergence is slower. With orthogonal factors ($\rho=0$), CPCA recovers the noiseless tensor exactly, so all methods start close to the error level in that column.

\paragraph{Running time.}\label{appendix:timing}
Table~\ref{tab:timing} reports the median running time per iteration over all settings, measured on one core of an Intel Xeon 2.1\,GHz CPU. All methods are implemented in R. RGD and RGN have the lowest cost per iteration in both problems.

\paragraph{Time to the error level.}
Tables~\ref{tab:time_regression} and~\ref{tab:time_decomposition} report the time to the error level in the runs of 100 iterations (in regression, of CP-ALS, RGD and RGN). In each replication, the level is the lowest error at iteration 100 among CP-ALS, RGD and RGN, adapting the convergence test of \citet{more2009benchmarking}, and a method reaches it at the first iteration at which its error is within $1\%$ of it. The time, a time-to-accuracy measure as in \citet{coleman2019analysis}, is the median over the 20 replications of this number of iterations times the median time per iteration of the method in the setting, from the timing study of Table~\ref{tab:timing}. In regression, RGD or RGN is the fastest in 11 of the 12 settings, and RGN is within a factor of 1.4 of the fastest method in every setting. In decomposition, RGN is the fastest in 5 of the 11 settings in which the methods do not all start at the level, NLS-GN in 3, CP-ALS in 2 and GN-CG in 1, and RGN is within a factor of 3.1 of the fastest method. RGN has the best average rank over all 12 settings in both problems, 1.4 and 2.0 \citep{demsar2006statistical}, and is faster than CP-ALS in 10 of the 12 regression and 9 of the 11 decomposition settings, by up to a factor of 2.5.

\paragraph{Performance profiles.}
Figure~\ref{fig:profiles} summarizes Tables~\ref{tab:time_regression} and~\ref{tab:time_decomposition} and the errors at the last iteration of Tables~\ref{tab:regression_errors} and~\ref{tab:decomposition_errors} as performance profiles \citep{dolan2002benchmarking}. For a method and a factor $\tau\ge1$, the profile is the fraction of the 12 settings in which the time of the method to the level, or its mean error at the last iteration, is at most $\tau$ times the best value among the methods shown. The value at $\tau=1$ is the fraction of settings in which the method is the best, and the smallest $\tau$ at which the profile reaches one bounds how far the method is from the best in its worst setting. In regression the time profile of RGN is never below those of the other methods and reaches one at $\tau=1.4$, and in decomposition it is never below them for $\tau\le2.3$ and reaches one at $\tau=3.0$. The error profiles of RGD, RGN and CP-ALS reach one by $\tau=1.07$ in both problems, whereas those of RRR and the methods in other tensor formats in regression, and of Orth-ALS and ICO in decomposition, stay below one for $\tau\le2$.

\paragraph{Errors at the last iteration.}
Each curve in Figures~\ref{fig:tensor_regression_convergence_main}--\ref{fig:decomposition_second_order} is the root mean square, over 20 replications, of the relative Frobenius error $\|\widehat\calT-\calT^\star\|_{\rm F}/\|\calT^\star\|_{\rm F}$. Tables~\ref{tab:regression_errors} and~\ref{tab:decomposition_errors} report the mean and standard deviation of this error at the last iteration, 20 in regression and 100 in decomposition, in the settings with noise. In regression, RGN attains the lowest mean error in every setting, in most settings together with RGD and CP-ALS, whereas the methods in other tensor formats have errors 1.5 to 2.7 times larger, and HOPLS and, in 10 of the 12 settings, RRR larger still. In decomposition, RGD, RGN, CP-ALS, NLS-GN, GN-CG, CP-ALS+ELS and Hybrid AMDM attain the same error up to one standard deviation in every setting, whereas Orth-ALS and ICO have larger errors when $\rho>0$ and $s\ge10$.

\paragraph{Settings without noise.}
Without noise the errors fall to machine precision, so we describe these settings instead of tabulating them. In regression, RGN reaches machine precision in fewer iterations than CP-ALS when $\rho\le0.5$, and at $\rho=0.75$ only Tucker-RGN reaches it within 20 iterations. In decomposition with $\rho>0$, ICO and Hybrid AMDM reach machine precision within two iterations, NLS-GN and GN-CG level off between $10^{-12}$ and $10^{-7}$, and at $\rho=0.75$ CP-ALS, CP-ALS+ELS and RGN reach $10^{-8}$ or less within 100 iterations.

\paragraph{Sensitivity to the initialization.}
Tables~\ref{tab:initialization_regression} and~\ref{tab:initialization_decomposition} report, in the settings with noise, the errors at iterations 0, 3, 5, 10 and 20 from three starts: CPCA, T-HOSVD, which takes the $i$-th component from the $i$-th leading singular vectors of the mode unfoldings, and random unit factors with weights fitted to the data. When $\rho\le0.5$, CPCA starts closer to $\calT^\star$ than T-HOSVD, and at $\rho=0.75$ the errors of the two starts are within 4\% of each other. When $\rho\le0.5$ and $s\ge10$, RGN started from CPCA is within 1\% of its error at iteration 20 after five iterations in both problems. In the same settings, RGN started from T-HOSVD attains this error within ten iterations in decomposition, whereas in regression its error at iteration 10 is 1.2 to 3.5 times that from CPCA. At $\rho=0.75$, after three iterations of RGN, the error from CPCA is 5 to 19 times smaller than that from T-HOSVD in regression and 1.7 to 3.5 times smaller in decomposition. At iteration 20, some replications in regression started from T-HOSVD miss the weakest component, and when $\rho\le0.5$ this happens less often for RGD and RGN than for CP-ALS. In decomposition, the errors at iteration 20 from T-HOSVD are within about one standard deviation (Table~\ref{tab:decomposition_errors}) of those from CPCA, except at $\rho=0.75$, where RGD and RGN converge more slowly. From a random start, all three methods often miss components within 20 iterations, RGD and RGN more often than CP-ALS in most settings. This agrees with the local nature of the guarantees in Section~\ref{sec:theory}.

\begin{table}[!htbp]
\centering
\caption{Median running time per iteration in milliseconds over all settings.}
\label{tab:timing}
{\small
\begin{tabular}{lr@{\qquad}lr}
\toprule
\multicolumn{2}{c}{Tensor-on-tensor regression} & \multicolumn{2}{c}{CP decomposition}\\
\cmidrule(r){1-2}\cmidrule(l){3-4}
RGD & 15.4 & RGD & 14.9 \\
RGN & 16.8 & RGN & 14.7 \\
CP-ALS & 21.6 & CP-ALS & 22.4 \\
RRR & 338.3 & Hybrid AMDM & 24.5 \\
TT-IHT & 18.6 & Orth-ALS & 27.2 \\
Tucker-RGD & 76.3 & ICO & 28.6 \\
Tucker-RGN & 262.3 & NLS-GN & 32.6 \\
 &  & GN-CG & 41.3 \\
 &  & CP-ALS+ELS & 531.8 \\
\bottomrule
\end{tabular}
}
\end{table}

\begin{figure}[!htbp]
  \centering
  \includegraphics[width=\textwidth]{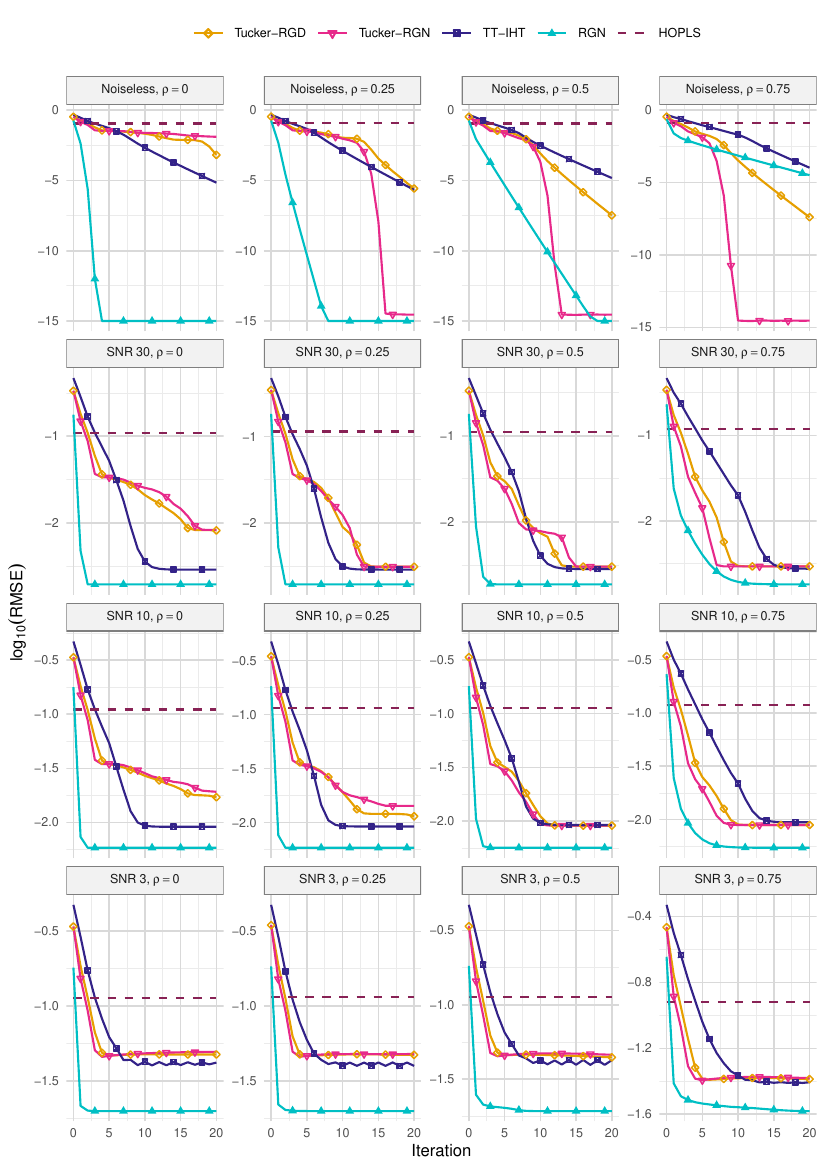}
  \caption{Convergence for tensor-on-tensor regression of methods in other tensor formats, with RGN as reference. HOPLS (dashed) is not iterative. Errors below $10^{-15}$ are drawn at $10^{-15}$. Without noise at $\rho=0$, Tucker-RGN misses the weakest component in three of the 20 replications, which keeps its root mean square error near $10^{-2}$.}
  \label{fig:regression_other_formats}
\end{figure}

\begin{figure}[!htbp]
  \centering
  \includegraphics[width=\textwidth]{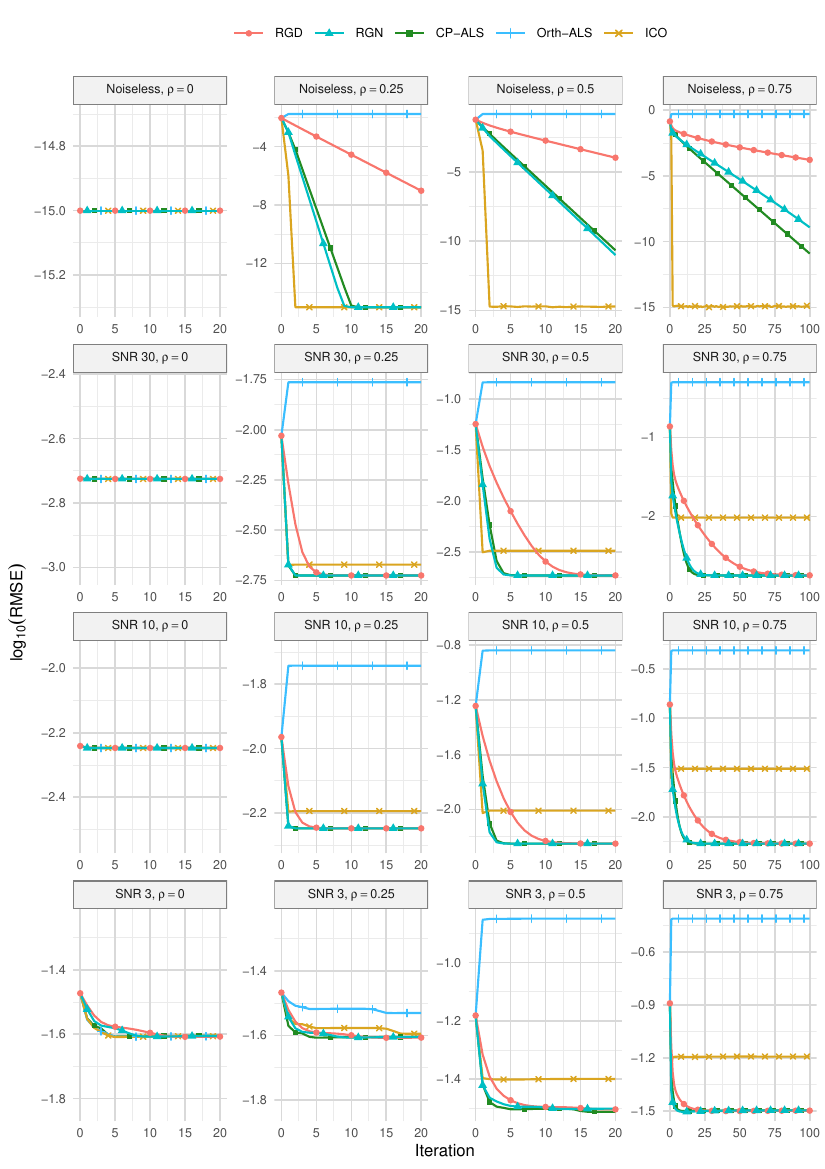}
  \caption{Convergence for CP decomposition of fourth-order tensors. At $\rho=0.75$ the horizontal axis extends to 100 iterations. Errors below $10^{-15}$ are drawn at $10^{-15}$. At $\rho=0$, CPCA recovers the noiseless tensor exactly, so without noise all curves stay at $10^{-15}$, and with noise they start close to the error level. Orth-ALS orthogonalizes the factor estimates before every sweep, which biases them when the factors are coherent, so for $\rho>0$ its error levels off above that of the other methods. ICO contracts the data with the pseudo-inverses of the factor matrices, which recovers the noiseless tensor within two iterations but amplifies the noise when $\rho>0$.}
  \label{fig:decomposition_convergence}
\end{figure}

\begin{figure}[!htbp]
  \centering
  \includegraphics[width=\textwidth]{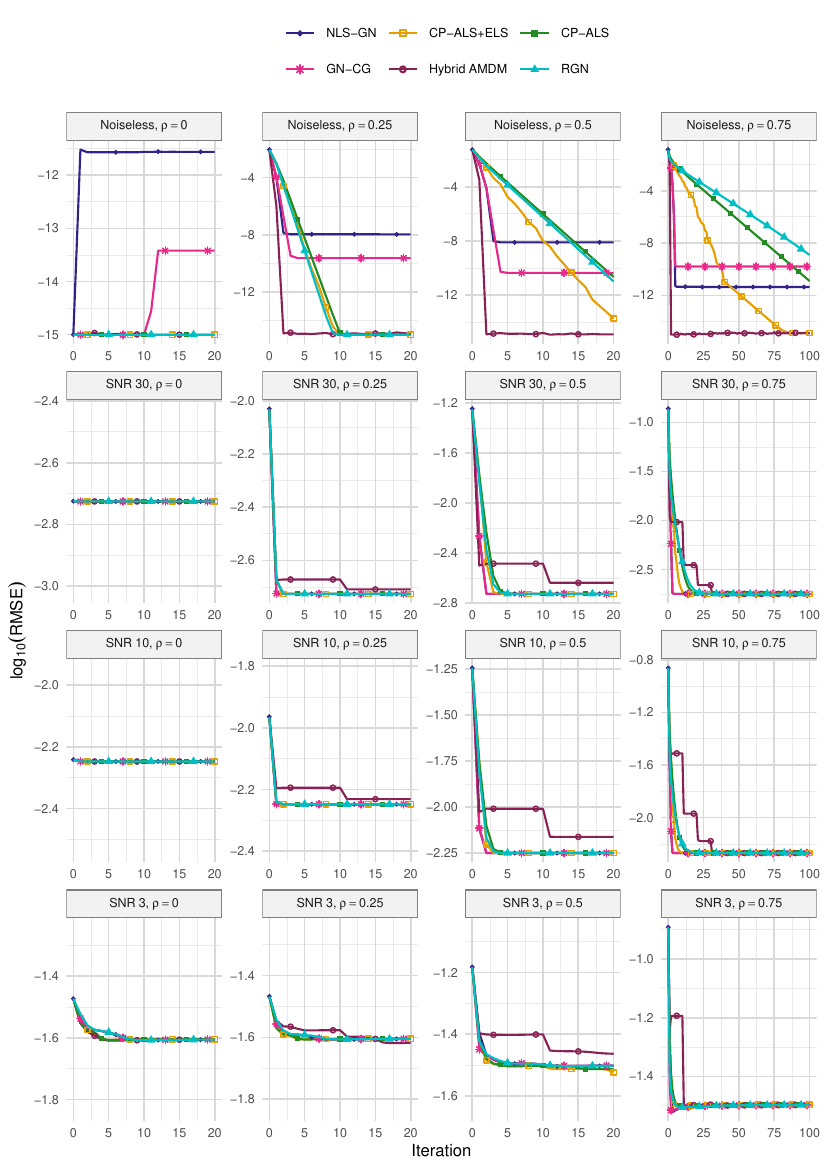}
  \caption{Convergence for CP decomposition: Gauss-Newton and accelerated alternating methods, with CP-ALS and RGN as reference. At $\rho=0.75$ the horizontal axis extends to 100 iterations. Errors below $10^{-15}$ are drawn at $10^{-15}$, and at $\rho=0$ the curves start at or close to the error level for the reason given in Figure~\ref{fig:decomposition_convergence}. Without noise at $\rho=0$, the errors of NLS-GN and GN-CG rise from $10^{-15}$ to their numerical floors, about $10^{-11.6}$ and $10^{-13.4}$. Hybrid AMDM moves from AMDM toward ALS updates every 10 iterations \citep{singh2023alternating}, which produces the steps in its curves.}
  \label{fig:decomposition_second_order}
\end{figure}

\begin{table}[!htbp]
\centering
\caption{Tensor-on-tensor regression: time in milliseconds to come within 1\% of the error level, the lowest error at iteration 100 among CP-ALS, RGD and RGN in the replication, in the settings with noise. The time is the median over the 20 replications of the number of iterations needed, times the median time per iteration of the method in the setting. A replication in which the method does not reach the level counts as infinitely many iterations. Bold: fastest in each setting. RRR, Tucker-RGD, Tucker-RGN, TT-IHT and HOPLS are not shown, as they reach the level in fewer than half of the replications in every setting.}
\label{tab:time_regression}
{\small
\begin{tabular}{llcccc}
\toprule
$s$ & Method & $\rho=0$ & $\rho=0.25$ & $\rho=0.5$ & $\rho=0.75$ \\
\midrule
30 & RGD & 86 & 84 & 74 & \textbf{196} \\
 & RGN & \textbf{32} & \textbf{34} & \textbf{44} & 234 \\
 & CP-ALS & 69 & 62 & 64 & 264 \\
\midrule
10 & RGD & 78 & 63 & 79 & \textbf{135} \\
 & RGN & \textbf{38} & \textbf{34} & \textbf{49} & 180 \\
 & CP-ALS & 48 & 71 & 63 & 213 \\
\midrule
3 & RGD & 62 & 57 & 67 & \textbf{148} \\
 & RGN & \textbf{35} & 49 & \textbf{56} & 174 \\
 & CP-ALS & 45 & \textbf{42} & 63 & 150 \\
\bottomrule
\end{tabular}}
\end{table}

\begin{table}[!htbp]
\centering
\caption{CP decomposition: time in milliseconds to come within 1\% of the error level, the lowest error at iteration 100 among CP-ALS, RGD and RGN in the replication, in the settings with noise. The time is the median over the 20 replications of the number of iterations needed, times the median time per iteration of the method in the setting. A replication in which the method does not reach the level counts as infinitely many iterations. Bold: fastest in each setting. At $s=30$ and $\rho=0$, every method starts at the level. Orth-ALS and ICO are not shown, as they reach the level in fewer than half of the replications when $\rho>0$.}
\label{tab:time_decomposition}
{\small
\begin{tabular}{llcccc}
\toprule
$s$ & Method & $\rho=0$ & $\rho=0.25$ & $\rho=0.5$ & $\rho=0.75$ \\
\midrule
30 & RGD & \textbf{0} & 103 & 241 & 1290 \\
 & RGN & \textbf{0} & \textbf{28} & 68 & 391 \\
 & CP-ALS & \textbf{0} & 45 & 101 & 543 \\
 & NLS-GN & \textbf{0} & 30 & \textbf{55} & \textbf{130} \\
 & GN-CG & \textbf{0} & 42 & 86 & 174 \\
 & CP-ALS+ELS & \textbf{0} & 1067 & 2022 & 6253 \\
 & Hybrid AMDM & \textbf{0} & 746 & 741 & 808 \\
\midrule
10 & RGD & \textbf{14} & 66 & 187 & 895 \\
 & RGN & \textbf{14} & \textbf{28} & \textbf{56} & 293 \\
 & CP-ALS & 22 & 41 & 86 & 377 \\
 & NLS-GN & 32 & 29 & 67 & \textbf{108} \\
 & GN-CG & 41 & 38 & 83 & 139 \\
 & CP-ALS+ELS & 538 & 1017 & 1594 & 5274 \\
 & Hybrid AMDM & 26 & 746 & 789 & 744 \\
\midrule
3 & RGD & 107 & 89 & 140 & 359 \\
 & RGN & 53 & \textbf{41} & 76 & 68 \\
 & CP-ALS & \textbf{46} & 42 & \textbf{74} & 167 \\
 & NLS-GN & 108 & 98 & 132 & 103 \\
 & GN-CG & 103 & 99 & 129 & \textbf{66} \\
 & CP-ALS+ELS & 1124 & 1066 & 1580 & 2703 \\
 & Hybrid AMDM & 54 & 693 & 749 & 295 \\
\bottomrule
\end{tabular}}
\end{table}

\begin{figure}[!htbp]
  \centering
  \includegraphics[width=\textwidth]{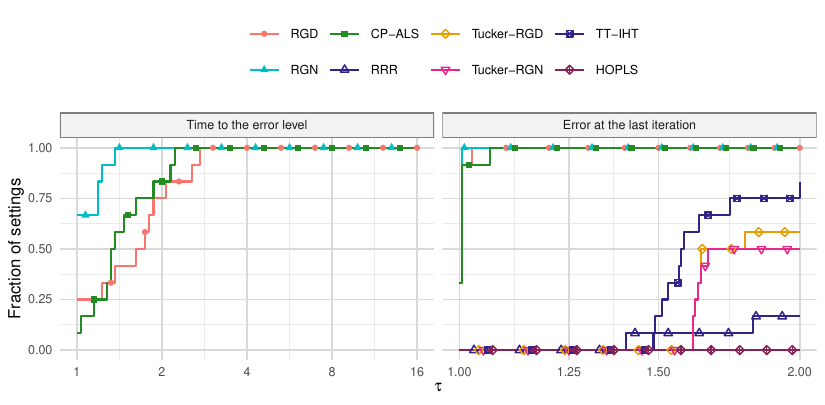}\\[4pt]
  \includegraphics[width=\textwidth]{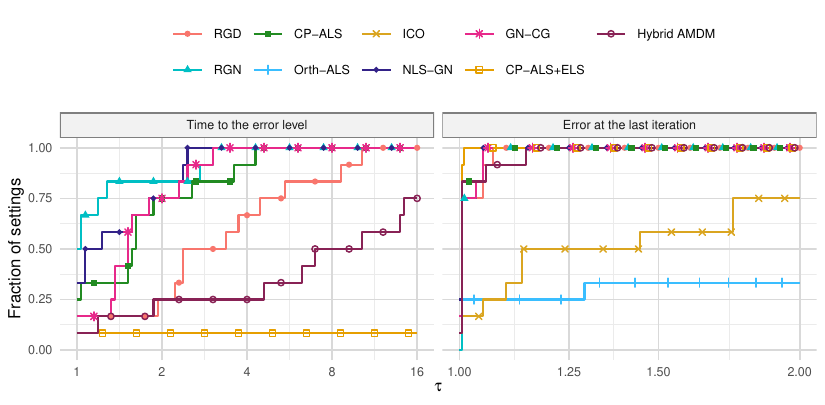}
  \caption{Performance profiles over the 12 settings with noise, for tensor-on-tensor regression (top) and CP decomposition (bottom). Left: the fraction of settings in which the method comes within $1\%$ of the error level in at most $\tau$ times the time of the fastest method (Tables~\ref{tab:time_regression} and~\ref{tab:time_decomposition}); at $s=30$, $\rho=0$ in decomposition every method starts at the level. Right: the fraction of settings in which the mean error of the method at the last iteration is at most $\tau$ times the lowest among the methods shown (Tables~\ref{tab:regression_errors} and~\ref{tab:decomposition_errors}). The horizontal axis is logarithmic.}
  \label{fig:profiles}
\end{figure}

\begin{table}[!htbp]
\centering
\caption{Tensor-on-tensor regression: mean (standard deviation) over 20 replications of the relative Frobenius error at iteration 20, in units of $10^{-3}$, in the settings with noise. Bold: lowest mean in each setting. HOPLS is not iterative.}
\label{tab:regression_errors}
{\small
\begin{tabular}{llcccc}
\toprule
$s$ & Method & $\rho=0$ & $\rho=0.25$ & $\rho=0.5$ & $\rho=0.75$ \\
\midrule
30 & RGD & \textbf{1.93} (0.07) & \textbf{1.94} (0.16) & \textbf{1.88} (0.12) & \textbf{1.81} (0.08) \\
 & RGN & \textbf{1.93} (0.07) & \textbf{1.94} (0.16) & \textbf{1.88} (0.12) & \textbf{1.81} (0.08) \\
 & CP-ALS & \textbf{1.93} (0.07) & \textbf{1.94} (0.16) & \textbf{1.88} (0.12) & \textbf{1.81} (0.08) \\
 & RRR & 22.0 (15.1) & 32.6 (37.7) & 23.2 (14.3) & 20.1 (10.9) \\
 & Tucker-RGD & 4.58 (6.91) & 3.11 (0.15) & 3.02 (0.11) & 2.94 (0.10) \\
 & Tucker-RGN & 4.63 (7.03) & 3.11 (0.15) & 3.02 (0.11) & 2.94 (0.10) \\
 & TT-IHT & 2.87 (0.14) & 2.88 (0.14) & 2.83 (0.12) & 2.77 (0.09) \\
 & HOPLS & 109 (16) & 114 (13) & 111 (12) & 119 (12) \\
\midrule
10 & RGD & \textbf{5.81} (0.22) & \textbf{5.82} (0.49) & \textbf{5.64} (0.36) & \textbf{5.45} (0.25) \\
 & RGN & \textbf{5.81} (0.22) & \textbf{5.82} (0.49) & \textbf{5.64} (0.36) & \textbf{5.45} (0.25) \\
 & CP-ALS & \textbf{5.81} (0.22) & \textbf{5.82} (0.49) & \textbf{5.64} (0.36) & \textbf{5.45} (0.25) \\
 & RRR & 27.9 (11.4) & 29.3 (10.0) & 27.3 (11.2) & 19.1 (10.5) \\
 & Tucker-RGD & 14.0 (10.1) & 10.4 (4.8) & 9.08 (0.35) & 8.90 (0.32) \\
 & Tucker-RGN & 15.6 (11.5) & 11.9 (8.0) & 9.08 (0.35) & 8.90 (0.32) \\
 & TT-IHT & 9.09 (0.42) & 9.18 (0.50) & 9.18 (0.47) & 9.43 (0.66) \\
 & HOPLS & 110 (16) & 114 (13) & 111 (12) & 119 (12) \\
\midrule
3 & RGD & \textbf{19.9} (0.7) & \textbf{19.8} (1.7) & \textbf{19.2} (1.3) & 25.6 (8.0) \\
 & RGN & \textbf{19.9} (0.7) & \textbf{19.8} (1.7) & \textbf{19.2} (1.3) & \textbf{25.0} (8.1) \\
 & CP-ALS & \textbf{19.9} (0.7) & \textbf{19.8} (1.7) & \textbf{19.2} (1.3) & 26.6 (8.2) \\
 & RRR & 52.7 (45.4) & 36.0 (7.8) & 46.0 (30.6) & 35.2 (15.6) \\
 & Tucker-RGD & 47.2 (5.4) & 47.0 (5.8) & 43.8 (6.6) & 40.9 (3.7) \\
 & Tucker-RGN & 49.5 (0.6) & 47.7 (4.5) & 45.8 (5.7) & 41.5 (3.9) \\
 & TT-IHT & 41.6 (5.2) & 39.6 (4.8) & 42.2 (4.5) & 39.3 (1.9) \\
 & HOPLS & 112 (16) & 115 (13) & 113 (12) & 120 (11) \\
\bottomrule
\end{tabular}}

\bigskip
\caption{CP decomposition: mean (standard deviation) over 20 replications of the relative Frobenius error at iteration 100, in units of $10^{-3}$, in the settings with noise. Bold: lowest mean in each setting.}
\label{tab:decomposition_errors}
{\small
\begin{tabular}{llcccc}
\toprule
$s$ & Method & $\rho=0$ & $\rho=0.25$ & $\rho=0.5$ & $\rho=0.75$ \\
\midrule
30 & RGD & \textbf{1.88} (0.06) & \textbf{1.88} (0.05) & \textbf{1.87} (0.07) & \textbf{1.79} (0.05) \\
 & RGN & \textbf{1.88} (0.06) & \textbf{1.88} (0.05) & \textbf{1.87} (0.07) & \textbf{1.79} (0.06) \\
 & CP-ALS & \textbf{1.88} (0.06) & \textbf{1.88} (0.05) & \textbf{1.87} (0.07) & \textbf{1.79} (0.06) \\
 & Orth-ALS & \textbf{1.88} (0.06) & 17.2 (0.1) & 145 (0) & 500 (2) \\
 & ICO & \textbf{1.88} (0.06) & 2.13 (0.07) & 3.27 (0.12) & 9.56 (1.33) \\
 & NLS-GN & \textbf{1.88} (0.06) & \textbf{1.88} (0.05) & \textbf{1.87} (0.07) & \textbf{1.79} (0.06) \\
 & GN-CG & \textbf{1.88} (0.06) & \textbf{1.88} (0.05) & \textbf{1.87} (0.07) & \textbf{1.79} (0.06) \\
 & CP-ALS+ELS & \textbf{1.88} (0.06) & \textbf{1.88} (0.05) & \textbf{1.87} (0.07) & \textbf{1.79} (0.06) \\
 & Hybrid AMDM & \textbf{1.88} (0.06) & \textbf{1.88} (0.05) & \textbf{1.87} (0.07) & \textbf{1.79} (0.06) \\
\midrule
10 & RGD & \textbf{5.65} (0.17) & \textbf{5.65} (0.16) & \textbf{5.63} (0.20) & 5.38 (0.17) \\
 & RGN & \textbf{5.65} (0.17) & \textbf{5.65} (0.16) & \textbf{5.63} (0.20) & \textbf{5.37} (0.17) \\
 & CP-ALS & \textbf{5.65} (0.17) & \textbf{5.65} (0.16) & \textbf{5.63} (0.20) & \textbf{5.37} (0.17) \\
 & Orth-ALS & \textbf{5.65} (0.17) & 18.1 (0.4) & 145 (1) & 485 (7) \\
 & ICO & \textbf{5.65} (0.17) & 6.39 (0.20) & 9.80 (0.36) & 30.6 (3.5) \\
 & NLS-GN & \textbf{5.65} (0.17) & \textbf{5.65} (0.16) & \textbf{5.63} (0.20) & \textbf{5.37} (0.17) \\
 & GN-CG & \textbf{5.65} (0.17) & \textbf{5.65} (0.16) & \textbf{5.63} (0.20) & \textbf{5.37} (0.17) \\
 & CP-ALS+ELS & \textbf{5.65} (0.17) & \textbf{5.65} (0.16) & \textbf{5.63} (0.20) & \textbf{5.37} (0.17) \\
 & Hybrid AMDM & \textbf{5.65} (0.17) & \textbf{5.65} (0.16) & \textbf{5.63} (0.20) & \textbf{5.37} (0.17) \\
\midrule
3 & RGD & 23.5 (8.5) & 23.4 (8.6) & 29.0 (10.4) & \textbf{31.7} (0.3) \\
 & RGN & 23.5 (8.6) & 23.5 (8.7) & 28.1 (10.4) & 31.8 (0.3) \\
 & CP-ALS & 23.5 (8.5) & 23.5 (8.7) & \textbf{27.2} (10.2) & 31.9 (0.3) \\
 & Orth-ALS & \textbf{22.4} (7.6) & 28.9 (6.2) & 142 (2) & 388 (6) \\
 & ICO & 23.4 (8.6) & 24.6 (7.0) & 39.2 (5.2) & 63.9 (5.6) \\
 & NLS-GN & 23.5 (8.5) & 23.5 (8.7) & 28.1 (10.4) & 31.8 (0.3) \\
 & GN-CG & 23.5 (8.5) & 23.5 (8.7) & 28.1 (10.4) & 31.8 (0.3) \\
 & CP-ALS+ELS & \textbf{22.4} (7.7) & \textbf{22.4} (7.7) & \textbf{27.2} (10.2) & 31.9 (0.3) \\
 & Hybrid AMDM & 23.5 (8.6) & \textbf{22.4} (7.7) & 31.1 (10.2) & 31.8 (0.3) \\
\bottomrule
\end{tabular}}
\end{table}

\begin{table}[!htbp]
\centering
\caption{Sensitivity to the initialization in tensor-on-tensor regression: mean over 20 replications of the relative Frobenius error, in units of $10^{-3}$, at iterations $t=0,3,5,10,20$ from the CPCA, T-HOSVD and random starts, in the settings with noise. At $t=0$ the entry is the error of the start, the same for all three methods.}
\label{tab:initialization_regression}
{\small\setlength{\tabcolsep}{2.5pt}
\begin{tabular}{lll*{15}{r}}
\toprule
 & & & \multicolumn{5}{c}{CPCA} & \multicolumn{5}{c}{T-HOSVD} & \multicolumn{5}{c}{Random} \\
\cmidrule(lr){4-8}\cmidrule(lr){9-13}\cmidrule(l){14-18}
$s$ & $\rho$ & Method & 0 & 3 & 5 & 10 & 20 & 0 & 3 & 5 & 10 & 20 & 0 & 3 & 5 & 10 & 20 \\
\midrule
30 & 0 & RGD & 176 & 3.78 & 1.97 & 1.93 & 1.93 & 248 & 33.2 & 13.1 & 5.75 & 1.93 & 1000 & 886 & 332 & 174 & 144 \\
 &  & RGN & 176 & 1.93 & 1.93 & 1.93 & 1.93 & 248 & 28.6 & 11.4 & 6.75 & 1.93 & 1000 & 887 & 311 & 166 & 136 \\
 &  & CP-ALS & 176 & 1.94 & 1.93 & 1.93 & 1.93 & 248 & 17.7 & 14.1 & 9.68 & 8.14 & 1000 & 168 & 152 & 93.2 & 48.4 \\
\addlinespace
 & 0.25 & RGD & 181 & 3.74 & 1.97 & 1.94 & 1.94 & 254 & 38.5 & 11.9 & 3.54 & 3.54 & 1000 & 971 & 622 & 114 & 73.6 \\
 &  & RGN & 181 & 1.94 & 1.94 & 1.94 & 1.94 & 254 & 33.8 & 9.67 & 5.14 & 3.54 & 1000 & 968 & 597 & 122 & 80.7 \\
 &  & CP-ALS & 181 & 1.94 & 1.94 & 1.94 & 1.94 & 254 & 17.1 & 11.3 & 9.69 & 8.14 & 1000 & 161 & 55.6 & 35.4 & 23.6 \\
\addlinespace
 & 0.5 & RGD & 178 & 3.48 & 1.90 & 1.88 & 1.88 & 250 & 97.0 & 31.2 & 2.35 & 1.88 & 1000 & 920 & 450 & 106 & 55.2 \\
 &  & RGN & 178 & 1.89 & 1.88 & 1.88 & 1.88 & 250 & 94.3 & 30.7 & 2.61 & 1.88 & 1000 & 915 & 433 & 113 & 37.2 \\
 &  & CP-ALS & 178 & 1.90 & 1.88 & 1.88 & 1.88 & 250 & 19.8 & 15.3 & 10.8 & 10.3 & 1000 & 126 & 56.0 & 30.9 & 26.1 \\
\addlinespace
 & 0.75 & RGD & 230 & 6.56 & 3.41 & 1.92 & 1.81 & 231 & 147 & 99.5 & 18.8 & 8.42 & 1000 & 931 & 369 & 67.0 & 35.2 \\
 &  & RGN & 230 & 7.69 & 4.03 & 1.94 & 1.81 & 231 & 147 & 96.4 & 14.5 & 7.60 & 1000 & 931 & 359 & 60.3 & 35.5 \\
 &  & CP-ALS & 230 & 10.4 & 4.66 & 1.92 & 1.81 & 231 & 38.7 & 20.7 & 13.4 & 7.46 & 1000 & 77.7 & 46.7 & 30.4 & 20.1 \\
\midrule
10 & 0 & RGD & 177 & 6.72 & 5.83 & 5.81 & 5.81 & 248 & 34.3 & 16.6 & 8.58 & 5.81 & 1000 & 886 & 333 & 176 & 157 \\
 &  & RGN & 177 & 5.81 & 5.81 & 5.81 & 5.81 & 248 & 30.0 & 17.0 & 9.41 & 7.09 & 1000 & 887 & 317 & 171 & 130 \\
 &  & CP-ALS & 177 & 5.81 & 5.81 & 5.81 & 5.81 & 248 & 26.1 & 18.4 & 15.5 & 12.8 & 1000 & 166 & 125 & 116 & 54.9 \\
\addlinespace
 & 0.25 & RGD & 182 & 6.67 & 5.83 & 5.82 & 5.82 & 254 & 39.4 & 14.9 & 8.72 & 5.97 & 1000 & 971 & 631 & 110 & 70.1 \\
 &  & RGN & 182 & 5.82 & 5.82 & 5.82 & 5.82 & 254 & 35.2 & 13.9 & 8.49 & 5.82 & 1000 & 969 & 612 & 101 & 60.7 \\
 &  & CP-ALS & 182 & 5.82 & 5.82 & 5.82 & 5.82 & 254 & 23.9 & 17.0 & 12.8 & 11.4 & 1000 & 129 & 64.5 & 38.3 & 29.3 \\
\addlinespace
 & 0.5 & RGD & 179 & 6.41 & 5.65 & 5.64 & 5.64 & 250 & 98.0 & 36.6 & 8.92 & 5.64 & 1000 & 921 & 463 & 120 & 37.5 \\
 &  & RGN & 179 & 5.65 & 5.64 & 5.64 & 5.64 & 250 & 95.0 & 30.4 & 7.04 & 5.64 & 1000 & 916 & 441 & 113 & 36.9 \\
 &  & CP-ALS & 179 & 5.65 & 5.64 & 5.64 & 5.64 & 250 & 24.8 & 16.0 & 12.1 & 9.75 & 1000 & 132 & 60.7 & 34.0 & 28.8 \\
\addlinespace
 & 0.75 & RGD & 230 & 8.36 & 6.18 & 5.48 & 5.45 & 231 & 148 & 106 & 20.1 & 9.75 & 1000 & 934 & 372 & 64.8 & 33.9 \\
 &  & RGN & 230 & 9.25 & 6.52 & 5.49 & 5.45 & 231 & 147 & 103 & 17.2 & 10.6 & 1000 & 934 & 359 & 58.3 & 30.4 \\
 &  & CP-ALS & 230 & 11.7 & 6.99 & 5.49 & 5.45 & 231 & 36.8 & 20.7 & 13.9 & 11.6 & 1000 & 77.1 & 48.9 & 32.7 & 22.9 \\
\midrule
3 & 0 & RGD & 179 & 20.2 & 19.9 & 19.9 & 19.9 & 249 & 43.4 & 34.1 & 27.3 & 27.1 & 1000 & 895 & 376 & 161 & 121 \\
 &  & RGN & 179 & 19.9 & 19.9 & 19.9 & 19.9 & 249 & 40.0 & 34.0 & 28.2 & 28.1 & 1000 & 898 & 365 & 162 & 111 \\
 &  & CP-ALS & 179 & 19.9 & 19.9 & 19.9 & 19.9 & 249 & 39.1 & 38.8 & 37.1 & 35.1 & 1000 & 155 & 130 & 90.1 & 75.7 \\
\addlinespace
 & 0.25 & RGD & 183 & 20.2 & 19.9 & 19.8 & 19.8 & 254 & 49.5 & 35.1 & 27.5 & 25.0 & 1000 & 974 & 691 & 127 & 79.8 \\
 &  & RGN & 183 & 20.0 & 19.9 & 19.8 & 19.8 & 254 & 47.1 & 35.6 & 27.7 & 25.2 & 1000 & 973 & 680 & 122 & 75.3 \\
 &  & CP-ALS & 183 & 19.9 & 19.8 & 19.8 & 19.8 & 254 & 37.9 & 37.4 & 35.3 & 35.5 & 1000 & 138 & 88.3 & 46.7 & 47.2 \\
\addlinespace
 & 0.5 & RGD & 181 & 20.6 & 20.1 & 19.3 & 19.2 & 250 & 106 & 53.5 & 34.0 & 30.6 & 1000 & 932 & 549 & 128 & 39.6 \\
 &  & RGN & 181 & 20.4 & 20.0 & 19.3 & 19.2 & 250 & 104 & 50.9 & 33.4 & 30.9 & 1000 & 927 & 532 & 119 & 45.0 \\
 &  & CP-ALS & 181 & 19.3 & 19.2 & 19.2 & 19.2 & 250 & 38.8 & 35.9 & 35.6 & 33.9 & 1000 & 134 & 81.9 & 57.0 & 46.9 \\
\addlinespace
 & 0.75 & RGD & 225 & 30.1 & 28.5 & 26.7 & 25.6 & 231 & 150 & 123 & 51.0 & 31.0 & 1000 & 950 & 440 & 76.1 & 51.0 \\
 &  & RGN & 225 & 30.3 & 28.6 & 26.7 & 25.0 & 231 & 150 & 121 & 54.8 & 37.7 & 1000 & 949 & 420 & 75.7 & 37.5 \\
 &  & CP-ALS & 225 & 30.3 & 27.5 & 26.5 & 26.6 & 231 & 70.4 & 42.1 & 33.7 & 30.9 & 1000 & 79.6 & 51.6 & 40.3 & 33.2 \\
\bottomrule
\end{tabular}}
\end{table}

\clearpage
\begin{table}[H]
\centering
\caption{Sensitivity to the initialization in CP decomposition, in the format of Table~\ref{tab:initialization_regression}.}
\label{tab:initialization_decomposition}
{\small\setlength{\tabcolsep}{2.5pt}
\begin{tabular}{lll*{15}{r}}
\toprule
 & & & \multicolumn{5}{c}{CPCA} & \multicolumn{5}{c}{T-HOSVD} & \multicolumn{5}{c}{Random} \\
\cmidrule(lr){4-8}\cmidrule(lr){9-13}\cmidrule(l){14-18}
$s$ & $\rho$ & Method & 0 & 3 & 5 & 10 & 20 & 0 & 3 & 5 & 10 & 20 & 0 & 3 & 5 & 10 & 20 \\
\midrule
30 & 0 & RGD & 1.88 & 1.88 & 1.88 & 1.88 & 1.88 & 1.96 & 1.88 & 1.88 & 1.88 & 1.88 & 1000 & 494 & 182 & 176 & 175 \\
 &  & RGN & 1.88 & 1.88 & 1.88 & 1.88 & 1.88 & 1.96 & 1.88 & 1.88 & 1.88 & 1.88 & 1000 & 292 & 174 & 139 & 138 \\
 &  & CP-ALS & 1.88 & 1.88 & 1.88 & 1.88 & 1.88 & 1.96 & 1.88 & 1.88 & 1.88 & 1.88 & 1000 & 95.1 & 85.2 & 85.2 & 70.3 \\
\addlinespace
 & 0.25 & RGD & 9.33 & 2.46 & 1.95 & 1.88 & 1.88 & 91.7 & 14.2 & 4.64 & 1.90 & 1.88 & 1000 & 465 & 176 & 153 & 144 \\
 &  & RGN & 9.33 & 1.88 & 1.88 & 1.88 & 1.88 & 91.7 & 1.88 & 1.88 & 1.88 & 1.88 & 1000 & 215 & 164 & 48.1 & 36.8 \\
 &  & CP-ALS & 9.33 & 1.88 & 1.88 & 1.88 & 1.88 & 91.7 & 1.88 & 1.88 & 1.88 & 1.88 & 1000 & 53.9 & 40.0 & 35.2 & 29.7 \\
\addlinespace
 & 0.5 & RGD & 56.7 & 15.1 & 7.98 & 2.56 & 1.88 & 152 & 41.4 & 20.4 & 4.67 & 1.89 & 1000 & 508 & 167 & 109 & 88.9 \\
 &  & RGN & 56.7 & 2.23 & 1.88 & 1.87 & 1.87 & 152 & 9.63 & 2.07 & 1.87 & 1.87 & 1000 & 296 & 147 & 57.2 & 42.2 \\
 &  & CP-ALS & 56.7 & 2.60 & 1.89 & 1.87 & 1.87 & 152 & 3.56 & 1.91 & 1.87 & 1.87 & 1000 & 72.6 & 47.1 & 35.0 & 27.1 \\
\addlinespace
 & 0.75 & RGD & 138 & 33.8 & 25.3 & 15.7 & 7.72 & 133 & 65.7 & 45.8 & 24.3 & 10.7 & 1000 & 484 & 134 & 68.4 & 37.1 \\
 &  & RGN & 138 & 14.2 & 8.98 & 3.80 & 1.90 & 133 & 49.6 & 20.4 & 6.34 & 2.06 & 1000 & 222 & 98.8 & 30.6 & 22.4 \\
 &  & CP-ALS & 138 & 18.7 & 10.2 & 3.38 & 1.81 & 133 & 22.1 & 11.5 & 3.65 & 1.82 & 1000 & 51.8 & 40.2 & 32.7 & 26.8 \\
\midrule
10 & 0 & RGD & 5.73 & 5.66 & 5.65 & 5.65 & 5.65 & 7.80 & 5.69 & 5.66 & 5.65 & 5.65 & 1000 & 519 & 190 & 182 & 175 \\
 &  & RGN & 5.73 & 5.65 & 5.65 & 5.65 & 5.65 & 7.80 & 5.65 & 5.65 & 5.65 & 5.65 & 1000 & 311 & 182 & 112 & 100 \\
 &  & CP-ALS & 5.73 & 5.65 & 5.65 & 5.65 & 5.65 & 7.80 & 5.65 & 5.65 & 5.65 & 5.65 & 1000 & 109 & 85.7 & 85.6 & 78.0 \\
\addlinespace
 & 0.25 & RGD & 10.9 & 5.89 & 5.68 & 5.65 & 5.65 & 92.0 & 15.2 & 7.08 & 5.66 & 5.65 & 1000 & 480 & 181 & 132 & 108 \\
 &  & RGN & 10.9 & 5.65 & 5.65 & 5.65 & 5.65 & 92.0 & 5.65 & 5.65 & 5.65 & 5.65 & 1000 & 223 & 172 & 51.7 & 31.9 \\
 &  & CP-ALS & 10.9 & 5.65 & 5.65 & 5.65 & 5.65 & 92.0 & 5.65 & 5.65 & 5.65 & 5.65 & 1000 & 81.7 & 34.5 & 33.3 & 30.6 \\
\addlinespace
 & 0.5 & RGD & 57.0 & 16.1 & 9.59 & 5.89 & 5.63 & 152 & 41.7 & 21.1 & 7.06 & 5.63 & 1000 & 528 & 177 & 104 & 76.1 \\
 &  & RGN & 57.0 & 5.76 & 5.63 & 5.63 & 5.63 & 152 & 11.0 & 5.70 & 5.63 & 5.63 & 1000 & 310 & 157 & 54.4 & 34.2 \\
 &  & CP-ALS & 57.0 & 5.92 & 5.64 & 5.63 & 5.63 & 152 & 6.41 & 5.64 & 5.63 & 5.63 & 1000 & 70.9 & 44.5 & 36.5 & 30.2 \\
\addlinespace
 & 0.75 & RGD & 138 & 34.4 & 25.9 & 16.5 & 9.23 & 133 & 66.0 & 46.2 & 24.9 & 11.8 & 1000 & 509 & 145 & 65.2 & 37.4 \\
 &  & RGN & 138 & 15.1 & 10.3 & 6.33 & 5.41 & 133 & 50.2 & 21.1 & 8.11 & 5.47 & 1000 & 248 & 110 & 35.5 & 22.0 \\
 &  & CP-ALS & 138 & 19.4 & 11.4 & 6.10 & 5.38 & 133 & 22.7 & 12.6 & 6.25 & 5.39 & 1000 & 45.8 & 32.5 & 26.2 & 21.6 \\
\midrule
3 & 0 & RGD & 33.5 & 26.4 & 25.2 & 24.1 & 23.3 & 37.8 & 33.2 & 28.7 & 23.5 & 21.4 & 1000 & 621 & 238 & 170 & 150 \\
 &  & RGN & 33.5 & 25.4 & 24.8 & 23.3 & 23.4 & 37.8 & 33.2 & 27.8 & 22.4 & 21.3 & 1000 & 440 & 228 & 143 & 96.3 \\
 &  & CP-ALS & 33.5 & 24.8 & 23.3 & 23.4 & 23.4 & 37.8 & 30.5 & 27.7 & 25.5 & 25.5 & 1000 & 155 & 140 & 82.8 & 75.8 \\
\addlinespace
 & 0.25 & RGD & 33.9 & 25.4 & 24.3 & 23.8 & 23.3 & 97.7 & 35.7 & 30.4 & 26.3 & 24.5 & 1000 & 535 & 202 & 133 & 85.1 \\
 &  & RGN & 33.9 & 24.4 & 24.1 & 23.3 & 23.4 & 97.7 & 32.6 & 29.7 & 25.3 & 24.4 & 1000 & 294 & 182 & 102 & 61.3 \\
 &  & CP-ALS & 33.9 & 24.0 & 23.3 & 23.4 & 23.4 & 97.7 & 29.7 & 27.3 & 26.5 & 26.6 & 1000 & 121 & 91.6 & 60.3 & 60.2 \\
\addlinespace
 & 0.5 & RGD & 65.8 & 36.5 & 32.9 & 30.5 & 29.7 & 154 & 54.7 & 41.6 & 35.9 & 35.5 & 1000 & 592 & 206 & 125 & 99.5 \\
 &  & RGN & 65.8 & 32.5 & 30.9 & 30.3 & 29.9 & 154 & 37.6 & 36.5 & 35.5 & 35.4 & 1000 & 392 & 177 & 100 & 68.5 \\
 &  & CP-ALS & 65.8 & 30.6 & 29.8 & 29.9 & 29.0 & 154 & 36.9 & 35.1 & 32.9 & 32.1 & 1000 & 122 & 77.9 & 53.4 & 41.6 \\
\addlinespace
 & 0.75 & RGD & 128 & 44.6 & 39.2 & 34.6 & 32.1 & 133 & 71.6 & 54.8 & 39.5 & 33.0 & 1000 & 611 & 191 & 105 & 68.8 \\
 &  & RGN & 128 & 33.5 & 31.9 & 31.2 & 31.4 & 133 & 57.3 & 36.4 & 31.5 & 31.3 & 1000 & 397 & 147 & 98.5 & 64.3 \\
 &  & CP-ALS & 128 & 35.9 & 32.6 & 31.6 & 31.6 & 133 & 37.8 & 33.0 & 31.4 & 31.5 & 1000 & 78.6 & 54.6 & 44.2 & 35.7 \\
\bottomrule
\end{tabular}}
\end{table}

\section{Real Data Analysis: International Trade Tensor Regression}
\label{appendix:real_data}

We evaluate the proposed methods on a tensor-on-tensor regression problem constructed from international trade data. Each observational unit is an exporter country. The lagged design uses 2018 covariates and trade-exposure weights to predict a 2019 tensor of export outcomes indexed by importer partner, product section, and response channel. The scientific question is whether a low-dimensional summary of bilateral trade costs, market size, and policy variables can predict this structured outcome. This setting is substantially more heterogeneous than the Gaussian designs used in the simulation study: trade outcomes are sparse, strongly right-skewed before transformation, and organized by both geography and product composition.

\paragraph{Data sources and tensor construction.}
The raw data are assembled from public international databases. Exporter--importer--HS6 trade values and quantities are downloaded from \href{https://www.cepii.fr/CEPII/en/bdd_modele/bdd_modele_item.asp?id=37}{CEPII BACI HS17} \citep{gaulier2010baci}. Bilateral distance, geography, and trade-agreement covariates are from the \href{https://www.cepii.fr/CEPII/en/bdd_modele/bdd_modele_item.asp?id=8}{CEPII Gravity} database \citep{conte2022cepii}. Most-favored-nation (MFN) tariff schedules are obtained from \href{https://wits.worldbank.org/}{World Bank WITS}, which disseminates UNCTAD TRAINS/WTO tariff data. The logistics-friction proxy is based on the \href{https://lpi.worldbank.org/international}{World Bank Logistics Performance Index}. GDP, GDP per capita, and population are from the \href{https://databank.worldbank.org/source/world-development-indicators}{World Bank World Development Indicators}.

We aggregate HS6 products to the 21 standard HS1 sections and retain the top 25 importer partners by 2018 import volume after requiring coverage for gravity, macroeconomic, tariff, and logistics inputs. For exporter $m$, the 2019 response tensor is
\[
  \calY_m \in \RR^{25 \times 21 \times 8},
\]
where the three modes correspond to the importer partner, the HS1 product section, and the response channel. The eight response channels are log export value, log quantity, log unit value, exporter share, partner market share, Balassa revealed comparative advantage, extensive margin, and intensive margin. The raw 2018 covariate tensor contains six channels: MFN tariff, logistics-friction proxy, partner log GDP, log bilateral distance, regional-trade-agreement scope, and partner world demand. Agreement scope is derived from CEPII's coverage field and is coded as 0 for no recorded agreement, 1 for goods only, and 2 for goods and services. For the tensor-on-tensor regression operator used here, these cell-level covariates are summarized for each exporter by a vector $\calX_m\in\RR^6$ with weights given by the square root of the exporter's 2018 trade value in each cell, while the full $25\times21\times8$ response tensor is retained. Under the input-first convention of Section~\ref{sec:model}, the resulting coefficient tensor has dimension
\[
  \calT^\star \in \RR^{6 \times 25 \times 21 \times 8}.
\]
We retain the 80 largest exporters in 2018 that meet the coverage requirements in Appendix~\ref{appendix:real_data_impl}.

\paragraph{Methods and evaluation protocol.}
We compare eight estimators. In the CP format, these are RGD and RGN, CP-ALS \citep{kolda2009tensor} and RRR \citep{lock2018tensor}, all with CP rank 3. In other tensor formats, Tucker-RGD \citep{luo2024tensor} and Tucker-RGN \citep{luo2023low} use multilinear rank $(3,3,3,3)$, TT-IHT \citep{qin2024computational} uses TT rank $(3,3,3)$, and HOPLS \citep{zhao2013hopls} extracts three latent blocks. These are rank-three specifications within each method's native format, not capacity-matched models. In the notation of Appendix~\ref{sec:complexity}, one rank-one CP component has $\mathrm{df}=1+(25-1)+(21-1)+(8-1)+(6-1)=57$, so the CP-rank-3 model has 171 intrinsic degrees of freedom, and the TT-rank-$(3,3,3)$ model, with the modes ordered as $(25,21,8,6)$ in the TT-IHT implementation, has $25(3)+3(21)(3)+3(8)(3)+3(6)-3(3^2)=327$. CP-ALS, RGD and RGN share the CPCA initialization within each fold. RGD uses the fixed step size $1/\lambda_{\max}(X^\top X/n)$ for the standardized design $X$, which equals $\sigma_X^{-2}$ when $X^\top X/n=\sigma_X^2I$, and Table~\ref{tab:trade_cv_metrics} reports the fits after 30 iterations.

Prediction accuracy is evaluated by five-fold cross-validation over exporters, with folds stratified by quintiles of 2018 trade volume and shared by all methods. Table~\ref{tab:trade_cv_metrics} reports the mean and standard error of the fold-specific test RMSE and test $R^2$, computed on the centered response scale.

\paragraph{Optimization paths.}
Table~\ref{tab:trade_cv_paths} reports the held-out test RMSE after 3, 5, 10, 30 and 300 iterations. HOPLS is reported after one, two and three latent blocks, because adding a block changes the fitted model. CP-ALS, RGN and Tucker-RGN are close to their limits after a few iterations. RGD and TT-IHT take steps set by the largest eigenvalue of the design Gram matrix, so they approach their limits slowly along the other directions. Their test RMSE is smallest at about 10 iterations and increases to the level of the other methods by 300 iterations. Their lower test RMSE in Table~\ref{tab:trade_cv_metrics} therefore comes from stopping before convergence.

\begin{table}[htbp]
\centering
\caption{Mean held-out test RMSE over the five folds along the fitting paths. Panel A reports iterations, with the CP-format methods first, and Panel B reports HOPLS latent blocks.}
\label{tab:trade_cv_paths}
{\small
\begin{tabular}{lccccc}
\toprule
& \multicolumn{5}{c}{Panel A: iteration} \\
\cmidrule(lr){2-6}
Method & 3 & 5 & 10 & 30 & 300 \\
\midrule
CP-ALS     & $2.6763$ & $2.6756$ & $2.6743$ & $2.6747$ & $2.6749$ \\
RRR        & $2.6955$ & $2.6867$ & $2.6776$ & $2.6763$ & $2.6749$ \\
RGD        & $2.6688$ & $2.6646$ & $2.6613$ & $2.6657$ & $2.6746$ \\
RGN        & $2.6732$ & $2.6736$ & $2.6736$ & $2.6743$ & $2.6749$ \\
\midrule
Tucker-RGD & $2.7087$ & $2.6970$ & $2.6822$ & $2.6748$ & $2.6743$ \\
Tucker-RGN & $2.6712$ & $2.6710$ & $2.6716$ & $2.6746$ & $2.6743$ \\
TT-IHT     & $2.6669$ & $2.6653$ & $2.6630$ & $2.6663$ & $2.6772$ \\
\midrule
& \multicolumn{3}{c}{Panel B: HOPLS latent blocks} \\
\cmidrule(lr){2-4}
Method & 1 & 2 & 3 \\
\midrule
HOPLS      & $2.7093$ & $2.6965$ & $2.6795$ \\
\bottomrule
\end{tabular}
}
\end{table}

\section{Additional Details on Algorithms}

In this section, we provide additional details on the implementation of the algorithms.

\subsection{Local T-HOSVD Retraction}
\label{appendix:thosvd-retraction}

\begin{lemma}[Local T-HOSVD retraction]
\label{lemma:local-thosvd-retraction}
Let $\calT\in\operatorname{Seg}$, let $\TT_{\calT}$ be its tangent space, and let $\xi\in\TT_{\calT}$. If
\begin{equation}
\|\xi\|_{\mathrm F}
<\frac{\|\calT\|_{\mathrm F}}{2(\sqrt d+1)},
\label{eq:local-thosvd-domain}
\end{equation}
then
\[
\mathcal R_{\calT}(\xi)
:=\operatorname{T\text{-}HOSVD}_{(1,\ldots,1)}(\calT+\xi)
\]
is a nonzero rank-one tensor, and the displayed map is smooth jointly in $(\calT,\xi)$ on the tangent-bundle neighborhood defined by \eqref{eq:local-thosvd-domain}. Moreover, for every $\zeta\in\TT_{\calT}$,
\begin{equation}
\mathcal R_{\calT}(0)=\calT,
\qquad
D\mathcal R_{\calT}(0)[\zeta]=\zeta.
\label{eq:local-thosvd-first-order}
\end{equation}
Furthermore,
\begin{equation}
\|\mathcal R_{\calT}(\xi)-(\calT+\xi)\|_{\mathrm F}
\leq\frac{\|\xi\|_{\mathrm F}^2}{\|\calT\|_{\mathrm F}}.
\label{eq:local-thosvd-quadratic}
\end{equation}
\end{lemma}

\begin{proof}
Set $\lambda=\|\calT\|_{\mathrm F}$ and $\calZ=\calT+\xi$. For each $k\in[d]$, Weyl's inequality gives, with singular values extended by zeros,
\[
\sigma_1\{\mat_k(\calZ)\}\geq\lambda-\|\xi\|_{\mathrm F},
\qquad
\sigma_2\{\mat_k(\calZ)\}\leq\|\xi\|_{\mathrm F}.
\]
Under \eqref{eq:local-thosvd-domain}, $\|\xi\|_{\mathrm F}<\lambda/2$, so the preceding display gives
\[
\sigma_1\{\mat_k(\calZ)\}\geq\lambda-\|\xi\|_{\mathrm F}>\|\xi\|_{\mathrm F}\geq\sigma_2\{\mat_k(\calZ)\},
\qquad k\in[d].
\]
Hence the projector $P_k(\calZ)$ onto the leading left singular vector of $\mat_k(\calZ)$ is uniquely defined and smooth in $\calZ$, and
\[
\mathcal R_{\calT}(\xi)
=\calZ\times_1P_1(\calZ)\times_2\cdots\times_dP_d(\calZ)
\]
is smooth jointly in $(\calT,\xi)$ on the displayed domain. The T-HOSVD quasi-best approximation bound \citep{de2000multilinear,kressner2014low} yields
\begin{equation}
\|\calZ-\mathcal R_{\calT}(\xi)\|_{\mathrm F}
\leq\sqrt d\inf_{V\in\operatorname{Seg}}
\|\calZ-V\|_{\mathrm F}
\leq\sqrt d\|\xi\|_{\mathrm F}.
\label{eq:local-thosvd-quasi-best}
\end{equation}
Consequently,
\[
\|\mathcal R_{\calT}(\xi)-\calT\|_{\mathrm F}
\leq(\sqrt d+1)\|\xi\|_{\mathrm F}
<\frac{\lambda}{2},
\]
so $\mathcal R_{\calT}(\xi)\ne0$ and therefore belongs to $\operatorname{Seg}$.

Write $\calT=\lambda u_1\otimes\cdots\otimes u_d$, where $\|u_k\|=1$ for every $k\in[d]$. Since $\xi\in\TT_{\calT}$,
\[
\calZ=a\,u_1\otimes\cdots\otimes u_d
+\sum_{k=1}^d u_1\otimes\cdots\otimes w_k\otimes\cdots\otimes u_d,
\qquad w_k\perp u_k.
\]
Set $\psi=\sum_{k=1}^d\|w_k\|^2$. Orthogonality gives $\|\xi\|_{\mathrm F}^2=(a-\lambda)^2+\psi$ and $a\geq\lambda-\|\xi\|_{\mathrm F}>0$. For $w_k\ne0$, the matrix $\mat_k(\calZ)\mat_k(\calZ)^\top$ is supported on $\operatorname{span}(u_k,w_k)$ and has representation
\begin{equation}
\begin{pmatrix}
a^2+\psi-\|w_k\|^2 & a\|w_k\|\\
a\|w_k\| & \|w_k\|^2
\end{pmatrix}
\label{eq:rank-one-tangent-mode-gram}
\end{equation}
in the orthonormal basis $(u_k,w_k/\|w_k\|)$. Its trace is $a^2+\psi$, and its determinant is $\|w_k\|^2(\psi-\|w_k\|^2)\leq \psi^2/4$. Consequently,
\[
\sigma_1^2\{\mat_k(\calZ)\}
\geq\frac{a^2+\psi+a\sqrt{a^2+2\psi}}2
=\frac{(a+\sqrt{a^2+2\psi})^2}{4},
\]
and
\[
\|\calZ-\calZ\times_kP_k(\calZ)\|_{\mathrm F}^2
=\sigma_2^2\{\mat_k(\calZ)\}
\leq
\frac{4\|w_k\|^2(\psi-\|w_k\|^2)}
{(a+\sqrt{a^2+2\psi})^2}.
\]
When $w_k=0$, both sides of the last inequality are zero. The T-HOSVD error bound gives
\begin{align*}
\|\calZ-\mathcal R_{\calT}(\xi)\|_{\mathrm F}^2
&\leq\sum_{k=1}^d
\|\calZ-\calZ\times_kP_k(\calZ)\|_{\mathrm F}^2\\
&\leq\frac{4\psi^2}{(a+\sqrt{a^2+2\psi})^2}
=\bigl(\sqrt{a^2+2\psi}-a\bigr)^2.
\end{align*}
Moreover,
\[
\left(a+\frac{\|\xi\|_{\mathrm F}^2}{\lambda}\right)^2-(a^2+2\psi)
=2\left(a-\lambda+\frac{\|\xi\|_{\mathrm F}^2}{2\lambda}\right)^2
+\frac{\|\xi\|_{\mathrm F}^4}{2\lambda^2}\geq0.
\]
Since $a>0$, these inequalities prove \eqref{eq:local-thosvd-quadratic}.

At $\xi=0$, each $P_k(\calT)$ fixes the corresponding mode factor, so $\mathcal R_{\calT}(0)=\calT$. For every $\zeta\in\TT_{\calT}$ and all sufficiently small $|t|$, \eqref{eq:local-thosvd-quadratic} gives
\[
\|\mathcal R_{\calT}(t\zeta)-(\calT+t\zeta)\|_{\mathrm F}
\leq\frac{t^2\|\zeta\|_{\mathrm F}^2}{\|\calT\|_{\mathrm F}},
\]
which proves $D\mathcal R_{\calT}(0)[\zeta]=\zeta$.
\end{proof}

\subsection{Code Implementation Details}
\label{appendix:algorithm_implementation}

\paragraph{Exact structured T-HOSVD step.}
The mode Gram matrices in \eqref{eq:rank-one-tangent-mode-gram} reduce the T-HOSVD of $\calZ=\calT+\xi$ to at most $d$ two-dimensional symmetric eigenproblems. When $w_k=0$, the leading mode vector is $u_k$. Contracting $\calZ$ with the resulting vectors gives the T-HOSVD weight. This evaluates the rank-$(1,\ldots,1)$ T-HOSVD without forming $\calZ$ or any full unfolding.

\paragraph{Tensor-on-tensor regression.}
Let $X\in\RR^{n\times p_{\rm in}^*}$ have $m$th row $\operatorname{vec}(\calX_m)^\top/\sqrt n$. For component $i$, let $u_{{\rm in},i}$ and $u_{{\rm out},i}$ be the Kronecker products of its input and output factors, let $V_{{\rm in},i}\in\RR^{p_{\rm in}^*\times\mathrm{df}_{\rm in}}$ be an orthonormal input-tangent basis whose first column is $u_{{\rm in},i}$, and let $V_{{\rm out},i}^{\perp}\in\RR^{p_{\rm out}^*\times(\mathrm{df}_{\rm out}-1)}$ have orthonormal columns spanning the output-tangent subspace orthogonal to $u_{{\rm out},i}$. The implementation stores the output modes first. Under this fixed mode permutation, the observation operator and its adjoint, tangent projection, and T-HOSVD step yield exactly the permuted RGD/RGN iterates. In these coordinates, the projected design is
\begin{align*}
\mathsf D_i&=\left[u_{{\rm out},i}\otimes(XV_{{\rm in},i}),\quad
V_{{\rm out},i}^{\perp}\otimes(Xu_{{\rm in},i})\right],\\
\mathsf D_i^\top \mathsf D_i&=\operatorname{diag}\!\left\{(XV_{{\rm in},i})^\top XV_{{\rm in},i},\
\ \|Xu_{{\rm in},i}\|^2I_{\mathrm{df}_{\rm out}-1}\right\}.
\end{align*}
In these coordinates, $\mathsf D_i^\top \mathsf D_i$ represents the restricted normal operator $\left.\calP_{\widehat\TT_i}\calA^*\calA\calP_{\widehat\TT_i}\right|_{\widehat\TT_i}$. The implementation forms this Gram matrix and the corresponding right-hand side by modewise contractions, without constructing $\mathsf D_i$ or the ambient tangent basis. These identities are exact: fixed-step RGD uses the right-hand side, whereas RGN solves the $\mathrm{df}_{\rm in}\times\mathrm{df}_{\rm in}$ block and divides each output coordinate by $\|Xu_{{\rm in},i}\|^2$. The structured T-HOSVD step above is then applied directly in tangent coordinates.

For RGN, whenever the restricted normal operator is invertible, $\widehat\calT_i\in\widehat\TT_i$ implies
\[
\mathop{\rm argmin}_{Z\in\widehat\TT_i}
\frac12\left\|\calA(Z)-\left\{\calY-\sum_{j\ne i}\calA(\widehat\calT_j)\right\}\right\|_{\mathrm F}^{2}
=\widehat\calT_i-\Delta_i,
\]
where $\Delta_i$ is the RGN direction in Section~\ref{subsec:RGN} and Algorithm~\ref{alg:RGN_CP}. Under \eqref{eq:local-thosvd-domain}, applying the structured T-HOSVD to this tangent-space least-squares solution is exactly $\mathcal R_{\widehat\calT_i}(-\Delta_i)$.

\paragraph{Baselines.} The competing methods are not part of \texttt{SegreCP} and are provided in the code repository. In regression, RRR is the reduced-rank regression of \citet{lock2018tensor}, Tucker-RGD and Tucker-RGN are Riemannian gradient descent and Gauss-Newton on the manifold of fixed Tucker rank \citep{luo2024tensor,luo2023low}, TT-IHT is iterative hard thresholding in the tensor-train format \citep{qin2024computational}, and HOPLS is higher-order partial least squares \citep{zhao2013hopls}. In decomposition, Orth-ALS is orthogonalized ALS \citep{sharan2017orthogonalized}, ICO is iterative concurrent orthogonalization \citep{han2022tensor}, NLS-GN and GN-CG are Gauss-Newton methods for the nonlinear least-squares formulation \citep{sorber2013optimization,singh2021comparison}, CP-ALS+ELS is CP-ALS with enhanced line search \citep{rajih2008enhanced}, and Hybrid AMDM moves from Mahalanobis distance minimization to ALS by inverting one fewer singular value of each factor every ten iterations \citep{singh2023alternating}. Orth-ALS follows the MATLAB implementation of \citet{sharan2017orthogonalized}, CP-ALS follows the \texttt{cp} function of the \texttt{rTensor} package \citep{li2018rtensor}, RRR uses the \texttt{MultiwayRegression} package of \citet{lock2018tensor}, and Tucker-RGN follows the code of \citet{luo2023low}. The other baselines are our R implementations of the cited algorithms.

\paragraph{Compute resources.} All experiments were run in R on one machine with an Intel Xeon 2.1\,GHz CPU and 8\,GB of memory, with one core per fit.

\subsection{Additional Details on Real Data Analysis}
\label{appendix:real_data_impl}

\paragraph{Raw inputs and preprocessing.}
The pipeline uses 2018 to construct the covariates and lagged trade-exposure weights and 2019 to construct the response tensor. HS6 products are mapped to the 21 HS1 sections. Trade values and quantities are summed within each exporter--importer--HS1 cell. MFN tariffs are aggregated to importer--HS1 cells using import-value-weighted averages with a simple-mean fallback, and 2018 world demand is computed as the log total imports of each importer in each HS1 section. The share and revealed-comparative-advantage response channels use totals from the full 2019 BACI cross-section as their denominators.

\paragraph{Sample construction.}
The partner mode consists of the top 25 importer partners by 2018 aggregate import volume subject to macroeconomic, tariff, logistics, and gravity-data availability. Exporters must trade positively with at least 13 selected partners and at least 11 HS1 sections and have complete required macroeconomic and selected-partner distance information. The top 80 eligible exporters by 2018 export value are retained. The final covariate and response tensors have dimensions $80\times25\times21\times6$ and $80\times25\times21\times8$.

\paragraph{Regression representation.}
For each exporter $m$, the raw 2018 covariate tensor records six channels over partner and HS1 section: MFN tariff, logistics-friction proxy, partner log GDP, log distance, regional-trade-agreement scope, and world demand. The agreement-scope channel is coded as 0 for no recorded agreement, 1 for goods only, and 2 for goods and services. Because the solver uses one input mode, the cell-level covariates are collapsed into a six-dimensional vector by averaging over partner--section cells with weights given by the square root of the exporter's 2018 trade value in each cell. Within each cross-validation split, feature standardization and response-channel centering are estimated from the training exporters and then applied to the held-out exporters. The response tensor keeps the full partner--section--channel resolution. The resulting tensor-on-tensor model is
\[
  \calY_m=\langle\calX_m,\calT^\star\rangle_{\rm in}+\calE_m,
  \qquad \calX_m\in\RR^6,
  \qquad \calT^\star\in\RR^{6\times25\times21\times8},
\]
where the first mode of $\calT^\star$ indexes input features and the remaining modes index response coordinates.

\subsection{Tensor CP Decomposition}

\paragraph{Initialization for CP decomposition.}\label{appendix:cpca-initialization} Our convergence results apply to any initialization satisfying the stated componentwise local condition. We use CPCA (factor construction: \citep[Algorithm~3]{han2022tensor}) because \citet{han2022tensor} establish componentwise recovery guarantees for noisy non-orthogonal CP decomposition under mild incoherence. In their numerical settings, CPCA also yields smaller initialization error than T-HOSVD \citep{de2000multilinear}. With $\operatorname{vec}(\calE)\sim\mathcal N(0,\sigma^2I_{p^*})$, Theorem~5 of \citet{han2022tensor} shows that, with high probability, the CPCA factors satisfy $\max_{i\in[r],k\in[d]}\|\widehat u_{k,i}^{\mathrm{cpca}}\widehat u_{k,i}^{\mathrm{cpca}\top}-u_{k,i}u_{k,i}^\top\|\leq\tau$, where $\tau$ is given there. When $p_S=\sqrt{p^*}$ in Algorithm~\ref{alg:cpca} and the factors are nearly orthogonal, $\tau$ is of order $\sigma(p^*)^{1/4}/\min_{i\in[r]}(\lambda_i-\lambda_{i+1})$, with $\lambda_{r+1}=0$. Write the CPCA initialization as $\calT_i^{(0)}=\langle\calY,\widehat U_i\rangle\widehat U_i$, where $\widehat U_i=\widehat u_{1,i}^{\mathrm{cpca}}\otimes\cdots\otimes\widehat u_{d,i}^{\mathrm{cpca}}$. If $d^2(1+\kappa r)\tau\leq c$, expanding $\langle\calY,\widehat U_i\rangle$ and bounding $\langle\calE,\widehat U_i\rangle$ by Lemma~\ref{lemma:local-gaussian-tangent} give, with high probability,
\[
\varepsilon^{(0)}\leq\sqrt2d\tau+d\tau^2+(r-1)\kappa(\eta+\sqrt2\tau)^d+C\frac{\sigma\sqrt{\mathrm{df}}}{\lambda_r}.
\]
Hence CPCA satisfies the initialization condition of Corollaries~\ref{corollary:rate_RGD_decomposition} and~\ref{cor:RGN1} if, in addition, $d\kappa r(1+\kappa r)(\eta+\sqrt2\tau)^d\leq c$ and $\sigma\sqrt{\mathrm{df}}/\lambda_r\leq c/\{d(1+\kappa r)\}$, where $c>0$ is a sufficiently small universal constant.

\begin{algorithm}[ht]
\caption{Composite PCA (CPCA) initialization used in the experiments \citep{han2022tensor}}
\label{alg:cpca}
\textbf{Input:} Noisy tensor $\calY$, target CP rank $r$, and an optional nonempty proper subset $S\subsetneq[d]$
\begin{algorithmic}[1]
\IF{$S$ is not supplied}
  \STATE Choose $S$ to minimize $|\prod_{k\in S}p_k-\sqrt{p^*}|$ over nonempty proper subsets of $[d]$, where $p^*=\prod_{k=1}^dp_k$.
\ENDIF
\STATE Set $S^c=[d]\setminus S$ and $p_S=\prod_{k\in S}p_k$, with $r\leq\min\{p_S,p^*/p_S\}$.
\STATE Compute the top-$r$ singular triplets $\{(\widehat s_i,\widehat u_i,\widehat v_i)\}_{i=1}^r$ of the $p_S\times(p^*/p_S)$ matricization $\mat_S(\calY)$, whose rows and columns index the modes in $S$ and $S^c$, respectively.
\FOR{$i = 1, \dots, r$}
  \STATE Refold $\widehat u_i$ over the modes in $S$ and $\widehat v_i$ over the modes in $S^c$, and extract their leading modewise left singular vectors $\{\widehat u_{k,i}^{\mathrm{cpca}}\}_{k=1}^d$.
  \STATE Choose their signs so that
  $\langle\widehat u_i,\bigotimes_{k\in S}\widehat u_{k,i}^{\mathrm{cpca}}\rangle\geq0$
  and
  $\langle\widehat v_i,\bigotimes_{k\in S^c}\widehat u_{k,i}^{\mathrm{cpca}}\rangle\geq0$.
  \STATE Set $\widetilde\lambda_i=\langle\calY,\bigotimes_{k=1}^d\widehat u_{k,i}^{\mathrm{cpca}}\rangle$. If $\widetilde\lambda_i<0$, replace $\widehat u_{1,i}^{\mathrm{cpca}}$ by $-\widehat u_{1,i}^{\mathrm{cpca}}$. Set $\widehat\lambda_i^{\mathrm{cpca}}=|\widetilde\lambda_i|$.
\ENDFOR
\end{algorithmic}
\textbf{Output:} $\{\widehat\lambda_i^{\mathrm{cpca}},\widehat u_{k,i}^{\mathrm{cpca}}\}_{i\in[r],k\in[d]}$.
\end{algorithm}

\paragraph{Riemannian Gradient Descent for tensor decomposition.} 

\begin{algorithm}[h!]
\caption{Riemannian Gradient Descent for CP Tensor Decomposition}
\label{alg:RGD_CP_decomposition}
\textbf{Input:} Noisy tensor $\calY$, target CP rank $r$, step size $\alpha_t$, and initial rank-one tensors $\left\{\calT_i^{(0)}\right\}_{i=1}^r$. \\
\begin{algorithmic}[1]
\STATE Compute the residual $\mathrm{res}\gets\sum_{j=1}^r\calT_j^{(0)}-\calY$.
\FOR{$t = 0, 1, \dots, t_{\max}-1$}
\FOR{$i = 1, \dots, r$}
  \STATE \textbf{(RGD Update)} Compute
  \[
    \calT_i^{(t + 1)} = \mathcal{R}_{\calT_i^{(t)}}\left(-\alpha_t\calP_{\TT_i^{(t)}}\mathrm{res}\right).
  \]
  \STATE Update $\mathrm{res}\gets\mathrm{res}+\calT_i^{(t+1)}-\calT_i^{(t)}$.
  \ENDFOR
\ENDFOR

\end{algorithmic}
\textbf{Output:} $\left\{\calT_i^{(t_{\max})}\right\}_{i=1}^r$.
\end{algorithm}

In the CP decomposition experiments, RGD uses the constant step size $\alpha_t\equiv0.5$. It satisfies $0<\alpha_t\leq1$ and $|1-\alpha_t|\leq1/2$ in Corollary~\ref{corollary:rate_decomposition_appendix}, and the factor $1-3\alpha_t/4$ in that corollary equals $5/8$.

\paragraph{Riemannian Gauss-Newton for tensor decomposition.}

For tensor decomposition, Riemannian Gauss-Newton is Riemannian gradient descent with step size $\alpha_t\equiv 1$.

\begin{algorithm}[h!]
\caption{Riemannian Gauss-Newton for CP Tensor Decomposition}
\label{alg:RGN_CP_decomposition}
\textbf{Input:} Noisy tensor $\calY$, target CP rank $r$, and initial rank-one tensors $\left\{\calT_i^{(0)}\right\}_{i=1}^r$. \\
\begin{algorithmic}[1]
\STATE Compute the residual $\mathrm{res}\gets\sum_{j=1}^r\calT_j^{(0)}-\calY$.
\FOR{$t = 0, 1, \dots, t_{\max}-1$}
\FOR{$i = 1, \dots, r$}
  \STATE \textbf{(RGN Update)} 
  \[
    \calT_i^{(t + 1)} = \mathcal{R}_{\calT_i^{(t)}}\left(-\calP_{\TT_i^{(t)}}\mathrm{res}\right).
  \]
  \STATE Update $\mathrm{res}\gets\mathrm{res}+\calT_i^{(t+1)}-\calT_i^{(t)}$.
  \ENDFOR
\ENDFOR

\end{algorithmic}
\textbf{Output:} $\left\{\calT_i^{(t_{\max})}\right\}_{i=1}^r$.
\end{algorithm}

\subsection{Tensor-on-Tensor Regression}

\paragraph{Initialization for tensor-on-tensor regression.}

To estimate the low-rank tensor coefficient in a regression setting, we adopt an initialization strategy based on the adjoint operator of the linear map $\calA$ induced by the covariates $\left\{\calX_m\right\}_{m=1}^n$. Specifically, we use the scaled empirical cross-moment $\widetilde{\calT}:=\sigma_X^{-2}\calA^*(\calY)$, which satisfies $\E(\widetilde{\calT})=\calT^\star$ under the isotropic design and independent mean-zero noise assumed above \cite{han2022optimal,zhang2020islet}. We then apply CPCA \citep{han2022tensor} to $\widetilde{\calT}$ to obtain the initial component weights and mode factors.

\begin{algorithm}[ht]
\caption{Initialization of Low-Rank Tensor-on-Tensor Regression}
\label{alg:init-low-rank-tensor}
\textbf{Input:} Observations $\{\calX_m, \calY_m\}_{m=1}^n$, design scale $\sigma_X$, and target CP rank $r$
\begin{algorithmic}[1]
\STATE Compute 
$$
\widetilde{\calT}=\sigma_X^{-2}\calA^*(\calY)
=\frac{1}{n\sigma_X^2}\sum_{m=1}^n \calX_m\otimes\calY_m
$$
\STATE Compute $\{\widehat\lambda_i^{\mathrm{cpca}},\widehat u_{k,i}^{\mathrm{cpca}}\}_{i\in[r],k\in[d]}\gets\mathrm{CPCA}(\widetilde{\calT})$, where CPCA is given in Algorithm~\ref{alg:cpca} \citep{han2022tensor}.
\STATE Set $\calT_i^{(0)}=\widehat\lambda_i^{\mathrm{cpca}}\widehat u_{1,i}^{\mathrm{cpca}}\otimes\cdots\otimes\widehat u_{d,i}^{\mathrm{cpca}}$ for every $i\in[r]$.
\end{algorithmic}

\textbf{Output:} $\left\{\calT_i^{(0)}\right\}_{i=1}^r$
\end{algorithm}

\paragraph{Riemannian Gradient Descent for tensor-on-tensor regression.}

Using the normalized response $\calY$ and the operators $\calA,\calA^*$ defined in Section~\ref{sec:model}, RGD updates each rank-one component by tangent projection and retraction.

\begin{algorithm}[h!]
\caption{Riemannian Gradient Descent for CP Tensor-on-Tensor Regression}
\label{alg:RGD_CP_regression}
\textbf{Input:} Normalized response $\calY\in\calH$, linear operator $\calA$ defined in Section~\ref{sec:model}, target CP rank $r$, step size $\alpha_t$, and initial rank-one tensors $\left\{\calT_i^{(0)}\right\}_{i=1}^r$. \\
\begin{algorithmic}[1]
\STATE Compute the residual $\mathrm{res}\gets\calA\bigl(\sum_{j=1}^r\calT_j^{(0)}\bigr)-\calY$.
\FOR{$t = 0, 1, \dots, t_{\max}-1$}
\FOR{$i = 1, \dots, r$}
  \STATE \textbf{(RGD Update)} Compute
  \[
    \calT_i^{(t + 1)} = \mathcal{R}_{\calT_i^{(t)}}\left(-\alpha_t\calP_{\TT_i^{(t)}}\calA^*\bigl(\mathrm{res}\bigr)\right).
  \]
  \STATE Update $\mathrm{res}\gets\mathrm{res}+\calA\bigl(\calT_i^{(t+1)}-\calT_i^{(t)}\bigr)$.
  \ENDFOR
\ENDFOR

\end{algorithmic}
\textbf{Output:} $\left\{\calT_i^{(t_{\max})}\right\}_{i=1}^r$.
\end{algorithm}

\paragraph{Riemannian Gauss-Newton for tensor-on-tensor regression.}

\begin{algorithm}[h!]
\caption{Riemannian Gauss-Newton for CP Tensor-on-Tensor Regression}
\label{alg:RGN_CP_regression}
\textbf{Input:} Normalized response $\calY\in\calH$, linear operator $\calA$ defined in Section~\ref{sec:model}, target CP rank $r$, and initial rank-one tensors $\left\{\calT_i^{(0)}\right\}_{i=1}^r$. \\
\begin{algorithmic}[1]
\STATE Compute the residual $\mathrm{res}\gets\calA\bigl(\sum_{j=1}^r\calT_j^{(0)}\bigr)-\calY$.
\FOR{$t = 0, 1, \dots, t_{\max}-1$}
\FOR{$i = 1, \dots, r$}
  \STATE \textbf{(RGN Update)} Whenever the restricted normal operator below is invertible, solve for the unique $\Delta_i^{(t)}\in\TT_i^{(t)}$ satisfying
  \[
  \left(\left.\calP_{\TT_i^{(t)}}\calA^*\calA\calP_{\TT_i^{(t)}}\right|_{\TT_i^{(t)}}\right)\Delta_i^{(t)}
  =\calP_{\TT_i^{(t)}}\calA^*\bigl(\mathrm{res}\bigr).
  \]
  \STATE Update $\calT_i^{(t+1)}\gets\mathcal{R}_{\calT_i^{(t)}}\bigl(-\Delta_i^{(t)}\bigr)$.
  \STATE Update $\mathrm{res}\gets\mathrm{res}+\calA\bigl(\calT_i^{(t+1)}-\calT_i^{(t)}\bigr)$.
\ENDFOR
\ENDFOR

\end{algorithmic}
\textbf{Output:} $\left\{\calT_i^{(t_{\max})}\right\}_{i=1}^r$.
\end{algorithm}

An orthonormal basis of \(\TT_i^{(t)}\) reduces the displayed normal equation to a $\mathrm{df}\times\mathrm{df}$ linear system, where $\mathrm{df}=1+\sum_{l=1}^d(p_l-1)$.

\section{Proof of Main Theorems}
\label{sec:proof_main_theorems}

In this section, we provide the proofs of error bounds incurred by Riemannian updates.

\begin{proof}[Proof of Theorem~\ref{thm:local_convergence_rgd}]

Fix $i\in[r]$. The condition $\varepsilon\leq1/(4d)$ and Remark~\ref{remark:sign-alignment} permit Lemma~\ref{lemma: perturbation bound} to be applied to $\{\widehat{\calT}_j\}_{j=1}^r$. Define
\[
\widehat G_i
=
\calP_{\widehat{\TT}_i}\calA^*
\left\{\calA\left(\sum_{j=1}^r\widehat{\calT}_j\right)-\calY\right\},
\qquad
\widehat W_i=\widehat{\calT}_i-\alpha_t\widehat G_i,
\qquad
\widehat{\calT}_i^+=\mathcal R_{\widehat{\calT}_i}(-\alpha_t\widehat G_i).
\]
Since $\calY=\calA(\sum_{j=1}^r\calT_j)+\calE$ and $\widehat\xi_i=-\alpha_t\widehat G_i$, \eqref{eq:local-thosvd-quadratic} gives
\begin{align}
\|\widehat{\calT}_i^+-\calT_i\|_{\mathrm F}
&\leq\|\widehat W_i-\calT_i\|_{\mathrm F}
+\frac{\|\widehat\xi_i\|_{\mathrm F}^2}
{\|\widehat{\calT}_i\|_{\mathrm F}}\nonumber\\
&\leq
\left\|
(\widehat{\calT}_i-\calT_i)
-\alpha_t\calP_{\widehat{\TT}_i}\calA^*\calA
\sum_{j=1}^r(\widehat{\calT}_j-\calT_j)
\right\|_{\mathrm F}
\nonumber\\
&\quad+\alpha_t\|\calP_{\widehat{\TT}_i}\calA^*\calE\|_{\mathrm F}
+\frac{\|\widehat\xi_i\|_{\mathrm F}^2}
{\|\widehat{\calT}_i\|_{\mathrm F}}.
\label{eq:rgd-current-quadratic-retraction}
\end{align}
Since $\calT_i\in\operatorname{Seg}$, the quasi-best approximation bound \eqref{eq:local-thosvd-quasi-best} gives $\|\widehat{\calT}_i^+-\widehat W_i\|_{\mathrm F}\leq\sqrt d\|\widehat W_i-\calT_i\|_{\mathrm F}$, so
\begin{align}
\|\widehat{\calT}_i^+-\calT_i\|_{\mathrm F}
&\leq(\sqrt d+1)\|\widehat W_i-\calT_i\|_{\mathrm F}\nonumber\\
&\leq(\sqrt d+1)
\left\|
(\widehat{\calT}_i-\calT_i)
-\alpha_t\calP_{\widehat{\TT}_i}\calA^*\calA
\sum_{j=1}^r(\widehat{\calT}_j-\calT_j)
\right\|_{\mathrm F}
\nonumber\\
&\quad+(\sqrt d+1)\alpha_t
\|\calP_{\widehat{\TT}_i}\calA^*\calE\|_{\mathrm F}.
\label{eq:rgd-current-thosvd}
\end{align}
Since $\calP_{\widehat{\TT}_i}^2=\calP_{\widehat{\TT}_i}$ and $\calP_{\widehat{\TT}_j}+\calP_{\widehat{\TT}_j}^{\perp}=I$ for $j\in[r]$,
\begin{align*}
&(\widehat{\calT}_i-\calT_i)
-\alpha_t\calP_{\widehat{\TT}_i}\calA^*\calA\sum_{j=1}^r(\widehat{\calT}_j-\calT_j)\\
&\quad=\calP_{\widehat{\TT}_i}(I-\alpha_t\calA^*\calA)\calP_{\widehat{\TT}_i}(\widehat{\calT}_i-\calT_i)
+\calP_{\widehat{\TT}_i}^{\perp}(\widehat{\calT}_i-\calT_i)
-\alpha_t\calP_{\widehat{\TT}_i}\calA^*\calA\calP_{\widehat{\TT}_i}^{\perp}(\widehat{\calT}_i-\calT_i)\\
&\qquad-\alpha_t\calP_{\widehat{\TT}_i}\calA^*\calA\calP_{\widehat{\TT}_i}\sum_{j\ne i}(\widehat{\calT}_j-\calT_j)
-\alpha_t\calP_{\widehat{\TT}_i}\calA^*\calA\calP_{\widehat{\TT}_i}^{\perp}\sum_{j\ne i}\calP_{\widehat{\TT}_j}(\widehat{\calT}_j-\calT_j)\\
&\qquad-\alpha_t\calP_{\widehat{\TT}_i}\calA^*\calA\calP_{\widehat{\TT}_i}^{\perp}\sum_{j\ne i}\calP_{\widehat{\TT}_j}^{\perp}(\widehat{\calT}_j-\calT_j),
\end{align*}
so the triangle inequality bounds the norm in \eqref{eq:rgd-current-thosvd} by $\mathrm I+\mathrm{II}+\alpha_t(\mathrm{III}+\mathrm{IV}+\mathrm V+\mathrm{VI})$, where
\begin{align*}
\mathrm I
&=
\left\|
\calP_{\widehat{\TT}_i}(I-\alpha_t\calA^*\calA)
\calP_{\widehat{\TT}_i}(\widehat{\calT}_i-\calT_i)
\right\|_{\mathrm F},\\
\mathrm{II}
&=
\|\calP_{\widehat{\TT}_i}^{\perp}(\widehat{\calT}_i-\calT_i)\|_{\mathrm F},\\
\mathrm{III}
&=
\|\calP_{\widehat{\TT}_i}\calA^*\calA
\calP_{\widehat{\TT}_i}^{\perp}(\widehat{\calT}_i-\calT_i)\|_{\mathrm F},\\
\mathrm{IV}
&=
\left\|\calP_{\widehat{\TT}_i}\calA^*\calA
\calP_{\widehat{\TT}_i}
\sum_{j\ne i}(\widehat{\calT}_j-\calT_j)\right\|_{\mathrm F},\\
\mathrm V
&=
\left\|\calP_{\widehat{\TT}_i}\calA^*\calA
\calP_{\widehat{\TT}_i}^{\perp}
\sum_{j\ne i}\calP_{\widehat{\TT}_j}
(\widehat{\calT}_j-\calT_j)\right\|_{\mathrm F},\\
\mathrm{VI}
&=
\left\|\calP_{\widehat{\TT}_i}\calA^*\calA
\calP_{\widehat{\TT}_i}^{\perp}
\sum_{j\ne i}\calP_{\widehat{\TT}_j}^{\perp}
(\widehat{\calT}_j-\calT_j)\right\|_{\mathrm F}.
\end{align*}

The first two terms satisfy
\begin{align}
\mathrm I
&\leq
\|\calP_{\widehat{\TT}_i}(I-\alpha_t\calA^*\calA)
\calP_{\widehat{\TT}_i}\|
\|\widehat{\calT}_i-\calT_i\|_{\mathrm F},
\label{eq:rgd-current-I}\\
\mathrm{II}
&=
\|\calP_{\widehat{\TT}_i}^{\perp}\calT_i\|_{\mathrm F}
\leq
3d\frac{\|\widehat{\calT}_i-\calT_i\|_{\mathrm F}^2}{\lambda_i},
\label{eq:rgd-current-II}
\end{align}
where $\calP_{\widehat{\TT}_i}^{\perp}\widehat{\calT}_i=0$ and \eqref{eq:PTip(Tihat - Ti)} were used in \eqref{eq:rgd-current-II}.

Let
$\widehat M_i=\calP_{\widehat{\TT}_i}\calA^*\calA
\calP_{\widehat{\TT}_i}^{\perp}$.
For every tensor $W$, with the Segre atomic norm $\|\cdot\|_{\operatorname{Seg},*}$ of Lemma~\ref{lemma: perturbation bound},
\[
\|\widehat M_iW\|_{\mathrm F}
\leq
\sup_{\substack{V\in\operatorname{Seg}\\\|V\|_{\mathrm F}=1}}
\|\widehat M_iV\|_{\mathrm F}\,
\|W\|_{\operatorname{Seg},*}.
\]
Consequently, since $\widehat M_i=\widehat M_i\calP_{\widehat{\TT}_i}^{\perp}$, \eqref{eq:PTip_atomic} gives
\begin{align}
\mathrm{III}
&\leq
4d
\sup_{\substack{V\in\operatorname{Seg}\\\|V\|_{\mathrm F}=1}}
\|\widehat M_iV\|_{\mathrm F}
\frac{\|\widehat{\calT}_i-\calT_i\|_{\mathrm F}^2}{\lambda_i},
\label{eq:rgd-current-III}\\
\mathrm{VI}
&\leq
4d
\sup_{\substack{V\in\operatorname{Seg}\\\|V\|_{\mathrm F}=1}}
\|\widehat M_iV\|_{\mathrm F}
\sum_{j\ne i}
\frac{\|\widehat{\calT}_j-\calT_j\|_{\mathrm F}^2}{\lambda_j}.
\label{eq:rgd-current-VI}
\end{align}
After division by $\lambda_i$, \eqref{eq:rgd-current-VI} and
$\lambda_j/\lambda_i\leq\kappa$ imply
\begin{equation}
\frac{\mathrm{VI}}{\lambda_i}
\leq
4d(r-1)\kappa
\sup_{\substack{V\in\operatorname{Seg}\\\|V\|_{\mathrm F}=1}}
\|\widehat M_iV\|_{\mathrm F}\,\varepsilon^2.
\label{eq:rgd-current-VI-normalized}
\end{equation}

For the remaining cross-component terms, \eqref{eq:PTip(Tjhat - Tj)} yields
\begin{align}
\frac{\mathrm{IV}}{\lambda_i}
&\leq
4d(r-1)\kappa
\|\calP_{\widehat{\TT}_i}\calA^*\calA
\calP_{\widehat{\TT}_i}\|
\bigl\{\eta^{d-2}+\varepsilon\bigr\}\varepsilon,
\label{eq:rgd-current-IV}\\
\frac{\mathrm V}{\lambda_i}
&\leq
(r-1)\kappa
\max_{j\ne i}
\|\calP_{\widehat{\TT}_i}\calA^*\calA
\calP_{\widehat{\TT}_i}^{\perp}
\calP_{\widehat{\TT}_j}\|\,\varepsilon.
\label{eq:rgd-current-V}
\end{align}
Since $\widehat W_i-\calT_i=(\widehat{\calT}_i-\calT_i)-\alpha_t\calP_{\widehat{\TT}_i}\calA^*\calA\sum_{j=1}^r(\widehat{\calT}_j-\calT_j)+\alpha_t\calP_{\widehat{\TT}_i}\calA^*\calE$, the decomposition above gives $\|\widehat W_i-\calT_i\|_{\mathrm F}\leq\mathrm I+\mathrm{II}+\alpha_t(\mathrm{III}+\mathrm{IV}+\mathrm V+\mathrm{VI})+\alpha_t\|\calP_{\widehat{\TT}_i}\calA^*\calE\|_{\mathrm F}$. Dividing by $\lambda_i$ and inserting \eqref{eq:rgd-current-I}--\eqref{eq:rgd-current-III} and \eqref{eq:rgd-current-VI-normalized}--\eqref{eq:rgd-current-V} give
\begin{align}
\frac{\|\widehat W_i-\calT_i\|_{\mathrm F}}{\lambda_i}
\leq{}&
\bigl\|\calP_{\widehat{\TT}_i}(I-\alpha_t\calA^*\calA)\calP_{\widehat{\TT}_i}\bigr\|\,\varepsilon
+(r-1)\alpha_t\kappa\max_{j\ne i}\bigl\|\calP_{\widehat{\TT}_i}\calA^*\calA\calP_{\widehat{\TT}_i}^{\perp}\calP_{\widehat{\TT}_j}\bigr\|\,\varepsilon
\nonumber\\
&+\Bigl[3d+4d\{1+(r-1)\kappa\}\,\alpha_t\sup_{\substack{V\in\operatorname{Seg}\\\|V\|_{\mathrm F}=1}}\|\widehat M_iV\|_{\mathrm F}\Bigr]\varepsilon^2
\nonumber\\
&+4d(r-1)\alpha_t\kappa\bigl\|\calP_{\widehat{\TT}_i}\calA^*\calA\calP_{\widehat{\TT}_i}\bigr\|\bigl\{\eta^{d-2}+\varepsilon\bigr\}\varepsilon
+\alpha_t\frac{\|\calP_{\widehat{\TT}_i}\calA^*\calE\|_{\mathrm F}}{\lambda_i}.
\label{eq:rgd-current-combined}
\end{align}
Since $r-1\leq r$, $3d\leq4d$ and $1+(r-1)\kappa\leq r\kappa$, bounding each $i$-dependent norm in \eqref{eq:rgd-current-combined} by its maximum over $i\in[r]$ and adding the remainder in \eqref{eq:rgd-current-quadratic-retraction} prove the theorem with $C_1=1$ and $C_2=4d$. By \eqref{eq:rgd-current-thosvd}, $\|\widehat{\calT}_i^+-\calT_i\|_{\mathrm F}\leq(\sqrt d+1)\|\widehat W_i-\calT_i\|_{\mathrm F}$, so the same argument gives the bound without the retraction remainder, with $C_1=\sqrt d+1$ and $C_2=4d(\sqrt d+1)$.
\end{proof}

\begin{proof}[Proof of Theorem~\ref{thm:local_convergence_rgn}]

Fix $i\in[r]$. The condition $\varepsilon\leq1/(4d)$ and Remark~\ref{remark:sign-alignment} permit Lemma~\ref{lemma: perturbation bound} to be applied to $\{\widehat{\calT}_j\}_{j=1}^r$.

The current RGN direction and updated component are
\begin{align*}
\widehat\Delta_i
&=
\left(
\calP_{\widehat{\TT}_i}\calA^*\calA\calP_{\widehat{\TT}_i}
\big|_{\widehat{\TT}_i}
\right)^{-1}
\calP_{\widehat{\TT}_i}\calA^*
\left[
\calA\left\{\sum_{j=1}^r(\widehat{\calT}_j-\calT_j)\right\}-\calE
\right],\\
\widehat{\calT}_i^+
&=\mathcal R_{\widehat{\calT}_i}(-\widehat\Delta_i).
\end{align*}
Define
\begin{align*}
\widehat{\mathsf B}_i
&=
\left(
\calP_{\widehat{\TT}_i}\calA^*\calA\calP_{\widehat{\TT}_i}
\big|_{\widehat{\TT}_i}
\right)^{-1}
\calP_{\widehat{\TT}_i}\calA^*\calA,\\
\widehat{\mathsf H}_i
&=
\left(
\calP_{\widehat{\TT}_i}\calA^*\calA\calP_{\widehat{\TT}_i}
\big|_{\widehat{\TT}_i}
\right)^{-1}
\calP_{\widehat{\TT}_i}\calA^*.
\end{align*}
The inverse on $\widehat{\TT}_i$ gives
\begin{equation}
\widehat{\mathsf B}_i\calP_{\widehat{\TT}_i}
=\calP_{\widehat{\TT}_i},
\qquad
\widehat{\mathsf B}_i\calP_{\widehat{\TT}_i}^{\perp}
=\widehat{\mathsf C}_i,
\qquad
\widehat{\mathsf C}_i\calP_{\widehat{\TT}_i}^{\perp}
=\widehat{\mathsf C}_i.
\label{eq:rgn-current-operator-identities}
\end{equation}
Since $\calP_{\widehat{\TT}_i}^{\perp}\widehat{\calT}_i=0$,
\begin{align}
\widehat{\calT}_i-\widehat\Delta_i-\calT_i
={}&
(\widehat{\calT}_i-\calT_i)
-\widehat{\mathsf B}_i
\sum_{j=1}^r(\widehat{\calT}_j-\calT_j)
+\widehat{\mathsf H}_i\calE
\nonumber\\
={}&
\calP_{\widehat{\TT}_i}
\sum_{j\ne i}(\calT_j-\widehat{\calT}_j)
+\widehat{\mathsf C}_i\calT_i
\nonumber\\
&+\widehat{\mathsf C}_i
\sum_{j\ne i}\calP_{\widehat{\TT}_j}
(\calT_j-\widehat{\calT}_j)
+\widehat{\mathsf C}_i
\sum_{j\ne i}\calP_{\widehat{\TT}_j}^{\perp}
(\calT_j-\widehat{\calT}_j)
\nonumber\\
&-\calP_{\widehat{\TT}_i}^{\perp}\calT_i
+\widehat{\mathsf H}_i\calE.
\label{eq:rgn-current-error-identity}
\end{align}

Let $\widehat W_i=\widehat{\calT}_i-\widehat\Delta_i$. Since $\widehat\xi_i=-\widehat\Delta_i\in\widehat{\TT}_i$, \eqref{eq:local-thosvd-quadratic} gives
\begin{equation}
\|\widehat{\calT}_i^+-\calT_i\|_{\mathrm F}
\leq
\|\widehat W_i-\calT_i\|_{\mathrm F}
+\frac{\|\widehat\xi_i\|_{\mathrm F}^2}
{\|\widehat{\calT}_i\|_{\mathrm F}}.
\label{eq:rgn-current-quadratic-retraction}
\end{equation}
Since $\calT_i\in\operatorname{Seg}$, \eqref{eq:local-thosvd-quasi-best} also yields
\begin{align}
\|\widehat{\calT}_i^+-\calT_i\|_{\mathrm F}
&\leq
\|\widehat{\calT}_i^+-\widehat W_i\|_{\mathrm F}
+\|\widehat W_i-\calT_i\|_{\mathrm F}
\nonumber\\
&\leq
(\sqrt d+1)\|\widehat W_i-\calT_i\|_{\mathrm F}.
\label{eq:rgn-current-thosvd}
\end{align}
Let $\mathrm I,\ldots,\mathrm{VI}$ denote the Frobenius norms of the six terms on the right-hand side of \eqref{eq:rgn-current-error-identity}, in their displayed order.

By \eqref{eq:PTip(Tjhat - Tj)},
\begin{align}
\frac{\mathrm I}{\lambda_i}
&\leq
4d(r-1)\kappa
\bigl\{\eta^{d-2}+\varepsilon\bigr\}\varepsilon.
\label{eq:rgn-current-I}
\end{align}
For every tensor $W$, with the Segre atomic norm of Lemma~\ref{lemma: perturbation bound},
$\|\widehat{\mathsf C}_iW\|_{\mathrm F}\leq\|\widehat{\mathsf C}_i\|_{\operatorname{Seg}\to\mathrm F}\|W\|_{\operatorname{Seg},*}$.
For $\mathrm{II}$, \eqref{eq:rgn-current-operator-identities},
$\calP_{\widehat{\TT}_i}^{\perp}\widehat{\calT}_i=0$, and \eqref{eq:PTip_atomic} give
\begin{align}
\mathrm{II}
&=
\left\|\widehat{\mathsf C}_i
\calP_{\widehat{\TT}_i}^{\perp}(\calT_i-\widehat{\calT}_i)\right\|_{\mathrm F}
\nonumber\\
&\leq
4d\|\widehat{\mathsf C}_i\|_{\operatorname{Seg}\to\mathrm F}
\frac{\|\widehat{\calT}_i-\calT_i\|_{\mathrm F}^2}{\lambda_i}.
\label{eq:rgn-current-II}
\end{align}
Since $\|\widehat{\calT}_j-\calT_j\|_{\mathrm F}\leq\varepsilon\lambda_j\leq\kappa\varepsilon\lambda_i$, we have
\begin{align}
\frac{\mathrm{III}}{\lambda_i}
&\leq
(r-1)\kappa
\max_{j\ne i}
\|\widehat{\mathsf C}_i\calP_{\widehat{\TT}_j}\|\,\varepsilon.
\label{eq:rgn-current-III}
\end{align}
Applying \eqref{eq:PTip_atomic} to each $\calP_{\widehat{\TT}_j}^{\perp}(\calT_j-\widehat{\calT}_j)$ in $\mathrm{IV}$ gives
\begin{align}
\frac{\mathrm{IV}}{\lambda_i}
&\leq
4d(r-1)\kappa
\|\widehat{\mathsf C}_i\|_{\operatorname{Seg}\to\mathrm F}\,\varepsilon^2.
\label{eq:rgn-current-IV}
\end{align}
Finally, \eqref{eq:PTip(Tihat - Ti)} and the definition of $\widehat{\mathsf H}_i$ imply
\begin{align}
\frac{\mathrm V}{\lambda_i}
&\leq3d\varepsilon^2,
\label{eq:rgn-current-V}\\
\frac{\mathrm{VI}}{\lambda_i}
&=
\frac{1}{\lambda_i}
\left\|
\left(
\calP_{\widehat{\TT}_i}\calA^*\calA\calP_{\widehat{\TT}_i}
\big|_{\widehat{\TT}_i}
\right)^{-1}
\calP_{\widehat{\TT}_i}\calA^*\calE
\right\|_{\mathrm F}.
\label{eq:rgn-current-VI}
\end{align}
Since $\widehat W_i-\calT_i$ is the left-hand side of \eqref{eq:rgn-current-error-identity}, $\|\widehat W_i-\calT_i\|_{\mathrm F}\leq\mathrm I+\cdots+\mathrm{VI}$, and \eqref{eq:rgn-current-I}--\eqref{eq:rgn-current-VI} give
\begin{align}
\frac{\|\widehat W_i-\calT_i\|_{\mathrm F}}{\lambda_i}
\leq{}&
(r-1)\kappa\max_{j\ne i}\|\widehat{\mathsf C}_i\calP_{\widehat{\TT}_j}\|\,\varepsilon
+\frac{\|\widehat{\mathsf H}_i\calE\|_{\mathrm F}}{\lambda_i}
+4d(r-1)\kappa\bigl\{\eta^{d-2}+\varepsilon\bigr\}\varepsilon
\nonumber\\
&+\Bigl[3d+4d\{1+(r-1)\kappa\}\|\widehat{\mathsf C}_i\|_{\operatorname{Seg}\to\mathrm F}\Bigr]\varepsilon^2.
\label{eq:rgn-current-combined}
\end{align}
Since $r-1\leq r$, $3d\leq4d$ and $1+(r-1)\kappa\leq1+r\kappa$, bounding each $i$-dependent norm in \eqref{eq:rgn-current-combined} by its maximum over $i\in[r]$ and adding the remainder in \eqref{eq:rgn-current-quadratic-retraction} prove the theorem with $C_1=1$ and $C_2=4d$. By \eqref{eq:rgn-current-thosvd}, the same argument gives the bound without the retraction remainder, with $C_1=\sqrt d+1$ and $C_2=4d(\sqrt d+1)$.

\end{proof}

\section{Proof of Corollaries} \label{sec:proof_cor}

This appendix proves variable-step versions of the RGD corollaries and the RGN corollaries in Section~\ref{subsec:applications}. Throughout, $c$ and $C$ denote universal constants whose values may change from line to line, and $c_0,C_0,C_3,C_4$ denote universal constants fixed within each proof. For a current collection $\{\widehat{\calT}_j\}_{j=1}^r\subset\operatorname{Seg}$, $\widehat{\TT}_j$ is the tangent space at $\widehat{\calT}_j$, $D_j=\widehat{\calT}_j-\calT_j$, $\widehat\xi_i\in\widehat{\TT}_i$ is the tangent increment of the method under consideration, $\widehat W_i=\widehat{\calT}_i+\widehat\xi_i$, and $\widehat{\calT}_i^+=\mathcal R_{\widehat{\calT}_i}(\widehat\xi_i)$ is the updated component.

\begin{lemma}[Composition of sequential component updates]
\label{lemma:sequential-composition}
Let $\Phi_t:[0,R]\to[0,\infty)$ be nondecreasing. Suppose that, for every current collection $\{\widehat{\calT}_j\}_{j=1}^r$ satisfying
\[
\max_{j\in[r]}
\frac{\|\widehat{\calT}_j-\calT_j\|_{\mathrm F}}{\lambda_j}
\leq R,
\]
the update of component $i$ satisfies
\[
\frac{\|\widehat{\calT}_i^+-\calT_i\|_{\mathrm F}}{\lambda_i}
\leq
\Phi_t\!\left(
\max_{j\in[r]}
\frac{\|\widehat{\calT}_j-\calT_j\|_{\mathrm F}}{\lambda_j}
\right),
\qquad i\in[r].
\]
Set $\varepsilon^{(t)}=\max_{j\in[r]}\|\calT_j^{(t)}-\calT_j\|_{\mathrm F}/\lambda_j$. If $\varepsilon^{(t)}\leq\nu\leq R$ and $\Phi_t(\nu)\leq\nu$, then every intermediate collection generated by the sequential updates $i=1,\ldots,r$ has relative error at most $\nu$, and
\begin{equation}
\varepsilon^{(t+1)}\leq\Phi_t(\nu).
\label{eq:sequential-composition}
\end{equation}
\end{lemma}

\begin{proof}
For $s=0,\ldots,r$, let
\[
\calT_j^{(t,s)}=
\begin{cases}
\calT_j^{(t+1)},&j\leq s,\\
\calT_j^{(t)},&j>s,
\end{cases}
\qquad
\varepsilon^{(t,s)}=
\max_{j\in[r]}
\frac{\|\calT_j^{(t,s)}-\calT_j\|_{\mathrm F}}{\lambda_j},
\]
so that $\varepsilon^{(t,0)}=\varepsilon^{(t)}$ and $\varepsilon^{(t,r)}=\varepsilon^{(t+1)}$. The $s$-th update changes component $s$ only. If $\varepsilon^{(t,s-1)}\leq\nu$, the hypothesis and the monotonicity of $\Phi_t$ give
\[
\frac{\|\calT_s^{(t+1)}-\calT_s\|_{\mathrm F}}{\lambda_s}
\leq\Phi_t(\varepsilon^{(t,s-1)})
\leq\Phi_t(\nu)
\leq\nu,
\qquad
\varepsilon^{(t,s)}\leq\nu.
\]
Induction over $s=1,\ldots,r$ yields $\max_{0\leq s\leq r}\varepsilon^{(t,s)}\leq\nu$ and $\|\calT_s^{(t+1)}-\calT_s\|_{\mathrm F}/\lambda_s\leq\Phi_t(\nu)$ for every $s\in[r]$, which is \eqref{eq:sequential-composition}.
\end{proof}

\begin{lemma}[Local Gaussian tangent projection]\label{lemma:local-gaussian-tangent}
Let $\calT_i=\lambda_i u_{1,i}\otimes\cdots\otimes u_{d,i}$, $i\in[r]$, be the components of the fixed rank-$r$ tensor in \eqref{cp-tensor}, and let $\operatorname{vec}(\calE)\sim\mathcal N(0,\Sigma_{p^*})$ with $\Sigma_{p^*}\preceq\overline\sigma^{2}I_{p^*}$. There exist universal constants $c,C>0$ such that, for every $0<\nu\leq1/(2\sqrt d)$, with probability at least $1-2\exp(-c\,\mathrm{df})$,
\begin{equation}
\max_{i\in[r]}
\sup_{\substack{\widetilde\calT\in\operatorname{Seg}\\
\|\widetilde\calT-\calT_i\|_{\mathrm F}\leq\nu\lambda_i}}
\|\calP_{\TT_{\widetilde\calT}}\calE\|_{\mathrm F}
\leq C\overline\sigma\sqrt{\mathrm{df}}.
\label{eq:local-gaussian-tangent}
\end{equation}
\end{lemma}

\begin{proof}
Fix $i\in[r]$ and $\widetilde\calT\in\operatorname{Seg}$ with $\|\widetilde\calT-\calT_i\|_{\mathrm F}\leq\nu\lambda_i$. The hypothesis $\nu\leq1/(2\sqrt d)$ gives $\nu\leq1/2$ and $\nu\sqrt d\leq1$. Write $\calT_i=\lambda_iU_i$ and $\widetilde\calT=\widetilde\lambda\widetilde U$ with $U_i=u_{1,i}\otimes\cdots\otimes u_{d,i}$, $\widetilde U=\widetilde u_1\otimes\cdots\otimes\widetilde u_d$, unit factor vectors, and $\widetilde\lambda>0$. By polarization,
\[
2\langle\widetilde\calT,\calT_i\rangle
=\|\widetilde\calT\|_{\mathrm F}^2+\lambda_i^2-\|\widetilde\calT-\calT_i\|_{\mathrm F}^2
\geq2(1-\nu)\lambda_i^2>0,
\]
so the factor signs can be chosen with $\langle u_{\ell,i},\widetilde u_\ell\rangle=\cos\theta_\ell\geq0$ for angles $\theta_\ell\in[0,\pi/2]$, $\ell\in[d]$. Then
\begin{align}
\nu^2
&\geq
\inf_{a\in\mathbb R}\|a\widetilde U-U_i\|_{\mathrm F}^2
=1-\langle\widetilde U,U_i\rangle^2
=1-\prod_{\ell=1}^d\cos^2\theta_\ell,
\nonumber\\
\sum_{\ell=1}^d
\|\calP_{u_{\ell,i}}^\perp\widetilde u_\ell\|_2^2
=\sum_{\ell=1}^d\sin^2\theta_\ell
&\leq-\log\!\left(\prod_{\ell=1}^d\cos^2\theta_\ell\right)
\leq-\log(1-\nu^2)
\leq2\nu^2.
\label{eq:local-angle-sum}
\end{align}
With $K_\ell=\widetilde u_\ell u_{\ell,i}^\top-u_{\ell,i}\widetilde u_\ell^\top$, let
\[
Q_\ell=I+K_\ell+
\frac{K_\ell^2}{1+u_{\ell,i}^\top\widetilde u_\ell},
\qquad
Q=Q_1\otimes\cdots\otimes Q_d.
\]
Each $Q_\ell$ is the orthogonal rotation of $\operatorname{span}\{u_{\ell,i},\widetilde u_\ell\}$ that maps $u_{\ell,i}$ to $\widetilde u_\ell$, extended by the identity on the orthogonal complement, and the denominator is at least one. Hence
\begin{equation}
\TT_{\widetilde\calT}=Q\TT_{\calT_i}.
\label{eq:rotated-tangent-space}
\end{equation}

\smallskip\noindent\emph{Step 1 (Entropy bound).}
Since $\|\widetilde u_\ell\|_2=1$ and $\cos\theta_\ell\geq0$,
\[
\widetilde u_\ell
=
\bigl(1-\|\calP_{u_{\ell,i}}^\perp\widetilde u_\ell\|_2^2\bigr)^{1/2}
\,u_{\ell,i}
+\calP_{u_{\ell,i}}^\perp\widetilde u_\ell.
\]
For two factor tuples $(\widetilde u_\ell)_{\ell=1}^d$ and $(\widetilde u'_\ell)_{\ell=1}^d$ of this form satisfying \eqref{eq:local-angle-sum},
\[
\sum_{\ell=1}^d
\|\widetilde u_\ell-\widetilde u'_\ell\|_2^2
\leq
\frac{1}{1-2\nu^2}
\sum_{\ell=1}^d
\|\calP_{u_{\ell,i}}^\perp
(\widetilde u_\ell-\widetilde u'_\ell)\|_2^2
\leq
2\sum_{\ell=1}^d
\|\calP_{u_{\ell,i}}^\perp
(\widetilde u_\ell-\widetilde u'_\ell)\|_2^2.
\]
The tuples $(\calP_{u_{\ell,i}}^\perp\widetilde u_\ell)_{\ell=1}^d$ lie in the Euclidean ball of radius $\sqrt2\nu$ of the $(\mathrm{df}-1)$-dimensional space $\prod_{\ell=1}^d u_{\ell,i}^{\perp}$. Equip the factor tuples of the tensors $\widetilde\calT\in\operatorname{Seg}$ satisfying $\|\widetilde\calT-\calT_i\|_{\mathrm F}\leq\nu\lambda_i$ with the distance $(\sum_{\ell=1}^d\|\widetilde u_\ell-\widetilde u'_\ell\|_2^2)^{1/2}$. By the preceding display, an $\epsilon$-separated set of factor tuples has $\epsilon/\sqrt2$-separated projections, so it has at most $(1+4\nu/\epsilon)^{\mathrm{df}-1}$ elements by the volume comparison in the proof of \citet[Proposition~4.2.12]{vershynin2018high}, and a maximal $\epsilon$-separated set is an $\epsilon$-net \citep[Lemma~4.2.6]{vershynin2018high}.

Let $Q'_\ell$ and $Q'=\bigotimes_{\ell=1}^dQ'_\ell$ be defined from $(\widetilde u'_\ell)_{\ell=1}^d$ as $Q_\ell$ and $Q$ are defined from $(\widetilde u_\ell)_{\ell=1}^d$. Since the denominators are at least one and $\|K_\ell\|\leq1$, and since the factors $Q_\ell,Q'_\ell$ are orthogonal,
\begin{align*}
\|Q_\ell-Q'_\ell\|
&\leq
3\|(\widetilde u_\ell-\widetilde u'_\ell)u_{\ell,i}^{\top}
-u_{\ell,i}(\widetilde u_\ell-\widetilde u'_\ell)^{\top}\|
+\|\widetilde u_\ell-\widetilde u'_\ell\|_2
\leq7\|\widetilde u_\ell-\widetilde u'_\ell\|_2,\\
\|Q-Q'\|
&\leq\sum_{\ell=1}^d\|Q_\ell-Q'_\ell\|
\leq7\sqrt d
\left(\sum_{\ell=1}^d
\|\widetilde u_\ell-\widetilde u'_\ell\|_2^2\right)^{1/2}.
\end{align*}
Consequently, for $Z,Z'\in\TT_{\calT_i}$ with $\|Z\|_{\mathrm F}=\|Z'\|_{\mathrm F}=1$,
\[
\|QZ-Q'Z'\|_{\mathrm F}
\leq
\|Z-Z'\|_{\mathrm F}
+7\sqrt d
\left(\sum_{\ell=1}^d
\|\widetilde u_\ell-\widetilde u'_\ell\|_2^2\right)^{1/2}.
\]
By \eqref{eq:rotated-tangent-space}, every unit vector of $\TT_{\widetilde\calT}$ equals $QZ$ for a unit vector $Z\in\TT_{\calT_i}$. Combining an $\epsilon/2$-net of the unit sphere of the $\mathrm{df}$-dimensional space $\TT_{\calT_i}$, of cardinality at most $(6/\epsilon)^{\mathrm{df}}$ \citep[Corollary~4.2.13]{vershynin2018high}, with an $\epsilon/(14\sqrt d)$-net of the factor tuples gives, for $0<\epsilon<1$ and $\nu\sqrt d\leq1$,
\begin{align}
\log N\!\left(\epsilon,
\bigcup_{\substack{\widetilde\calT\in\operatorname{Seg}\\
\|\widetilde\calT-\calT_i\|_{\mathrm F}\leq\nu\lambda_i}}
\{Z\in\TT_{\widetilde\calT}:\|Z\|_{\mathrm F}=1\},
\|\cdot\|_{\mathrm F}\right)
&\leq
\mathrm{df}\log\frac6\epsilon
+(\mathrm{df}-1)
\log\left(1+\frac{56\nu\sqrt d}{\epsilon}\right)
\nonumber\\
&\leq C\,\mathrm{df}\log\frac C\epsilon.
\label{eq:local-tangent-entropy}
\end{align}
Since $\log r\leq\mathrm{df}-1$ (Appendix~\ref{appendix:notation}), the union of these sets over $i\in[r]$ satisfies \eqref{eq:local-tangent-entropy} with a larger $C$.

\smallskip\noindent\emph{Step 2 (Chaining and concentration).}
Let $Z$ range over the union over $i\in[r]$ of the sets in \eqref{eq:local-tangent-entropy}. The covariance assumption gives
\[
\mathbb E\langle Z-Z',\calE\rangle^2
=\operatorname{vec}(Z-Z')^\top\Sigma_{p^*}\operatorname{vec}(Z-Z')
\leq\overline\sigma^{2}\|Z-Z'\|_{\mathrm F}^2.
\]
The index set is symmetric with Frobenius diameter at most two, so $\sup_Z|\langle Z,\calE\rangle|=\sup_Z\langle Z,\calE\rangle$, and the increments $\langle Z-Z',\calE\rangle$ are centered Gaussian. Dudley's inequality \citep[Theorem~8.1.3]{vershynin2018high} and \eqref{eq:local-tangent-entropy} give
\[
\mathbb E\sup_Z|\langle Z,\calE\rangle|
\leq
C\overline\sigma
\int_0^2\sqrt{\mathrm{df}\log(1+C/\epsilon)}\,d\epsilon
\leq C\overline\sigma\sqrt{\mathrm{df}}.
\]
Write $\operatorname{vec}(\calE)=\Sigma_{p^*}^{1/2}g$ with $g\sim\mathcal N(0,I_{p^*})$. The map $g\mapsto\sup_Z|\langle\operatorname{vec}(Z),\Sigma_{p^*}^{1/2}g\rangle|$ is $\overline\sigma$-Lipschitz because $\|\Sigma_{p^*}^{1/2}\|\leq\overline\sigma$, so Gaussian concentration \citep[Theorem~5.2.2]{vershynin2018high} gives
\[
\mathbb P\left\{
\sup_Z|\langle Z,\calE\rangle|>C\overline\sigma\sqrt{\mathrm{df}}
\right\}
\leq2\exp(-c\,\mathrm{df}).
\]
Finally, $\|\calP_{\TT_{\widetilde\calT}}\calE\|_{\mathrm F}
=\sup\{|\langle Z,\calE\rangle|:Z\in\TT_{\widetilde\calT},\ \|Z\|_{\mathrm F}=1\}$, which proves \eqref{eq:local-gaussian-tangent}.
\end{proof}

\begin{corollary}[Convergence of RGD for CP tensor decomposition]\label{corollary:rate_decomposition_appendix}
Let $\varepsilon^{(t)} = \max_{i \in [r]} \|\calT_i^{(t)}-\calT_i\|_{\mathrm{F}}/\lambda_i$. There exist universal constants $c,C>0$ such that the following holds. Suppose, for every $t\geq0$,
\[
0<\alpha_t\leq1,
\qquad
|1-\alpha_t|\leq\frac12,
\qquad
\alpha_t\eta^{d-2}\leq
\frac{c}{d\kappa r},
\]
and
\[
\varepsilon^{(0)}\leq
\frac{c}{d(1+\kappa r)},
\qquad
\frac{\overline\sigma\sqrt{\mathrm{df}}}{\lambda_r}
\leq
\frac{c}{d(1+\kappa r)}.
\]
Then, with probability at least $1-2\exp(-c\,\mathrm{df})$, simultaneously for all $t\geq0$,
\[
\varepsilon^{(t)}
\leq
\left\{\prod_{s=0}^{t-1}\left(1-\frac{3\alpha_s}{4}\right)\right\}\varepsilon^{(0)}
+C\frac{\overline\sigma\sqrt{\mathrm{df}}}{\lambda_r}.
\]
\end{corollary}

\begin{proof}
Let
\[
R=\frac{c_0}{d(1+\kappa r)},
\qquad
b=C_0
\frac{\overline\sigma\sqrt{\mathrm{df}}}{\lambda_r},
\]
where $C_0$ is eight times the constant $C$ in \eqref{eq:local-gaussian-tangent} and $c_0$ is so small that Lemma~\ref{lemma:local-gaussian-tangent} applies with $\nu=R$ and
\[
R\leq\min\left\{\frac1{4d},\frac1{24}\right\},
\qquad
2R<\frac1{2(\sqrt d+1)}.
\]
The constant $c$ in the corollary is then chosen so small that $\varepsilon^{(0)}\leq R$, $b\leq R/4$, and the inequalities below hold. All statements are on the event of Lemma~\ref{lemma:local-gaussian-tangent}.

\smallskip\noindent\emph{Step 1 (One component update).}
Let $\{\widehat{\calT}_j\}_{j=1}^r\subset\operatorname{Seg}$ satisfy $\varepsilon=\max_{j}\|D_j\|_{\mathrm F}/\lambda_j\leq R$. Since $\calA^*\calA=I$ and $\calY-\sum_{j=1}^r\calT_j=\calE$,
\begin{align*}
\widehat\xi_i&=-\alpha_t\calP_{\widehat{\TT}_i}
\left(\sum_{j=1}^rD_j-\calE\right),\\
\widehat W_i-\calT_i
&=(1-\alpha_t)D_i+\alpha_t\calP_{\widehat{\TT}_i}^{\perp}D_i
-\alpha_t\sum_{j\ne i}\calP_{\widehat{\TT}_i}D_j
+\alpha_t\calP_{\widehat{\TT}_i}\calE.
\end{align*}
By \eqref{eq:PTip(Tihat - Ti)}, \eqref{eq:PTip(Tjhat - Tj)} with $\|D_j\|_{\mathrm F}\leq \varepsilon\lambda_j\leq\kappa \varepsilon\lambda_i$, \eqref{eq:local-gaussian-tangent} with $\nu=R$, and $0<\alpha_t\leq1$,
\begin{align*}
\frac{\|\widehat W_i-\calT_i\|_{\mathrm F}}{\lambda_i}
&\leq(1-\alpha_t)\varepsilon+\alpha_t\Bigl[3d\varepsilon^2+(r-1)\kappa\bigl\{(d+1)\eta^{d-2}+(2d+4)\varepsilon\bigr\}\varepsilon+\frac b8\Bigr],\\
\frac{\|\widehat\xi_i\|_{\mathrm F}}{\lambda_i}
&\leq\alpha_t\Bigl[\varepsilon+(r-1)\kappa\bigl\{(d+1)\eta^{d-2}+(2d+4)\varepsilon\bigr\}\varepsilon+\frac b8\Bigr].
\end{align*}
Since $2d+4\leq4d$ and $3d+(2d+4)(r-1)\kappa\leq4d(1+r\kappa)$,
\begin{equation}
\begin{aligned}
\frac{\|\widehat W_i-\calT_i\|_{\mathrm F}}{\lambda_i}
&\leq\varphi_t(\varepsilon)
:=(1-\alpha_t)\varepsilon
+\alpha_t\left\{(d+1)r\kappa\eta^{d-2}\varepsilon+4d(1+r\kappa)\varepsilon^2+\frac b8\right\},\\
\frac{\|\widehat\xi_i\|_{\mathrm F}}{\lambda_i}
&\leq\alpha_t\left\{\varepsilon+(d+1)r\kappa\eta^{d-2}\varepsilon+4dr\kappa \varepsilon^2+\frac b8\right\}.
\end{aligned}
\label{eq:decomposition-rgd-pre-retraction}
\end{equation}
The step-size conditions give $\alpha_t\geq1/2$ and $\eta^{d-2}\leq2c/(d\kappa r)$. Since $d+1\leq4d/3$,
\[
(d+1)r\kappa\eta^{d-2}+4d(1+r\kappa)R
\leq\frac83c+4c_0\leq\frac18,
\]
and therefore, for $0\leq \varepsilon\leq R$,
\begin{equation}
\varphi_t(\varepsilon)\leq\left(1-\frac{7\alpha_t}8\right)\varepsilon+\frac{\alpha_tb}8,
\qquad
\frac{\|\widehat\xi_i\|_{\mathrm F}}{\lambda_i}
\leq\frac{\alpha_t(9\varepsilon+b)}8.
\label{eq:decomposition-rgd-linearized}
\end{equation}
Since $\|\widehat{\calT}_i\|_{\mathrm F}\geq(1-\varepsilon)\lambda_i$, $R\leq1/24$, and $b\leq R/4$,
\begin{equation}
\begin{aligned}
\frac{\|\widehat\xi_i\|_{\mathrm F}}
{\|\widehat{\calT}_i\|_{\mathrm F}}
&\leq
\frac{\alpha_t(9\varepsilon+b)}{8(1-\varepsilon)}
\leq
\frac{9R+b}{8(1-R)}
\leq2R
<\frac{1}{2(\sqrt d+1)},\\
\frac{\|\widehat\xi_i\|_{\mathrm F}^2}
{\lambda_i\|\widehat{\calT}_i\|_{\mathrm F}}
&\leq\frac{\alpha_t^2(9\varepsilon+b)^2}{64(1-\varepsilon)}
\leq\alpha_t(3\varepsilon^2+b^2)
\leq\frac{\alpha_t \varepsilon}8+\frac{7\alpha_t b}8.
\end{aligned}
\label{eq:decomposition-rgd-retraction-domain}
\end{equation}
Hence \eqref{eq:local-thosvd-quadratic} applies to $\widehat{\calT}_i^+=\mathcal R_{\widehat{\calT}_i}(\widehat\xi_i)$, and \eqref{eq:decomposition-rgd-linearized}--\eqref{eq:decomposition-rgd-retraction-domain} give
\begin{equation}
\frac{\|\widehat{\calT}_i^+-\calT_i\|_{\mathrm F}}{\lambda_i}
\leq\varphi_t(\varepsilon)
+\frac{\|\widehat\xi_i\|_{\mathrm F}^2}
{\lambda_i\|\widehat{\calT}_i\|_{\mathrm F}}
\leq\left(1-\frac{3\alpha_t}4\right)\varepsilon+\alpha_tb,
\qquad 0\leq \varepsilon\leq R.
\label{eq:decomposition-rgd-majorant}
\end{equation}

\smallskip\noindent\emph{Step 2 (Composition).}
Suppose $\varepsilon^{(t)}\leq R$ and set $\nu_t=\max\{\varepsilon^{(t)},2b\}\leq R$. The right-hand side of \eqref{eq:decomposition-rgd-majorant} is nondecreasing in $\varepsilon$ and satisfies $(1-3\alpha_t/4)\nu_t+\alpha_tb\leq\nu_t$, so Lemma~\ref{lemma:sequential-composition} gives $\varepsilon^{(t+1)}\leq\nu_t\leq R$ and
\[
\varepsilon^{(t+1)}
\leq
\left(1-\frac{3\alpha_t}4\right)\nu_t+\alpha_tb
\leq
2b+\left(1-\frac{3\alpha_t}4\right)\{\varepsilon^{(t)}-2b\}_+.
\]
Since $\varepsilon^{(0)}\leq R$, induction over $t$ yields
\[
\varepsilon^{(t)}
\leq
2b+\left\{\prod_{s=0}^{t-1}\left(1-\frac{3\alpha_s}4\right)\right\}
\{\varepsilon^{(0)}-2b\}_+
\leq
\left\{\prod_{s=0}^{t-1}\left(1-\frac{3\alpha_s}4\right)\right\}\varepsilon^{(0)}+2b,
\qquad t\geq0.
\]
This proves the corollary with $C=2C_0$. Under the conditions of Corollary~\ref{corollary:rate_RGD_decomposition}, $c\leq1/3$ gives $\alpha_t\geq2/3$ and $1-3\alpha_t/4\leq1/2$, which proves that corollary.

\smallskip\noindent\emph{Step 3 (RGN).}
For $\calA=\mathrm{Id}$, the RGN direction is $\widehat\Delta_i=\calP_{\widehat{\TT}_i}(\sum_{j=1}^rD_j-\calE)$, so the RGN update is the RGD update with $\alpha_t=1$. Taking $\alpha_t=1$ in \eqref{eq:decomposition-rgd-pre-retraction} and \eqref{eq:decomposition-rgd-retraction-domain} and using $b/8+b^2\leq b$ gives, for $0\leq \varepsilon\leq R$,
\begin{equation}
\frac{\|\widehat{\calT}_i^+-\calT_i\|_{\mathrm F}}{\lambda_i}
\leq\Psi(\varepsilon)
:=\{4d(1+r\kappa)+5\}\varepsilon^2
+(d+1)r\kappa\eta^{d-2}\varepsilon+b
\leq\frac \varepsilon2+b.
\label{eq:decomposition-rgn-recursion}
\end{equation}
Step~2 with $\Psi$ in place of the right-hand side of \eqref{eq:decomposition-rgd-majorant} proves the linear assertion of Corollary~\ref{cor:RGN1}.

For the two-phase assertion, let $A=C_3d(1+\kappa r)$ with a sufficiently large universal constant $C_3$, and fix $C_4\geq4$. By \eqref{eq:decomposition-rgn-recursion},
\begin{equation}
\Psi(\varepsilon)\leq \frac{A}{2}\{\varepsilon^2+\eta^{d-2}\varepsilon\}+b,
\qquad 0\leq \varepsilon\leq R.
\label{eq:decomposition-rgn-majorant}
\end{equation}
After fixing $C_3$ and $C_4$, decrease $c_0$ and then $c$, if necessary, so that
\[
AR\leq\frac12,
\qquad
A(C_4+1)\eta^{d-2}\leq1,
\qquad
b\leq\frac{A\eta^{2(d-2)}}8,
\]
where the last inequality follows from the additional noise condition in Corollary~\ref{cor:RGN1}. If $C_4\eta^{d-2}<\varepsilon\leq R$, then $\eta^{d-2}\varepsilon\leq \varepsilon^2/C_4$ and $b\leq A\varepsilon^2/(8C_4^2)$. If $0\leq \varepsilon\leq\min\{R,C_4\eta^{d-2}\}$, then $A(\varepsilon+\eta^{d-2})\leq1$. Hence \eqref{eq:decomposition-rgn-majorant} gives
\begin{align}
\Psi(\varepsilon)&\leq A\varepsilon^2\leq \varepsilon,
&&C_4\eta^{d-2}<\varepsilon\leq R,
\label{eq:decomposition-rgn-quadratic-majorant}\\
\Psi(\varepsilon)&\leq\frac{\varepsilon}{2}+b,
&&0\leq \varepsilon\leq\min\{R,C_4\eta^{d-2}\},
\label{eq:decomposition-rgn-linear-majorant}
\end{align}
and $2b\leq A\eta^{2(d-2)}/4\leq C_4\eta^{d-2}$.

For $t<t^*$, $\varepsilon^{(t)}>C_4\eta^{d-2}$. If also $\varepsilon^{(t)}\leq R$, Lemma~\ref{lemma:sequential-composition} with $\nu=\varepsilon^{(t)}$ and \eqref{eq:decomposition-rgn-quadratic-majorant} give $\varepsilon^{(t+1)}\leq A\{\varepsilon^{(t)}\}^2\leq\varepsilon^{(t)}\leq R$. Induction from $\varepsilon^{(0)}\leq R$ yields
\[
\varepsilon^{(t)}\leq\frac{\{A\varepsilon^{(0)}\}^{2^t}}{A},
\qquad 0\leq t<t^*,
\]
which is the quadratic-phase bound in Corollary~\ref{cor:RGN1}. If $t^*<\infty$, then $\varepsilon^{(t^*)}\leq\min\{R,C_4\eta^{d-2}\}$. For $t\geq t^*$ with $\varepsilon^{(t)}\leq\min\{R,C_4\eta^{d-2}\}$, the number $\nu_t=\max\{\varepsilon^{(t)},2b\}$ satisfies $\nu_t\leq\min\{R,C_4\eta^{d-2}\}$ and, by \eqref{eq:decomposition-rgn-linear-majorant}, $\Psi(\nu_t)\leq\nu_t/2+b\leq\nu_t$. Lemma~\ref{lemma:sequential-composition} gives $\varepsilon^{(t+1)}\leq\nu_t\leq\min\{R,C_4\eta^{d-2}\}$ and $\varepsilon^{(t+1)}\leq2b+\{\varepsilon^{(t)}-2b\}_+/2$, and induction yields
\[
\varepsilon^{(t)}
\leq
2^{-(t-t^*)}\varepsilon^{(t^*)}+2b,
\qquad t\geq t^*.
\]
Enlarging the universal constant multiplying the statistical error proves Corollary~\ref{cor:RGN1}.
\end{proof}

\begin{corollary}[Convergence of RGD for CP tensor-on-tensor regression]\label{corollary:rate_regression_appendix}
Let $\varepsilon^{(t)} = \max_{i \in [r]} \|\calT_i^{(t)}-\calT_i\|_{\mathrm{F}}/\lambda_i$. There exist universal constants $c,C>0$ such that the following holds. Suppose $n\geq C\,\mathrm{df}$ and
\[
\gamma\leq\frac{c}{1+\kappa r},
\qquad
\varepsilon^{(0)}\leq
\frac{c}{d(1+\kappa r)},
\]
and, for every $t\geq0$,
\[
0<\alpha_t\sigma_X^2\leq1,
\qquad
|1-\alpha_t\sigma_X^2|\leq\frac12,
\qquad
\alpha_t\sigma_X^2\eta^{d-2}
\leq
\frac{c}{d\kappa r}.
\]
Assume also that
\[
\frac{\overline\sigma_\xi}{\sigma_X\lambda_r}
\sqrt{\frac{\mathrm{df}}{n}}
\leq
\frac{c}{d(1+\kappa r)}.
\]
Then, with probability at least
$1-C\exp\{-c(\mathrm{df}_{\rm in}+\log r)\}$, simultaneously for all $t\geq0$,
\[
\varepsilon^{(t)}
\leq
\left\{\prod_{s=0}^{t-1}
\left(1-\frac{3\alpha_s\sigma_X^2}{4}\right)\right\}
\varepsilon^{(0)}
+C
\frac{\overline\sigma_\xi}{\sigma_X\lambda_r}
\sqrt{\frac{\mathrm{df}}{n}}.
\]
\end{corollary}

\begin{proof}
Write $\gamma=C_0\sqrt{(\mathrm{df}_{\rm in}+\log r)/n}$ with a sufficiently large universal constant $C_0>1$, so that $n\geq C\mathrm{df}$ and $\mathrm{df}_{\rm in}+\log r\leq2\mathrm{df}$ give $\gamma\leq1/2$. Let $x_m=\operatorname{vec}(\calX_m)\in\RR^{p_{\rm in}^*}$, $\widehat\Sigma_X=n^{-1}\sum_{m=1}^nx_mx_m^\top$, and let $\operatorname{mat}_{\rm in}(\calT)\in\RR^{p_{\rm in}^*\times p_{\rm out}^*}$ be the unfolding of $\calT$ with the input modes as rows. By the definition of $\calA$ in Section~\ref{sec:model}, for every tensor $W$,
\begin{equation}
\operatorname{mat}_{\rm in}\{\calA^*\calA(\calT)\}
=
\widehat\Sigma_X\operatorname{mat}_{\rm in}(\calT),
\qquad
\langle W,\calA^*\calE\rangle
=
\frac1n\sum_{m=1}^n x_m^\top\operatorname{mat}_{\rm in}(W)\operatorname{vec}(\calE_m).
\label{eq:tot-reg-empirical-covariance}
\end{equation}
Let
\[
\mathfrak B_{\rm loc}
=
\bigcup_{i=1}^r
\Bigl\{\widetilde\calT\in\operatorname{Seg}:\|\widetilde\calT-\calT_i\|_{\rm F}\leq\frac{\lambda_i}{4d}\Bigr\}.
\]

\smallskip\noindent\emph{Step 1 (Restricted design bounds).}
For a nonzero $\widetilde\calT=\widetilde\lambda\widetilde u_1\otimes\cdots\otimes\widetilde u_d\in\operatorname{Seg}$ with unit factor vectors, the tangent space of the Segre manifold in $\RR^{p_1\times\cdots\times p_{d_{\rm in}}}$ at $\widetilde u_1\otimes\cdots\otimes\widetilde u_{d_{\rm in}}$ is
\[
\TT_{\widetilde u_1\otimes\cdots\otimes\widetilde u_{d_{\rm in}}}
=
\operatorname{span}\{\widetilde u_1\otimes\cdots\otimes
\widetilde u_{d_{\rm in}}\}
+
\sum_{\ell=1}^{d_{\rm in}}
\left\{
\widetilde u_1\otimes\cdots\otimes v_\ell\otimes\cdots\otimes
\widetilde u_{d_{\rm in}}:v_\ell\perp\widetilde u_\ell
\right\}.
\]
It has dimension $\mathrm{df}_{\rm in}$, and by \eqref{eq:tangent space projection} every column of $\operatorname{mat}_{\rm in}(Z)$ belongs to it whenever $Z\in\TT_{\widetilde\calT}$. Let $\mathfrak U_{\rm loc}$ be the union of the unit spheres of $\TT_{\widetilde u_1\otimes\cdots\otimes\widetilde u_{d_{\rm in}}}$ over $\widetilde\calT\in\mathfrak B_{\rm loc}$. For $\|\widetilde\calT-\calT_i\|_{\rm F}\leq\lambda_i/(4d)$, \eqref{eq:local-angle-sum} with $\nu=1/(4d)$ gives $\sum_{\ell=1}^{d_{\rm in}}\|\calP_{u_{\ell,i}}^\perp\widetilde u_\ell\|_2^2\leq1/(8d^2)$, and the proof of \eqref{eq:local-tangent-entropy}, applied to the first $d_{\rm in}$ factors with $(d,\mathrm{df})$ replaced by $(d_{\rm in},\mathrm{df}_{\rm in})$, gives
\begin{equation}
\log N(\epsilon,\mathfrak U_{\rm loc},\|\cdot\|_2)
\leq
C\mathrm{df}_{\rm in}\log(C/\epsilon)+\log r
\leq
C(\mathrm{df}_{\rm in}+\log r)\log(C/\epsilon),
\qquad 0<\epsilon<1.
\label{eq:tot-reg-input-local-entropy}
\end{equation}
Let $\mathbb X\in\RR^{n\times p_{\rm in}^*}$ have rows $\sigma_X^{-1}x_m^\top$, so that $\mathbb X$ has independent $\mathcal N(0,1)$ entries. Since $\mathfrak U_{\rm loc}=-\mathfrak U_{\rm loc}$, the set $\mathfrak U_{\rm loc}+\mathfrak U_{\rm loc}$ contains $y\pm y'$ for all $y,y'\in\mathfrak U_{\rm loc}$ and has radius at most two, and by Dudley's inequality \citep[Theorem~8.1.3]{vershynin2018high} and \eqref{eq:tot-reg-input-local-entropy} its Gaussian complexity is at most $C\sqrt{\mathrm{df}_{\rm in}+\log r}$. The matrix deviation inequality \citep[Theorem~9.1.1]{vershynin2018high} gives
\[
\mathbb E\sup_{z\in\mathfrak U_{\rm loc}+\mathfrak U_{\rm loc}}\bigl|\|\mathbb Xz\|_2-\sqrt n\|z\|_2\bigr|
\leq C\sqrt{\mathrm{df}_{\rm in}+\log r}.
\]
The supremum is a $2$-Lipschitz function of $\mathbb X$ in the Frobenius norm, so Gaussian concentration \citep[Theorem~5.2.2]{vershynin2018high} gives, with probability at least $1-2\exp\{-c(\mathrm{df}_{\rm in}+\log r)\}$,
\[
\sup_{z\in\mathfrak U_{\rm loc}+\mathfrak U_{\rm loc}}\bigl|\|\mathbb Xz\|_2-\sqrt n\|z\|_2\bigr|
\leq C\sqrt{\mathrm{df}_{\rm in}+\log r}\leq\frac{\sqrt n\,\gamma}{32}
\]
for $C_0$ sufficiently large. On this event, since $\|\mathbb Xz\|_2^2/n=z^\top\sigma_X^{-2}\widehat\Sigma_Xz$, every $z=y\pm y'$ with $y,y'\in\mathfrak U_{\rm loc}$ satisfies
\[
\begin{aligned}
\bigl|(z^\top\sigma_X^{-2}\widehat\Sigma_Xz)^{1/2}-\|z\|_2\bigr|
&\leq\frac{\gamma}{32},\\
\bigl|z^\top(\sigma_X^{-2}\widehat\Sigma_X-I)z\bigr|
&\leq
\frac{\gamma}{32}\left(2\|z\|_2+\frac{\gamma}{32}\right)
\leq\frac{\gamma}{8}+\frac{\gamma^2}{1024},
\end{aligned}
\]
and polarization with $\gamma\leq1/2$ yields
\begin{equation}
\sup_{y,y'\in\mathfrak U_{\rm loc}}
\bigl|y^\top(\sigma_X^{-2}\widehat\Sigma_X-I)y'\bigr|
\leq\frac12\left(\frac{\gamma}{8}+\frac{\gamma^2}{1024}\right)
\leq\frac{\gamma}{4}.
\label{eq:tot-reg-local-bilinear}
\end{equation}
For $\widetilde\calT,\widetilde\calT'\in\mathfrak B_{\rm loc}$, $Y\in\TT_{\widetilde\calT}$, and $Z\in\TT_{\widetilde\calT'}$, let $y_k$ and $z_k$ be the columns of $\operatorname{mat}_{\rm in}(Y)$ and $\operatorname{mat}_{\rm in}(Z)$. Then \eqref{eq:tot-reg-empirical-covariance}, \eqref{eq:tot-reg-local-bilinear}, and the Cauchy--Schwarz inequality give
\begin{equation}
\bigl|\langle Y,(\sigma_X^{-2}\calA^*\calA-I)Z\rangle\bigr|
=
\Bigl|\sum_k y_k^\top(\sigma_X^{-2}\widehat\Sigma_X-I)z_k\Bigr|
\leq
\frac{\gamma}{4}\sum_k\|y_k\|_2\|z_k\|_2
\leq
\frac{\gamma}{4}\|Y\|_{\rm F}\|Z\|_{\rm F}.
\label{eq:tot-reg-tangent-bilinear}
\end{equation}
For $Y\in\TT_{\widetilde\calT}$ and an arbitrary tensor $Z$, since $\langle Y,\calP_{\TT_{\widetilde\calT}}^\perp W\rangle=0$ for every tensor $W$,
\[
\bigl\langle Y,\sigma_X^{-2}\calA^*\calA\,\calP_{\TT_{\widetilde\calT}}^\perp\calP_{\TT_{\widetilde\calT'}}Z\bigr\rangle
=\bigl\langle Y,(\sigma_X^{-2}\calA^*\calA-I)\calP_{\TT_{\widetilde\calT'}}Z\bigr\rangle
-\bigl\langle Y,(\sigma_X^{-2}\calA^*\calA-I)\calP_{\TT_{\widetilde\calT}}\calP_{\TT_{\widetilde\calT'}}Z\bigr\rangle,
\]
and \eqref{eq:tot-reg-tangent-bilinear} bounds the right-hand side by $(\gamma/2)\|Y\|_{\rm F}\|Z\|_{\rm F}$. Consequently, for all $\widetilde\calT,\widetilde\calT'\in\mathfrak B_{\rm loc}$,
\begin{align}
\left\|
\calP_{\TT_{\widetilde\calT}}
(\sigma_X^{-2}\calA^*\calA-I)
\calP_{\TT_{\widetilde\calT}}
\right\|
&\leq\frac{\gamma}{4},
\label{eq:tot-reg-op-tangent}\\
\left\|
\calP_{\TT_{\widetilde\calT}}
\sigma_X^{-2}\calA^*\calA
\calP_{\TT_{\widetilde\calT}}^\perp
\calP_{\TT_{\widetilde\calT'}}
\right\|
&\leq\gamma.
\label{eq:tot-reg-op-cross}
\end{align}

\smallskip\noindent\emph{Step 2 (Normal chord bound).}
Fix $i,j\in[r]$ and $\widehat\calT_i,\widehat\calT_j\in\operatorname{Seg}$ with $\|\widehat\calT_k-\calT_k\|_{\rm F}\leq\lambda_k/(4d)$ for $k\in\{i,j\}$, and let $\delta_j=\|D_j\|_{\rm F}/\lambda_j$. We show that, with probability at least $1-e^{-\mathrm{df}}$, simultaneously for all such $i,j,\widehat\calT_i,\widehat\calT_j$,
\begin{equation}
\left\|
\calP_{\widehat\TT_i}
(\sigma_X^{-2}\calA^*\calA-I)
\calP_{\widehat\TT_i}^\perp
\calP_{\widehat\TT_j}^\perp D_j
\right\|_{\rm F}
\leq
C\left\{
\sqrt{\frac{d\,\mathrm{df}}{n}}
+
\frac{d\,\mathrm{df}}{n}
\right\}
\frac{\|D_j\|_{\rm F}^2}{\lambda_j}
\leq
Cd\frac{\|D_j\|_{\rm F}^2}{\lambda_j}.
\label{eq:tot-reg-op-normal-chord}
\end{equation}
Both sides vanish when $\delta_j=0$, so let $\delta_j>0$. By \eqref{eq:subset-expansion} and the orthogonality argument following it, applied with $(\calT_i,\calT_i^{(t)})$ replaced by $(\calT_j,\widehat\calT_j)$, there are orthonormal rank-one tensors $V_S$, $S\subseteq[d]$, and numbers $s_\ell\in[0,1]$, $c_\ell=(1-s_\ell^2)^{1/2}$, such that
\begin{equation}
\calP_{\widehat\TT_j}^\perp D_j
=
-\lambda_j\sum_{|S|\geq2}
\left(\prod_{\ell\in S}s_\ell
\prod_{\ell\notin S}c_\ell\right)V_S,
\qquad
\sum_{\ell=1}^d s_\ell^2
\leq2\delta_j^2,
\label{eq:tot-reg-normal-subset}
\end{equation}
where the second bound is the inequality $q\leq2\{\delta_i^{(t)}\}^2$ in the proof of Lemma~\ref{lemma: perturbation bound}, applied to $(\calT_j,\widehat\calT_j)$. Since every $S$ with $|S|\geq2$ contains a pair $\{k,\ell\}$, all summands are nonnegative, and $s_m^2+c_m^2=1$,
\begin{equation}
\begin{aligned}
\frac{\|\calP_{\widehat\TT_j}^\perp D_j\|_{\rm F}^2}
{\lambda_j^2}
&=
\sum_{|S|\geq2}
\prod_{\ell\in S}s_\ell^2
\prod_{\ell\notin S}c_\ell^2
\leq
\sum_{1\leq k<\ell\leq d}s_k^2s_\ell^2
\sum_{S\supseteq\{k,\ell\}}\prod_{m\in S\setminus\{k,\ell\}}s_m^2\prod_{m\notin S}c_m^2\\
&=
\sum_{1\leq k<\ell\leq d}s_k^2s_\ell^2
\leq\frac12\left(\sum_{\ell=1}^d s_\ell^2\right)^2
\leq2\delta_j^4.
\end{aligned}
\label{eq:tot-reg-normal-radius}
\end{equation}
Set $\chi_\ell=s_\ell/\delta_j$, so that $\|\chi\|_2\leq\sqrt2$, $\|\chi\|_1\leq\sqrt{2d}$, and $\delta_j|\chi_\ell|\leq\sqrt2/(4d)$. Dividing \eqref{eq:tot-reg-normal-subset} by $\lambda_j\delta_j^2$ gives
\[
\frac{\calP_{\widehat\TT_j}^\perp D_j}{\lambda_j\delta_j^2}
=
-\sum_{|S|\geq2}a_S(\delta_j,\chi)V_S,
\qquad
a_S(\delta,\chi)
=
\delta^{|S|-2}
\left(\prod_{\ell\in S}\chi_\ell\right)
\prod_{\ell\notin S}
\sqrt{1-\delta^2\chi_\ell^2},
\]
which extends continuously to $\delta_j=0$. On the domain $0\leq\delta\leq1/(4d)$, $\|\chi\|_2\leq\sqrt2$, every square-root factor is at least $1/2$ and $\delta\|\chi\|_1\leq1$. Bounding the elementary symmetric polynomials of $|\chi_1|,\ldots,|\chi_d|$ by $\|\chi\|_1^k/k!$ gives
\begin{equation}
\sum_{|S|\geq2}|a_S(\delta,\chi)|
\leq
\sum_{k\geq2}\frac{\delta^{k-2}\|\chi\|_1^k}{k!}
\leq
\frac{\|\chi\|_1^2}{2}
\exp(\delta\|\chi\|_1)
\leq Cd.
\label{eq:tot-reg-normal-coefficient-sum}
\end{equation}
Differentiating the products in $a_S$ and summing in the same way, the terms obtained from $\delta^{|S|-2}$ are at most $C\|\chi\|_1^3\exp(\delta\|\chi\|_1)$, those obtained from the square-root factors are at most $C\delta\|\chi\|_1^2\|\chi\|_2^2\exp(\delta\|\chi\|_1)$, and the derivatives in $\chi$ are at most $Cd\|\chi\|_1\exp(\delta\|\chi\|_1)$, so that
\begin{equation}
\sum_{|S|\geq2}
\left\{
|\partial_\delta a_S(\delta,\chi)|
+\|\nabla_\chi a_S(\delta,\chi)\|_1
\right\}
\leq Cd^2.
\label{eq:tot-reg-normal-coefficient-derivative}
\end{equation}

Let $\mathfrak H$ be the Frobenius-norm closure of the set consisting of the zero tensor and all tensors
\[
H
=
Z\pm
\frac{
\calP_{\widehat\TT_i}^\perp
\calP_{\widehat\TT_j}^\perp D_j}
{\lambda_j\delta_j^2},
\qquad
Z\in\widehat\TT_i,\qquad \|Z\|_{\rm F}=1,
\]
over $i,j\in[r]$ and all $\widehat\calT_i,\widehat\calT_j$ as above with $\delta_j>0$. By \eqref{eq:tot-reg-normal-radius}, $\sup_{H\in\mathfrak H}\|H\|_{\rm F}\leq1+\sqrt2$. For fixed $(i,j)$ and sign, $H=H(\vartheta)$ is a function of
\[
\vartheta=\bigl(\delta_j,\chi,(\widehat u_{\ell,j},w_\ell)_{\ell},(\widehat u_{\ell,i})_{\ell},\beta\bigr),
\]
where $\ell$ ranges over the modes with $p_\ell\geq2$, $(\widehat u_{\ell,j},w_\ell)$ are the orthonormal pairs that define $V_S$ in \eqref{eq:subset-expansion}, $(\widehat u_{\ell,i})_\ell$ are the unit factors of $\widehat\calT_i$, and $\beta\in\SS^{\mathrm{df}-1}$ gives $Z=Q\sum_{k=1}^{\mathrm{df}}\beta_kE_k$ for a fixed orthonormal basis $E_1,\ldots,E_{\mathrm{df}}$ of $\TT_{\calT_i}$ and the rotation $Q$ in \eqref{eq:rotated-tangent-space} with $\widetilde\calT=\widehat\calT_i$. If $p_\ell=1$, the sign convention gives $\widehat u_{\ell,i}=u_{\ell,i}$, $\widehat u_{\ell,j}=u_{\ell,j}$, and $s_\ell=\chi_\ell=0$, so mode $\ell$ carries no parameter, and if $\chi_\ell=0$, then $w_\ell$ may be any unit vector orthogonal to $\widehat u_{\ell,j}$. There are at most $\mathrm{df}-1$ modes with $p_\ell\geq2$, and $\sum_\ell p_\ell\leq2(\mathrm{df}-1)$ over them, so
\[
\|\vartheta\|_2^2\leq\frac1{16}+2+3d+1\leq9d,
\qquad
\dim(\vartheta)\leq1+(\mathrm{df}-1)+6(\mathrm{df}-1)+\mathrm{df}\leq8\,\mathrm{df}.
\]
For unit vectors $u_\ell,u'_\ell$, $\ell\in[d]$, by \eqref{eq:tangent space projection},
\begin{gather*}
\Bigl\|\bigotimes_{\ell=1}^d u_\ell-\bigotimes_{\ell=1}^d u'_\ell\Bigr\|_{\rm F}
\leq\sum_{\ell=1}^d\|u_\ell-u'_\ell\|_2,\\
\|\calP_{\TT_{u_1\otimes\cdots\otimes u_d}}-\calP_{\TT_{u'_1\otimes\cdots\otimes u'_d}}\|
\leq
2(d+1)\sum_{\ell=1}^d\|u_\ell-u_\ell'\|_2,\\
\calP_{\TT_{u_1\otimes\cdots\otimes u_d}}
=
\bigotimes_{\ell=1}^d u_\ell u_\ell^\top
+
\sum_{k=1}^d
(I-u_ku_k^\top)\otimes
\bigotimes_{\ell\ne k}u_\ell u_\ell^\top.
\end{gather*}
Together with the bound $\|Q-Q'\|\leq7\sqrt d(\sum_\ell\|\widehat u_{\ell,i}-\widehat u'_{\ell,i}\|_2^2)^{1/2}$ in the proof of Lemma~\ref{lemma:local-gaussian-tangent}, \eqref{eq:tot-reg-normal-coefficient-sum}, and \eqref{eq:tot-reg-normal-coefficient-derivative}, these inequalities give
\[
\|H(\vartheta)-H(\vartheta')\|_{\rm F}
\leq Cd^2\|\vartheta-\vartheta'\|_2.
\]
Since the parameter domain is compact and $H$ is continuous, every nonzero element of $\mathfrak H$ equals $H(\vartheta)$ for some $(i,j)$, sign, and $\vartheta$. By \citet[Lemma~4.2.8]{vershynin2018high}, the volume comparison in the proof of \citep[Proposition~4.2.12]{vershynin2018high}, the Lipschitz bound, and $2\log r\leq2(\mathrm{df}-1)$ for the union over $(i,j)$,
\begin{equation}
\log N(\epsilon,\mathfrak H,\|\cdot\|_{\rm F})
\leq
\log3+2\log r+8\,\mathrm{df}\log\left(1+\frac{6Cd^{5/2}}{\epsilon}\right)
\leq
Cd\,\mathrm{df}\log(C/\epsilon),
\qquad 0<\epsilon<1.
\label{eq:tot-reg-normal-entropy}
\end{equation}

For $H\in\mathfrak H$, let $B_H=n^{-1/2}\operatorname{diag}(\operatorname{mat}_{\rm in}(H)^\top,\ldots,\operatorname{mat}_{\rm in}(H)^\top)$ with $n$ diagonal blocks and $g=\sigma_X^{-1}(x_1^\top,\ldots,x_n^\top)^\top\sim\mathcal N(0,I)$. By \eqref{eq:tot-reg-empirical-covariance},
\begin{equation}
\begin{gathered}
\Xi_H
:=
\|B_Hg\|_2^2-\mathbb E\|B_Hg\|_2^2
=
\left\langle
H,(\sigma_X^{-2}\calA^*\calA-I)H
\right\rangle,\\
\|B_H\|_{\rm F}=\|H\|_{\rm F},
\qquad
\|B_H\|\leq\frac{\|H\|_{\rm F}}{\sqrt n},\\
\|B_H-B_{H'}\|_{\rm F}=\|H-H'\|_{\rm F},
\qquad
\|B_H-B_{H'}\|\leq\frac{\|H-H'\|_{\rm F}}{\sqrt n}.
\end{gathered}
\label{eq:tot-reg-normal-chaos-identity}
\end{equation}
Hence $\|B_H^\top B_H-B_{H'}^\top B_{H'}\|_{\rm F}\leq Cn^{-1/2}\|H-H'\|_{\rm F}$ and $\|B_H^\top B_H-B_{H'}^\top B_{H'}\|\leq Cn^{-1}\|H-H'\|_{\rm F}$, and the Gaussian Hanson--Wright inequality \citep[Theorem~6.2.1]{vershynin2018high} gives, for every $u>0$,
\[
\mathbb P\left\{
|\Xi_H-\Xi_{H'}|
>
C\|H-H'\|_{\rm F}
\left(\sqrt{u/n}+u/n\right)
\right\}
\leq2e^{-u}.
\]
Thus $(\Xi_H)_{H\in\mathfrak H}$ satisfies the mixed-tail condition (3.8) of \citet{dirksen2015tail} with $d_2(H,H')=C\|H-H'\|_{\rm F}/\sqrt n$ and $d_1(H,H')=C\|H-H'\|_{\rm F}/n$. Let $F\subset\mathfrak H$ be finite with $0\in F$. Since $\Xi_0=0$, $N(\epsilon,F,\|\cdot\|_{\rm F})\leq N(\epsilon/2,\mathfrak H,\|\cdot\|_{\rm F})$ \citep[Lemma~4.2.8]{vershynin2018high}, and $\sup_{H\in\mathfrak H}\|H\|_{\rm F}\leq1+\sqrt2$, Theorem~3.5 of \citet{dirksen2015tail} with $t_0=0$ and $u=\mathrm{df}$, the entropy bound (2.3) there, and \eqref{eq:tot-reg-normal-entropy} give, with probability at least $1-e^{-\mathrm{df}}$,
\begin{align*}
\max_{H\in F}|\Xi_H|
&\leq
\frac C{\sqrt n}\int_0^\infty\sqrt{\log N(\epsilon,F,\|\cdot\|_{\rm F})}\,d\epsilon
+\frac Cn\int_0^\infty\log N(\epsilon,F,\|\cdot\|_{\rm F})\,d\epsilon\\
&\quad+C\left(\sqrt{\frac{\mathrm{df}}n}+\frac{\mathrm{df}}n\right)
\leq
C\left\{
\sqrt{\frac{d\,\mathrm{df}}{n}}
+
\frac{d\,\mathrm{df}}{n}
\right\},
\end{align*}
where $C$ does not depend on $F$. Letting $F$ increase along finite sets whose union is dense in $\mathfrak H$ and using the continuity of $H\mapsto \Xi_H$,
\begin{equation}
\sup_{H\in\mathfrak H}|\Xi_H|
\leq
C\left\{
\sqrt{\frac{d\,\mathrm{df}}{n}}
+
\frac{d\,\mathrm{df}}{n}
\right\}
\label{eq:tot-reg-normal-chaos-bound}
\end{equation}
with probability at least $1-e^{-\mathrm{df}}$. For $W=\calP_{\widehat\TT_i}^\perp\calP_{\widehat\TT_j}^\perp D_j/(\lambda_j\delta_j^2)$ and unit $Z\in\widehat\TT_i$, polarization in \eqref{eq:tot-reg-normal-chaos-identity} gives
\[
4\left\langle
Z,(\sigma_X^{-2}\calA^*\calA-I)W
\right\rangle
=\Xi_{Z+W}-\Xi_{Z-W}.
\]
Taking the supremum over $Z$ and multiplying by $\lambda_j\delta_j^2=\|D_j\|_{\rm F}^2/\lambda_j$ proves \eqref{eq:tot-reg-op-normal-chord} on the event of \eqref{eq:tot-reg-normal-chaos-bound}, where the last inequality in \eqref{eq:tot-reg-op-normal-chord} uses $n\geq C\mathrm{df}$.

\smallskip\noindent\emph{Step 3 (Projected noise).}
Let
\[
\mathfrak Z_{\rm loc}
=
\bigcup_{\widetilde\calT\in\mathfrak B_{\rm loc}}
\{Z\in\TT_{\widetilde\calT}:\|Z\|_{\rm F}=1\},
\qquad
\sup_{Z\in\mathfrak Z_{\rm loc}}|\langle Z,\calA^*\calE\rangle|
=
\sup_{\widetilde\calT\in\mathfrak B_{\rm loc}}
\|\calP_{\TT_{\widetilde\calT}}\calA^*\calE\|_{\rm F}.
\]
By \eqref{eq:local-tangent-entropy} with $\nu=1/(4d)$ and $\log r\leq\mathrm{df}-1$,
\begin{equation}
\log N(\epsilon,\mathfrak Z_{\rm loc},\|\cdot\|_{\rm F})
\leq C\mathrm{df}\log(C/\epsilon),
\qquad 0<\epsilon<1.
\label{eq:tot-reg-noise-entropy}
\end{equation}
Fix a tensor $W$. Whitening $x_m$ and $\operatorname{vec}(\calE_m)$ and taking a singular value decomposition of $\sigma_X\operatorname{mat}_{\rm in}(W)\operatorname{Cov}\{\operatorname{vec}(\calE_m)\}^{1/2}$, with singular values $s_{m,k}$, give
\begin{gather*}
x_m^\top\operatorname{mat}_{\rm in}(W)\operatorname{vec}(\calE_m)
\overset{\rm d}{=}
\sum_k s_{m,k}g_{m,k}g'_{m,k},\\
\sum_k s_{m,k}^2
\leq
\sigma_X^2\overline\sigma_\xi^2\|W\|_{\rm F}^2,
\qquad
\max_k s_{m,k}
\leq
\sigma_X\overline\sigma_\xi\|W\|_{\rm F},
\end{gather*}
with mutually independent standard Gaussian variables $g_{m,k},g'_{m,k}$. For $|\omega|\leq\{2\sigma_X\overline\sigma_\xi\|W\|_{\rm F}\}^{-1}$,
\[
\log\mathbb E\exp\{\omega x_m^\top\operatorname{mat}_{\rm in}(W)\operatorname{vec}(\calE_m)\}
=
-\frac12\sum_k\log(1-\omega^2s_{m,k}^2)
\leq
\omega^2\sigma_X^2\overline\sigma_\xi^2\|W\|_{\rm F}^2.
\]
Hence the summands in the second identity of \eqref{eq:tot-reg-empirical-covariance} are independent, centered, and sub-exponential with parameter at most $C\sigma_X\overline\sigma_\xi\|W\|_{\rm F}$ \citep[Proposition~2.7.1]{vershynin2018high}, and Bernstein's inequality \citep[Corollary~2.8.3]{vershynin2018high} gives, for every $u>0$,
\begin{equation}
\mathbb P\left\{
|\langle W,\calA^*\calE\rangle|
>
C\sigma_X\overline\sigma_\xi\|W\|_{\rm F}
\left(\sqrt{u/n}+u/n\right)
\right\}
\leq2e^{-u}.
\label{eq:tot-reg-noise-bernstein}
\end{equation}
With $W=Z-Z'$, the process $(\langle Z,\calA^*\calE\rangle)_{Z\in\mathfrak Z_{\rm loc}}$ satisfies (3.8) of \citet{dirksen2015tail} with $d_2(Z,Z')=C\sigma_X\overline\sigma_\xi\|Z-Z'\|_{\rm F}/\sqrt n$ and $d_1(Z,Z')=C\sigma_X\overline\sigma_\xi\|Z-Z'\|_{\rm F}/n$. Fix $Z_0\in\mathfrak Z_{\rm loc}$. The diameter of $\mathfrak Z_{\rm loc}$ is at most two, so the chaining argument of Step~2, with Theorem~3.5 of \citet{dirksen2015tail} applied at $t_0=Z_0$ and $u=\mathrm{df}$ and with \eqref{eq:tot-reg-noise-entropy} in place of \eqref{eq:tot-reg-normal-entropy}, together with \eqref{eq:tot-reg-noise-bernstein} for $W=Z_0$ and $u=\mathrm{df}$, gives, with probability at least $1-3e^{-\mathrm{df}}$,
\begin{equation}
\begin{aligned}
\sup_{\widetilde\calT\in\mathfrak B_{\rm loc}}
\|\calP_{\TT_{\widetilde\calT}}\calA^*\calE\|_{\rm F}
&\leq|\langle Z_0,\calA^*\calE\rangle|+\sup_{Z\in\mathfrak Z_{\rm loc}}|\langle Z-Z_0,\calA^*\calE\rangle|\\
&\leq
C\sigma_X\overline\sigma_\xi
\left(
\sqrt{\frac{\mathrm{df}}n}+\frac{\mathrm{df}}n
\right)
\leq
C_0\sigma_X\overline\sigma_\xi
\sqrt{\frac{\mathrm{df}}n},
\end{aligned}
\label{eq:tot-reg-noise-bound}
\end{equation}
where the last inequality uses $n\geq C\mathrm{df}$ and $C_0$ sufficiently large.

Since $\mathrm{df}_{\rm in}+\log r\leq2\mathrm{df}$, the events of Steps~1--3 hold simultaneously with probability at least
\[
1-2\exp\{-c(\mathrm{df}_{\rm in}+\log r)\}-4e^{-\mathrm{df}}
\geq
1-C\exp\{-c(\mathrm{df}_{\rm in}+\log r)\}.
\]
The remaining steps are on this event.

\smallskip\noindent\emph{Step 4 (One component update).}
Let $\{\widehat\calT_j\}_{j=1}^r\subset\operatorname{Seg}$ satisfy $\varepsilon=\max_{j}\|D_j\|_{\mathrm F}/\lambda_j\leq1/(4d)$, so that $\widehat\calT_j\in\mathfrak B_{\rm loc}$ for every $j\in[r]$. Equations \eqref{eq:tot-reg-op-tangent}, \eqref{eq:tot-reg-op-cross}, and \eqref{eq:tot-reg-noise-bound} give
\begin{equation}
\begin{gathered}
\max_i
\|\calP_{\widehat{\TT}_i}(I-\alpha_t\calA^*\calA)
\calP_{\widehat{\TT}_i}\|
\leq
|1-\alpha_t\sigma_X^2|+\alpha_t\sigma_X^2\gamma,
\qquad
\max_i
\|\calP_{\widehat{\TT}_i}\calA^*\calA
\calP_{\widehat{\TT}_i}\|
\leq\sigma_X^2(1+\gamma),\\
\max_{i\ne j}
\|\calP_{\widehat{\TT}_i}\calA^*\calA
\calP_{\widehat{\TT}_i}^{\perp}\calP_{\widehat{\TT}_j}\|
\leq\sigma_X^2\gamma,
\qquad
\max_i
\frac{\|\calP_{\widehat{\TT}_i}\calA^*\calE\|_{\rm F}}{\lambda_i}
\leq
C_0\frac{\sigma_X\overline\sigma_\xi}{\lambda_r}
\sqrt{\frac{\mathrm{df}}{n}}.
\end{gathered}
\label{eq:tot-reg-rgd-operators}
\end{equation}
For RGD, $\widehat\xi_i=-\alpha_t\calP_{\widehat{\TT}_i}\calA^*\calA\sum_{j=1}^rD_j+\alpha_t\calP_{\widehat{\TT}_i}\calA^*\calE$ and $\widehat W_i-\calT_i=D_i+\widehat\xi_i$. With $\mathrm{II},\ldots,\mathrm{VI}$ as in the proof of Theorem~\ref{thm:local_convergence_rgd},
\[
\begin{aligned}
D_i-\alpha_t\calP_{\widehat{\TT}_i}\calA^*\calA D_i
={}&(1-\alpha_t\sigma_X^2)D_i
+\alpha_t\sigma_X^2\calP_{\widehat{\TT}_i}^{\perp}D_i\\
&-\alpha_t\calP_{\widehat{\TT}_i}
(\calA^*\calA-\sigma_X^2I)\calP_{\widehat{\TT}_i}D_i
-\alpha_t\calP_{\widehat{\TT}_i}\calA^*\calA
\calP_{\widehat{\TT}_i}^{\perp}D_i,
\end{aligned}
\]
where the third term has norm at most $\alpha_t\sigma_X^2\gamma\lambda_i\varepsilon$ by \eqref{eq:tot-reg-op-tangent} and the last term has norm $\alpha_t\mathrm{III}$. Since $\calP_{\widehat\TT_i}\calP_{\widehat\TT_i}^\perp=0$, \eqref{eq:tot-reg-op-normal-chord} remains valid with $\sigma_X^{-2}\calA^*\calA-I$ replaced by $\sigma_X^{-2}\calA^*\calA$. Applying it with $j=i$ and, after the triangle inequality, with $j\ne i$, and using \eqref{eq:PTip(Tihat - Ti)}, \eqref{eq:PTip(Tjhat - Tj)}, and \eqref{eq:tot-reg-rgd-operators}, we obtain
\[
\begin{gathered}
\frac{\mathrm{II}}{\lambda_i}\leq3d\varepsilon^2,
\qquad
\frac{\mathrm{III}}{\lambda_i}\leq Cd\sigma_X^2\varepsilon^2,
\qquad
\frac{\mathrm{IV}}{\lambda_i}
\leq r\kappa\sigma_X^2(1+\gamma)
\{(d+1)\eta^{d-2}+4d\varepsilon\}\varepsilon,\\
\frac{\mathrm V}{\lambda_i}\leq(r-1)\kappa\sigma_X^2\gamma\varepsilon,
\qquad
\frac{\mathrm{VI}}{\lambda_i}\leq Cd(r-1)\kappa\sigma_X^2\varepsilon^2.
\end{gathered}
\]
Since $\|\widehat W_i-\calT_i\|_{\mathrm F}\leq\|D_i-\alpha_t\calP_{\widehat{\TT}_i}\calA^*\calA D_i\|_{\mathrm F}+\alpha_t(\mathrm{IV}+\mathrm V+\mathrm{VI})+\alpha_t\|\calP_{\widehat{\TT}_i}\calA^*\calE\|_{\mathrm F}$, $0<\alpha_t\sigma_X^2\leq1$, and $\gamma\leq1/2$,
\begin{align}
\frac{\|\widehat W_i-\calT_i\|_{\mathrm F}}{\lambda_i}
&\leq(1-\alpha_t\sigma_X^2)\varepsilon
+\alpha_t\sigma_X^2
\bigl\{(1+r\kappa)\gamma
+(d+1)r\kappa(1+\gamma)\eta^{d-2}\bigr\}\varepsilon
\nonumber\\
&\quad
+C d\alpha_t\sigma_X^2(1+r\kappa)\varepsilon^2
+C_0\alpha_t\frac{\sigma_X\overline\sigma_\xi}{\lambda_r}
\sqrt{\frac{\mathrm{df}}n},
\label{eq:tot-reg-rgd-pre-retraction}\\
\frac{\|\widehat\xi_i\|_{\mathrm F}}{\lambda_i}
&\leq\alpha_t\sigma_X^2\varepsilon
+\alpha_t\sigma_X^2\bigl\{(1+r\kappa)\gamma+(d+1)r\kappa(1+\gamma)\eta^{d-2}\bigr\}\varepsilon
\nonumber\\
&\quad
+Cd\alpha_t\sigma_X^2(1+r\kappa)\varepsilon^2
+C_0\alpha_t\frac{\sigma_X\overline\sigma_\xi}{\lambda_r}\sqrt{\frac{\mathrm{df}}n},
\label{eq:tot-reg-rgd-increment}
\end{align}
where the second bound uses $\widehat\xi_i=(D_i-\alpha_t\calP_{\widehat{\TT}_i}\calA^*\calA D_i)-D_i-\alpha_t\calP_{\widehat{\TT}_i}\calA^*\calA\sum_{j\ne i}D_j+\alpha_t\calP_{\widehat{\TT}_i}\calA^*\calE$ and $\|(1-\alpha_t\sigma_X^2)D_i+\alpha_t\sigma_X^2\calP_{\widehat{\TT}_i}^{\perp}D_i-D_i\|_{\mathrm F}=\alpha_t\sigma_X^2\|\calP_{\widehat{\TT}_i}D_i\|_{\mathrm F}\leq\alpha_t\sigma_X^2\lambda_i\varepsilon$.

For RGN, with $\widehat{\mathsf H}_i$ as in the proof of Theorem~\ref{thm:local_convergence_rgn}, \eqref{eq:tot-reg-op-tangent}, $\gamma\leq1/2$, and \eqref{eq:tot-reg-rgd-operators} give
\begin{equation}
\begin{gathered}
\max_i
\left\|
\left(
\calP_{\widehat{\TT}_i}\calA^*\calA\calP_{\widehat{\TT}_i}
\big|_{\widehat{\TT}_i}
\right)^{-1}
\right\|
\leq\frac{2}{\sigma_X^2},
\qquad
\max_{i\ne j}
\|\widehat{\mathsf C}_i\calP_{\widehat\TT_j}\|
\leq2\gamma,\\
\max_i\frac{\|\widehat{\mathsf H}_i\calE\|_{\rm F}}{\lambda_i}
\leq
2C_0\frac{\overline\sigma_\xi}{\sigma_X\lambda_r}
\sqrt{\frac{\mathrm{df}}{n}}.
\end{gathered}
\label{eq:tot-reg-rgn-operators}
\end{equation}
With $\mathrm I,\ldots,\mathrm{VI}$ as in the proof of Theorem~\ref{thm:local_convergence_rgn}, \eqref{eq:tot-reg-rgn-operators}, \eqref{eq:tot-reg-op-normal-chord}, \eqref{eq:PTip(Tihat - Ti)}, and \eqref{eq:PTip(Tjhat - Tj)} give
\begin{gather*}
\frac{\mathrm I}{\lambda_i}
\leq r\kappa\{(d+1)\eta^{d-2}+4d\varepsilon\}\varepsilon,
\qquad
\frac{\mathrm{II}}{\lambda_i}\leq Cd\varepsilon^2,
\qquad
\frac{\mathrm{III}}{\lambda_i}\leq2(r-1)\kappa\gamma\varepsilon,\\
\frac{\mathrm{IV}}{\lambda_i}\leq Cd(r-1)\kappa\varepsilon^2,
\qquad
\frac{\mathrm V}{\lambda_i}\leq3d\varepsilon^2,
\end{gather*}
and therefore
\begin{equation}
\frac{\|\widehat W_i-\calT_i\|_{\mathrm F}}{\lambda_i}
\leq
\{2r\kappa\gamma+(d+1)r\kappa\eta^{d-2}\}\varepsilon
+C d(1+r\kappa)\varepsilon^2
+2C_0\frac{\overline\sigma_\xi}{\sigma_X\lambda_r}
\sqrt{\frac{\mathrm{df}}n}.
\label{eq:tot-reg-rgn-pre-retraction}
\end{equation}

\smallskip\noindent\emph{Step 5 (Composition).}
Let
\[
R=\frac{c_0}{d(1+\kappa r)},
\qquad
b=8C_0
\frac{\overline\sigma_\xi}{\sigma_X\lambda_r}
\sqrt{\frac{\mathrm{df}}{n}},
\]
with $c_0$ so small that $R\leq\min\{1/(4d),1/24\}$ and $2R<1/\{2(\sqrt d+1)\}$, and with the constant $c$ of the corollary chosen afterwards so small that $\varepsilon^{(0)}\leq R$, $b\leq R/4$, and the inequalities below hold. Since $\alpha_t\sigma_X^2\geq1/2$, the conditions of the corollary give $\eta^{d-2}\leq2c/(d\kappa r)$ and $(1+r\kappa)\gamma\leq c$, so that
\[
(1+r\kappa)\gamma
+(d+1)r\kappa(1+\gamma)\eta^{d-2}
+C d(1+r\kappa)R
\leq5c+C c_0\leq\frac18.
\]
Hence \eqref{eq:tot-reg-rgd-pre-retraction} and \eqref{eq:tot-reg-rgd-increment} give, for every current collection with $\varepsilon=\max_j\|D_j\|_{\mathrm F}/\lambda_j\leq R$,
\[
\frac{\|\widehat W_i-\calT_i\|_{\mathrm F}}{\lambda_i}
\leq(1-\alpha_t\sigma_X^2)\varepsilon
+\frac{\alpha_t\sigma_X^2}{8}(\varepsilon+b),
\qquad
\frac{\|\widehat\xi_i\|_{\mathrm F}}{\lambda_i}
\leq\frac{\alpha_t\sigma_X^2}{8}(9\varepsilon+b),
\]
which is \eqref{eq:decomposition-rgd-linearized} with $\alpha_t\sigma_X^2$ in place of $\alpha_t$. The derivation of \eqref{eq:decomposition-rgd-retraction-domain}--\eqref{eq:decomposition-rgd-majorant} and Step~2 of the proof of Corollary~\ref{corollary:rate_decomposition_appendix} apply verbatim with this replacement and give
\[
\frac{\|\widehat\xi_i\|_{\mathrm F}}
{\|\widehat{\calT}_i\|_{\mathrm F}}
<\frac{1}{2(\sqrt d+1)},
\qquad
\frac{\|\widehat{\calT}_i^+-\calT_i\|_{\mathrm F}}{\lambda_i}
\leq
\left(1-\frac{3\alpha_t\sigma_X^2}{4}\right)\varepsilon
+\alpha_t\sigma_X^2b,
\qquad 0\leq \varepsilon\leq R,
\]
and then
\[
\varepsilon^{(t)}
\leq
\left\{\prod_{s=0}^{t-1}
\left(1-\frac{3\alpha_s\sigma_X^2}{4}\right)\right\}
\varepsilon^{(0)}+2b,
\qquad t\geq0.
\]
This proves the corollary with $C=16C_0$. Under the assumptions of Corollary~\ref{corollary:rate_RGD_regression}, $|1-\alpha_t\sigma_X^2|\leq c\leq1/3$ gives $1-3\alpha_t\sigma_X^2/4\leq1/2$, which proves that corollary.

For RGN, $\alpha_t\sigma_X^2=1$ and the conditions of the corollary give $2r\kappa\gamma+(d+1)r\kappa\eta^{d-2}+Cd(1+r\kappa)R\leq(10/3)c+Cc_0\leq1/4$, so \eqref{eq:tot-reg-rgn-pre-retraction} gives, for every current collection with $\varepsilon=\max_j\|D_j\|_{\mathrm F}/\lambda_j\leq R$,
\begin{gather*}
\frac{\|\widehat W_i-\calT_i\|_{\mathrm F}}{\lambda_i}
\leq\frac \varepsilon4+\frac b4,
\qquad
\frac{\|\widehat\xi_i\|_{\mathrm F}}{\lambda_i}
\leq\frac{5\varepsilon}4+\frac b4,\\
\frac{\|\widehat\xi_i\|_{\mathrm F}}
{\|\widehat{\calT}_i\|_{\mathrm F}}
\leq\frac{5R/4+b/4}{1-R}
<2R<\frac{1}{2(\sqrt d+1)},
\qquad
\frac{\|\widehat\xi_i\|_{\mathrm F}^2}
{\lambda_i\|\widehat{\calT}_i\|_{\mathrm F}}
\leq\frac{(5\varepsilon/4+b/4)^2}{1-\varepsilon}
\leq5\varepsilon^2+b^2.
\end{gather*}
Hence \eqref{eq:local-thosvd-quadratic} applies, and $b/4+b^2\leq b$ gives, after enlarging $C$,
\begin{equation}
\frac{\|\widehat{\calT}_i^+-\calT_i\|_{\mathrm F}}{\lambda_i}
\leq\Psi(\varepsilon)
:=\{2r\kappa\gamma+(d+1)r\kappa\eta^{d-2}\}\varepsilon
+C d(1+r\kappa)\varepsilon^2+b
\leq\frac \varepsilon2+b,
\qquad 0\leq \varepsilon\leq R.
\label{eq:regression-rgn-recursion}
\end{equation}
Step~2 of the proof of Corollary~\ref{corollary:rate_decomposition_appendix} with $\Psi$ in place of the right-hand side of \eqref{eq:decomposition-rgd-majorant} proves the linear assertion of Corollary~\ref{corollary:rate_RGN_regression}.

For the two-phase assertion, let $h=\gamma/d+\eta^{d-2}$ and $A=C_3d(1+\kappa r)$ with a sufficiently large universal constant $C_3$, and fix $C_4\geq4$. By \eqref{eq:regression-rgn-recursion},
\[
\Psi(\varepsilon)\leq\frac{A}{2}\{\varepsilon^2+h\varepsilon\}+b,
\qquad 0\leq \varepsilon\leq R,
\]
and $AR=C_3c_0$, $Ah=C_3(1+\kappa r)\{\gamma+d\eta^{d-2}\}\leq3C_3c$. After fixing $C_3$ and $C_4$, decrease $c_0$ and then $c$, if necessary, so that
\[
AR\leq\frac12,\qquad A(C_4+1)h\leq1,\qquad b\leq\frac{Ah^2}{8},
\]
where the last inequality follows from the additional noise condition in Corollary~\ref{corollary:rate_RGN_regression}. With $h$ in place of $\eta^{d-2}$, the argument following \eqref{eq:decomposition-rgn-majorant} gives
\begin{align*}
\Psi(\varepsilon)&\leq A\varepsilon^2\leq \varepsilon,
&&C_4h<\varepsilon\leq R,\\
\Psi(\varepsilon)&\leq\frac{\varepsilon}{2}+b,
&&0\leq \varepsilon\leq\min\{R,C_4h\},
\end{align*}
$2b\leq C_4h$, and then the quadratic-phase and linear-phase bounds of Corollary~\ref{corollary:rate_RGN_regression} for $t^*=\inf\{t\geq0:\varepsilon^{(t)}\leq C_4h\}$. Enlarging the universal constant multiplying the statistical error completes the proof.
\end{proof}

\section{Proof of Lemmas}
\label{sec:proof_lemmas}

This section proves Lemma~\ref{lemma: perturbation bound}, which underpins the convergence analysis.

\begin{lemma}\label{lemma: perturbation bound}
Suppose $d\geq3$. Let
$\calT_i=\lambda_i u_{1,i}\otimes\cdots\otimes u_{d,i}$, $i\in[r]$, be nonzero rank-one tensors with $\lambda_i>0$ and unit factor vectors satisfying Assumption~\ref{assumption:incoherence}. Let $\calT_i^{(t)}\in\operatorname{Seg}$ have tangent space $\TT_i^{(t)}$, and set
\[
\delta_i^{(t)}
=
\frac{\|\calT_i^{(t)}-\calT_i\|_{\mathrm F}}{\lambda_i}.
\]
If $\max_{i\in[r]}\delta_i^{(t)}\leq1/(4d)$, then, for every $i\in[r]$,
\begin{equation}\label{eq:PTip(Tihat - Ti)}
\left\|\calP_{\TT_i^{(t)}}^{\perp} \left(\calT_i\right)\right\|_{\mathrm{F}} \leq 3d \cdot \frac{\left\|\calT_i^{(t)} - \calT_i\right\|_{\mathrm{F}}^2}{\lambda_i}.
\end{equation}

Define the Segre atomic norm by
\[
\|W\|_{\operatorname{Seg},*}
=
\inf
\left\{
\sum_{k=1}^{m}|a_k|:
m\in\mathbb N,\quad
W=\sum_{k=1}^{m}a_kV_k,\quad
a_k\in\mathbb R,\quad
V_k\in\operatorname{Seg},\quad
\|V_k\|_{\mathrm F}=1
\right\}.
\]
Then
\begin{equation}
\label{eq:PTip_atomic}
\left\|
\calP_{\TT_i^{(t)}}^\perp
(\calT_i^{(t)}-\calT_i)
\right\|_{\operatorname{Seg},*}
=
\left\|
\calP_{\TT_i^{(t)}}^\perp
(\calT_i)
\right\|_{\operatorname{Seg},*}
\le
4d\,
\frac{
\|\calT_i^{(t)}-\calT_i\|_{\mathrm F}^2
}{\lambda_i}.
\end{equation}
For every pair $i\ne j$,
\begin{equation}\label{eq:PTip(Tjhat - Tj)}
\left\|\calP_{\TT_i^{(t)}} \left(\calT_j^{(t)} - \calT_j\right)\right\|_{\mathrm{F}}
\leq
\left\{(d+1)\eta^{d-2}+2d\delta_i^{(t)}+4\delta_j^{(t)}\right\}
\left\|\calT_j^{(t)}-\calT_j\right\|_{\mathrm{F}}.
\end{equation}

\end{lemma}

\begin{proof}
Fix $i\in[r]$. Write
\[
\calT_i
=
\lambda_i u_{1,i}\otimes\cdots\otimes u_{d,i},
\qquad
\calT_i^{(t)}
=
\lambda_i^{(t)}
u_{1,i}^{(t)}\otimes\cdots\otimes u_{d,i}^{(t)},
\]
where $\lambda_i^{(t)}=\|\calT_i^{(t)}\|_{\mathrm F}>0$ and all factor vectors have unit Euclidean norm. Since
$\|\calT_i^{(t)}\|_{\mathrm F}\geq(1-\delta_i^{(t)})\lambda_i$, the polarization identity gives
$2\langle\calT_i^{(t)},\calT_i\rangle\geq2(1-\delta_i^{(t)})\lambda_i^2>0$.
We therefore choose the factor signs so that
\[
c_\ell
=
\langle u_{\ell,i},u_{\ell,i}^{(t)}\rangle
\ge 0,
\qquad
\ell=1,\ldots,d.
\]
Let
\[
s_\ell=(1-c_\ell^2)^{1/2}.
\]
If \(s_\ell>0\), write
\[
u_{\ell,i}
=
c_\ell u_{\ell,i}^{(t)}
+
s_\ell w_\ell,
\qquad
w_\ell\perp u_{\ell,i}^{(t)},
\qquad
\|w_\ell\|_2=1.
\]
If \(s_\ell=0\), set \(w_\ell=0\). Terms containing this zero coefficient are
omitted from the expansion below.

Let
\[
U_i^{(t)}
=
u_{1,i}^{(t)}\otimes\cdots\otimes u_{d,i}^{(t)},
\qquad
U_i
=
u_{1,i}\otimes\cdots\otimes u_{d,i}.
\]
Since \(\calT_i^{(t)}=\lambda_i^{(t)}U_i^{(t)}\in\TT_i^{(t)}\), we have
\[
\calP_{\TT_i^{(t)}}^\perp(\calT_i^{(t)}-\calT_i)
=
-\lambda_i\calP_{\TT_i^{(t)}}^\perp U_i.
\]
Expanding \(U_i\) around \(U_i^{(t)}\) gives
\begin{equation}
U_i
=
\sum_{S\subseteq[d]}
\left(
\prod_{\ell\in S}s_\ell
\prod_{\ell\notin S}c_\ell
\right)
V_S,
\label{eq:subset-expansion}
\end{equation}
where
\[
V_S
=
\bigotimes_{\ell=1}^d v_{\ell,S},
\qquad
v_{\ell,S}
=
\begin{cases}
w_\ell, & \ell\in S,\\
u_{\ell,i}^{(t)}, & \ell\notin S.
\end{cases}
\]
Each retained $V_S$ has unit Frobenius norm. For $S\ne S'$ and $\ell\in S\triangle S'$, the product $\langle V_S,V_{S'}\rangle=\prod_{m=1}^d\langle v_{m,S},v_{m,S'}\rangle$ contains the factor $\langle w_\ell,u_{\ell,i}^{(t)}\rangle=0$, so the retained $V_S$ are orthonormal. By \eqref{eq:tangent space projection}, $\TT_i^{(t)}$ is spanned by $U_i^{(t)}=V_\emptyset$ and the tensors $u_{1,i}^{(t)}\otimes\cdots\otimes h_k\otimes\cdots\otimes u_{d,i}^{(t)}$ with $h_k\perp u_{k,i}^{(t)}$. Hence $V_S\in\TT_i^{(t)}$ if $|S|\leq1$. If $|S|\geq2$, then for every $k\in[d]$ some $\ell\in S\setminus\{k\}$ gives a factor $\langle w_\ell,u_{\ell,i}^{(t)}\rangle=0$ in
\[
\langle V_S,U_i^{(t)}\rangle
\quad\text{and}\quad
\langle V_S,u_{1,i}^{(t)}\otimes\cdots\otimes h_k\otimes\cdots\otimes u_{d,i}^{(t)}\rangle
=\langle v_{k,S},h_k\rangle\prod_{m\ne k}\langle v_{m,S},u_{m,i}^{(t)}\rangle,
\]
so $V_S\in\TT_i^{(t),\perp}$. Therefore,
\[
\calP_{\TT_i^{(t)}}^\perp U_i
=
\sum_{|S|\ge2}
\left(
\prod_{\ell\in S}s_\ell
\prod_{\ell\notin S}c_\ell
\right)
V_S.
\]
By the definition of the Segre atomic norm and \(0\le c_\ell\le1\),
\[
\left\|
\calP_{\TT_i^{(t)}}^\perp(\calT_i^{(t)}-\calT_i)
\right\|_{\operatorname{Seg},*}
\le
\lambda_i
\sum_{|S|\ge2}
\prod_{\ell\in S}s_\ell.
\]

It remains to bound the last sum. Let
\(
q=\sum_{\ell=1}^d s_\ell^2\),
\(
C_\Pi=\prod_{\ell=1}^d c_\ell.
\)
Since \(\langle\calT_i^{(t)},\calT_i\rangle=\lambda_i^{(t)}\lambda_iC_\Pi\), we have
\[
\left(\delta_i^{(t)}\right)^2
=
\frac{\|\calT_i^{(t)}-\calT_i\|_{\mathrm F}^2}{\lambda_i^2}
=
\frac{\{\lambda_i^{(t)}\}^2}{\lambda_i^2}+1-\frac{2\lambda_i^{(t)}}{\lambda_i}C_\Pi
=
\left(\frac{\lambda_i^{(t)}}{\lambda_i}-C_\Pi\right)^2+1-C_\Pi^2
\ge
1-C_\Pi^2.
\]
Since
\(
C_\Pi^2
=
\prod_{\ell=1}^d(1-s_\ell^2)
\le
\exp(-q),
\)
we obtain
\(
\left(\delta_i^{(t)}\right)^2
\ge
1-\exp(-q).
\)
Because \(\delta_i^{(t)}\le 1/(4d)\le1/2\), it follows that
\(
q
\le
-\log\left(1-\left(\delta_i^{(t)}\right)^2\right)
\le
2\left(\delta_i^{(t)}\right)^2.
\)
By the Cauchy--Schwarz inequality,
\[
\sum_{\ell=1}^d s_\ell
\le
\sqrt{dq}
\le
\sqrt{2d}\,\delta_i^{(t)}
\le
\frac{1}{2}.
\]
Hence
\[
\sum_{|S|\ge2}
\prod_{\ell\in S}s_\ell
\le
\exp\Bigl(\sum_{\ell=1}^d s_\ell\Bigr)-1-\sum_{\ell=1}^d s_\ell
\le
\Bigl(\sum_{\ell=1}^d s_\ell\Bigr)^2
\le
dq
\le
2d\left(\delta_i^{(t)}\right)^2.
\]
Consequently,
\[
\left\|
\calP_{\TT_i^{(t)}}^\perp(\calT_i^{(t)}-\calT_i)
\right\|_{\operatorname{Seg},*}
\le
2d\lambda_i\left(\delta_i^{(t)}\right)^2
=
2d
\frac{\|\calT_i^{(t)}-\calT_i\|_{\mathrm F}^2}{\lambda_i}.
\]
Since $\calP_{\TT_i^{(t)}}^\perp\calT_i^{(t)}=0$ and $\|W\|_{\mathrm F}\leq\|W\|_{\operatorname{Seg},*}$ for every tensor $W$, the preceding display gives
\[
\|\calP_{\TT_i^{(t)}}^\perp\calT_i\|_{\mathrm F}
\leq
\|\calP_{\TT_i^{(t)}}^\perp\calT_i\|_{\operatorname{Seg},*}
=
\|\calP_{\TT_i^{(t)}}^\perp(\calT_i^{(t)}-\calT_i)\|_{\operatorname{Seg},*}
\leq
2d\frac{\|\calT_i^{(t)}-\calT_i\|_{\mathrm F}^2}{\lambda_i},
\]
which proves \eqref{eq:PTip(Tihat - Ti)} and \eqref{eq:PTip_atomic} because $2d\leq3d\leq4d$.

Fix distinct $i,j\in[r]$ and write
$D_j=\calT_j^{(t)}-\calT_j$. Let $\TT_i$ and $\TT_j$ be the tangent spaces at
$\calT_i$ and $\calT_j$, respectively. For $k\in[d]$, let
$\calP_{k,i}=u_{k,i}u_{k,i}^{\top}$ and
$\calP_{k,i}^{\perp}=I_{p_k}-\calP_{k,i}$, and define the mutually
orthogonal projectors
\[
\mathsf P_{0,i}=\bigotimes_{\ell=1}^d\calP_{\ell,i},
\qquad
\mathsf P_{k,i}=\calP_{k,i}^{\perp}
\bigotimes_{\ell\ne k}\calP_{\ell,i},
\qquad k\in[d].
\]
Then $\calP_{\TT_i}=\sum_{k=0}^d\mathsf P_{k,i}$, and the analogous
decomposition holds for $\calP_{\TT_j}$. Assumption~\ref{assumption:incoherence}
and $0\leq\eta\leq1$ imply, for $k,k'\in\{0,\ldots,d\}$,
\[
\|\mathsf P_{k,i}\mathsf P_{k',j}\|
\leq
\begin{cases}
\eta^d,&k=k'=0,\\
\eta^{d-1},&k=0<k'\ \text{or}\ k'=0<k,\\
\eta^{d-1},&k=k'\geq1,\\
\eta^{d-2},&k,k'\geq1,\ k\ne k',
\end{cases}
\leq\eta^{d-2}.
\]
For $Z\in\TT_j$, write $Z=\sum_{k'=0}^d Z_{k'}$ with
$Z_{k'}=\mathsf P_{k',j}Z$. Orthogonality and Cauchy--Schwarz give
\begin{align*}
\|\calP_{\TT_i}Z\|_{\mathrm F}^2
&=
\sum_{k=0}^d
\left\|\sum_{k'=0}^d\mathsf P_{k,i}Z_{k'}\right\|_{\mathrm F}^2
\leq
(d+1)^2\eta^{2(d-2)}
\sum_{k'=0}^d\|Z_{k'}\|_{\mathrm F}^2.
\end{align*}
Consequently,
\begin{equation}
\|\calP_{\TT_i}\calP_{\TT_j}\|
\leq(d+1)\eta^{d-2}.
\label{eq:true-tangent-overlap}
\end{equation}

Since $\calP_{\TT_j}^{\perp}\calT_j=0$, we have $\calP_{\TT_j}^{\perp}D_j=\calP_{\TT_j}^{\perp}\calT_j^{(t)}$. Let $\theta_{\ell,j}\in[0,\pi/2]$ be the angle between $u_{\ell,j}$ and $u_{\ell,j}^{(t)}$, and let $q_j=\sum_{\ell=1}^d\sin^2\theta_{\ell,j}$. Applying \eqref{eq:subset-expansion} and the orthogonality argument following it with the roles of $\calT_j$ and $\calT_j^{(t)}$ interchanged, we obtain
\[
\frac{\|\calP_{\TT_j}^{\perp}D_j\|_{\mathrm F}^2}{\|\calT_j^{(t)}\|_{\mathrm F}^2}
=\sum_{|S|\geq2}\prod_{\ell\in S}\sin^2\theta_{\ell,j}\prod_{\ell\notin S}\cos^2\theta_{\ell,j}
\leq\sum_{m\geq2}\frac{q_j^m}{m!}
=\exp(q_j)-1-q_j.
\]
The bound on $q$ derived above depends only on these angles, so $q_j\leq2\{\delta_j^{(t)}\}^2\leq1/2$. Since also $\|\calT_j^{(t)}\|_{\mathrm F}\leq(1+\delta_j^{(t)})\lambda_j\leq2\lambda_j$,
\begin{align}
\|\calP_{\TT_j}^{\perp}D_j\|_{\mathrm F}
&\leq
\|\calT_j^{(t)}\|_{\mathrm F}
\left\{\exp(q_j)-1-q_j\right\}^{1/2}
\leq
4\lambda_j\{\delta_j^{(t)}\}^2
=4\delta_j^{(t)}\|D_j\|_{\mathrm F}.
\label{eq:true-normal-current-chord}
\end{align}
Here the second inequality uses
$\exp(q_j)-1-q_j\leq q_j^2$ for $0\leq q_j\leq1/2$.

For component $i$, choose the factor signs so that the angle
$\theta_{\ell,i}$ between $u_{\ell,i}$ and $u_{\ell,i}^{(t)}$ belongs to
$[0,\pi/2]$. Let $Q_{\ell,i}$ be the planar orthogonal rotation mapping
$u_{\ell,i}$ to $u_{\ell,i}^{(t)}$, extended by the identity on the
orthogonal complement of their span, and set
$Q_i=Q_{1,i}\otimes\cdots\otimes Q_{d,i}$. Then
$\TT_i^{(t)}=Q_i\TT_i$, so $\calP_{\TT_i^{(t)}}=Q_i\calP_{\TT_i}Q_i^\top$ and $\calP_{\TT_i}$ have equal rank. The next display uses $\|P-P'\|=\|P^\perp P'\|$ for orthogonal projectors $P,P'$ of equal rank, the identities
\[
\begin{gathered}
\calP_{\TT_i}^{\perp}\calP_{\TT_i^{(t)}}=\calP_{\TT_i}^{\perp}(Q_i-I)\calP_{\TT_i}Q_i^\top,
\\
Q_i-I=\sum_{k=1}^d\Bigl(\bigotimes_{\ell<k}Q_{\ell,i}\Bigr)\otimes(Q_{k,i}-I)\otimes\Bigl(\bigotimes_{\ell>k}I_{p_\ell}\Bigr),
\end{gathered}
\]
$\|Q_{\ell,i}\|=1$, and $\|Q_{\ell,i}-I\|=2\sin(\theta_{\ell,i}/2)\leq\theta_{\ell,i}$:
\begin{align}
\|\calP_{\TT_i^{(t)}}-\calP_{\TT_i}\|
&=
\|\calP_{\TT_i}^{\perp}\calP_{\TT_i^{(t)}}\|
\leq\|Q_i-I\|
\leq\sum_{\ell=1}^d\|Q_{\ell,i}-I\|
\leq
\sum_{\ell=1}^d\theta_{\ell,i}
\nonumber\\
&\leq
\frac{\pi}{2}\sqrt{d\sum_{\ell=1}^d\sin^2\theta_{\ell,i}}
\leq
\pi\sqrt{\frac d2}\,\delta_i^{(t)}
\leq2d\delta_i^{(t)}.
\label{eq:current-true-tangent-gap}
\end{align}
The fourth inequality uses $\theta\leq(\pi/2)\sin\theta$ for $\theta\in[0,\pi/2]$ and the Cauchy--Schwarz inequality. The fifth uses $\sum_{\ell}\sin^2\theta_{\ell,i}=q\leq2\{\delta_i^{(t)}\}^2$ from the proof of \eqref{eq:PTip_atomic}, and the last uses $d\geq3$.

Finally, \eqref{eq:true-tangent-overlap}--\eqref{eq:current-true-tangent-gap}
give
\begin{align*}
\|\calP_{\TT_i^{(t)}}D_j\|_{\mathrm F}
\leq{}&
\|\calP_{\TT_i}\calP_{\TT_j}D_j\|_{\mathrm F}
+\|\calP_{\TT_i}\calP_{\TT_j}^{\perp}D_j\|_{\mathrm F}
+\|(\calP_{\TT_i^{(t)}}-\calP_{\TT_i})D_j\|_{\mathrm F}\\
\leq{}&
\left\{(d+1)\eta^{d-2}
+4\delta_j^{(t)}+2d\delta_i^{(t)}\right\}\|D_j\|_{\mathrm F}.
\end{align*}
This proves
\eqref{eq:PTip(Tjhat - Tj)}.

\end{proof}

\end{document}